\documentclass[journal,compsoc]{IEEEtran}
\usepackage{amsmath, amsfonts, amsthm}
\usepackage{amssymb, mathtools}
\usepackage{multirow}
\usepackage{bbm}
\usepackage{algorithmic}
\usepackage[ruled,vlined]{algorithm2e}
\usepackage{array}
\usepackage[caption=false,font=normalsize,labelfont=sf,textfont=sf]{subfig}
\usepackage{pifont}
\usepackage{textcomp}
\usepackage{stfloats}
\usepackage{url}
\usepackage{verbatim}
\usepackage{graphicx}
\usepackage{cite}
\usepackage{xcolor}
\usepackage{MnSymbol}
\usepackage{booktabs}
\usepackage{kotex}
\usepackage{dashrule}
\usepackage[capitalize,noabbrev]{cleveref}
\definecolor{Ourcolor}{HTML}{007580}
\definecolor{Regkernelcolor}{HTML}{CE4A7E}
\definecolor{MIcolor}{HTML}{755139}
\definecolor{LAFTRcolor}{HTML}{4A274F}
\definecolor{MMDcolor}{HTML}{921416}
\definecolor{sIPMcolor}{HTML}{F9D142}
\definecolor{Trianglecolor}{HTML}{3B1877}
\definecolor{Epanecolor}{HTML}{DA5A2A}

\usepackage{amsmath,amsfonts,bm}

\def\eqref#1{equation~\ref{#1}}

\def\1{\bm{1}}

\DeclareMathAlphabet{\mathsfit}{\encodingdefault}{\sfdefault}{m}{sl}
\SetMathAlphabet{\mathsfit}{bold}{\encodingdefault}{\sfdefault}{bx}{n}

\DeclareMathOperator*{\argmin}{arg\,min}

\theoremstyle{plain}

\newtheorem{theorem}{Theorem}
\newtheorem{proposition}[theorem]{Proposition}
\newtheorem{lemma}[theorem]{Lemma}
\newtheorem{example}{Example}

\theoremstyle{definition}

\newtheorem{assumption}{Assumption}
\newtheorem{remark}{Remark}

\def\cV{{\cal V}}

\def\cF{{\cal F}}

\def\cX{{\cal X}}

\def\cZ{{\cal Z}}

\def\hat{\widehat}

\newcommand{\mbP}{\mathbb{P}}

\newcommand{\bz}{{\bf z}}

\newcommand{\bc}{\begin{center}}
\newcommand{\ec}{\end{center}}
\newcommand{\be}{\begin{equation}}
\newcommand{\ee}{\end{equation}}
\newcommand{\ba}{\begin{array}}
\newcommand{\ea}{\end{array}}
\newcommand{\bean}{\begin{eqnarray*}}
\newcommand{\eean}{\end{eqnarray*}}
\newcommand{\bea}{\begin{eqnarray}}
\newcommand{\eea}{\end{eqnarray}}
\newcommand{\ben}{\begin{enumerate}}
\newcommand{\een}{\end{enumerate}}
\newcommand{\bed}{\begin{itemize}}
\newcommand{\eed}{\end{itemize}}

\begin{document}

\title{Fair Representation Learning for Continuous Sensitive Attributes using Expectation of Integral Probability Metrics}

\author{Insung~Kong,
Kunwoong~Kim,
Yongdai~Kim

\thanks{I. Kong is with the Department of Applied Mathematics, University of Twente, Enschede, Netherlands and the Department of Statistics, Seoul National University, Seoul, South Korea, (e-mail: ggong369@snu.ac.kr)}
\thanks{K. Kim is with the Department of Statistics, Seoul National University, Seoul, South Korea, (e-mail: kwkim.online@gmail.com)}
\thanks{Y. Kim is with the Department of Statistics, Seoul National University, Seoul, South Korea, (e-mail: ydkim0903@gmail.com)}

\thanks{(Insung Kong and Kunwoong Kim are co-first authors.) (Corresponding author: Yongdai Kim.)
}}

\maketitle

\begin{abstract}
AI fairness, also known as algorithmic fairness, aims to ensure that algorithms operate without bias or discrimination towards any individual or group.
Among various AI algorithms, the Fair Representation Learning (FRL) approach has gained significant interest in recent years. 
However, existing FRL algorithms have a limitation: they are primarily designed for categorical sensitive attributes and thus cannot be applied to continuous sensitive attributes, such as age or income.
In this paper, we propose an FRL algorithm for continuous sensitive attributes.
First, we introduce a measure called the Expectation of Integral Probability Metrics (EIPM) to assess the fairness level of representation space for continuous sensitive attributes.
We demonstrate that if the distribution of the representation has a low EIPM value, then any prediction head constructed on the top of the representation become fair, regardless of the selection of the prediction head.
Furthermore, EIPM possesses a distinguished advantage in that it can be accurately estimated using our proposed estimator with finite samples.
Based on these properties, we propose a new FRL algorithm called Fair Representation using EIPM with MMD (FREM). 
Experimental evidences show that FREM outperforms other baseline methods.
\end{abstract}


\section{Introduction}
\label{intro}

\IEEEPARstart{A}{I} fairness, which is often referred to as algorithmic fairness in AI, is a widespread research area for ensuring social fairness in AI decision-making.
The basic philosophy of AI fairness is to fairly treat groups pre-defined by a given sensitive attribute (e.g., man vs. woman), which is called \textit{group fairness} \cite{calders2009building, feldman2015certifying, zafar2017fairness, donini2018empirical, pmlr-v80-agarwal18a}.
A primary criterion for group fairness is Demographic Parity (DP), which enforces that an AI model should not discriminate against different demographic groups in terms of predictions or decision-makings.

Among various research efforts for group fairness, Fair Representation Learning (FRL) has received significant attention recently
\cite{pmlr-v28-zemel13, louizos2015variational, 9aa5ba8a091248d597ff7cf0173da151_2016, calmon2017optimized, Madras2018LearningAF, Quadrianto_2019_CVPR, feng2019learning, pmlr-v89-song19a, sarhan2020fairness, ruoss2020learning, gupta2021controllable, pmlr-v130-gitiaux21a, zeng2021fair, kim2022learning, shui2022fair, oh2022learning, guo2022learning, pmlr-v202-jovanovic23a, deka2023mmd, shen2023fair}.
Fair representation, referred to as a feature vector whose distribution is aligned across protected groups, is obtained by feeding the input data into an encoder (i.e., feature extractor).
Once the fair representation is obtained, it is expected that any prediction head constructed on the top of it, where the fair representation serves as an input, will achieve a certain level of fairness.
See Fig. \ref{fig:illustration1} for the general illustration of FRL for binary sensitive attributes.

\begin{figure}[t]
    \centering
    \includegraphics[width=0.33\textwidth]{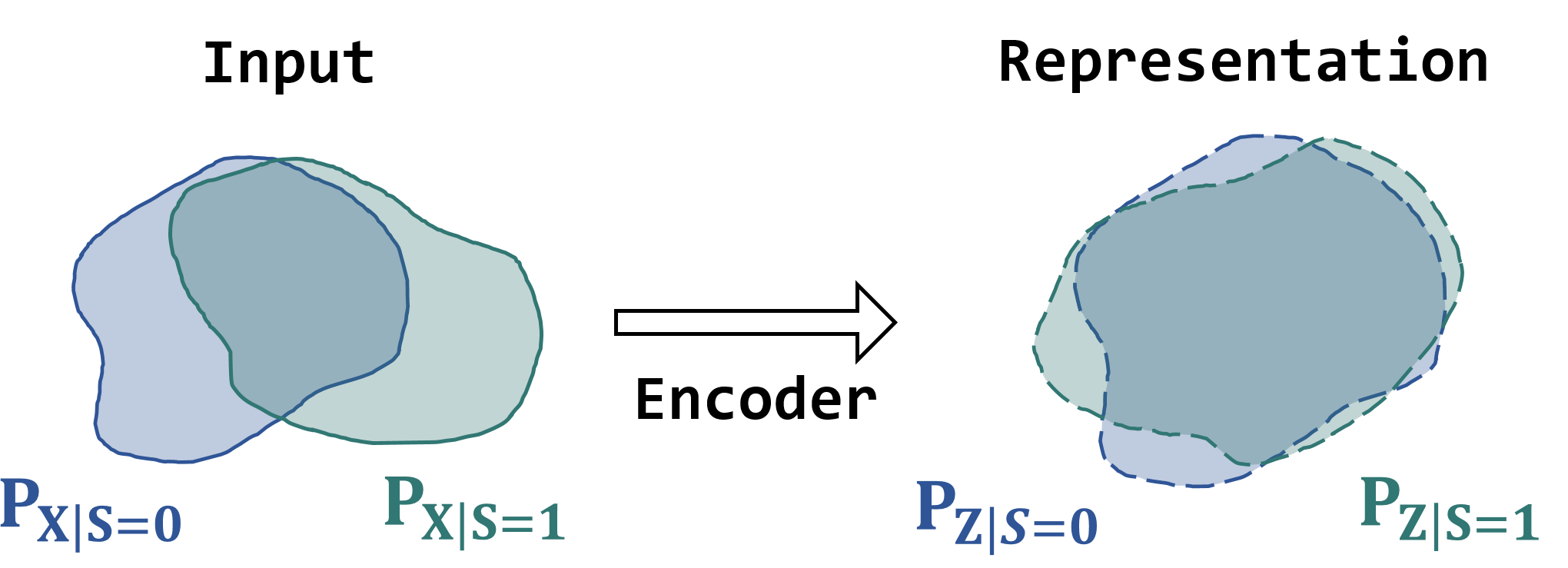}
    \caption{\textbf{A framework diagram of FRL for binary sensitive attributes.}}
    \label{fig:illustration1}
\end{figure}

\begin{figure}[t]
    \centering
    \includegraphics[width=0.33\textwidth]{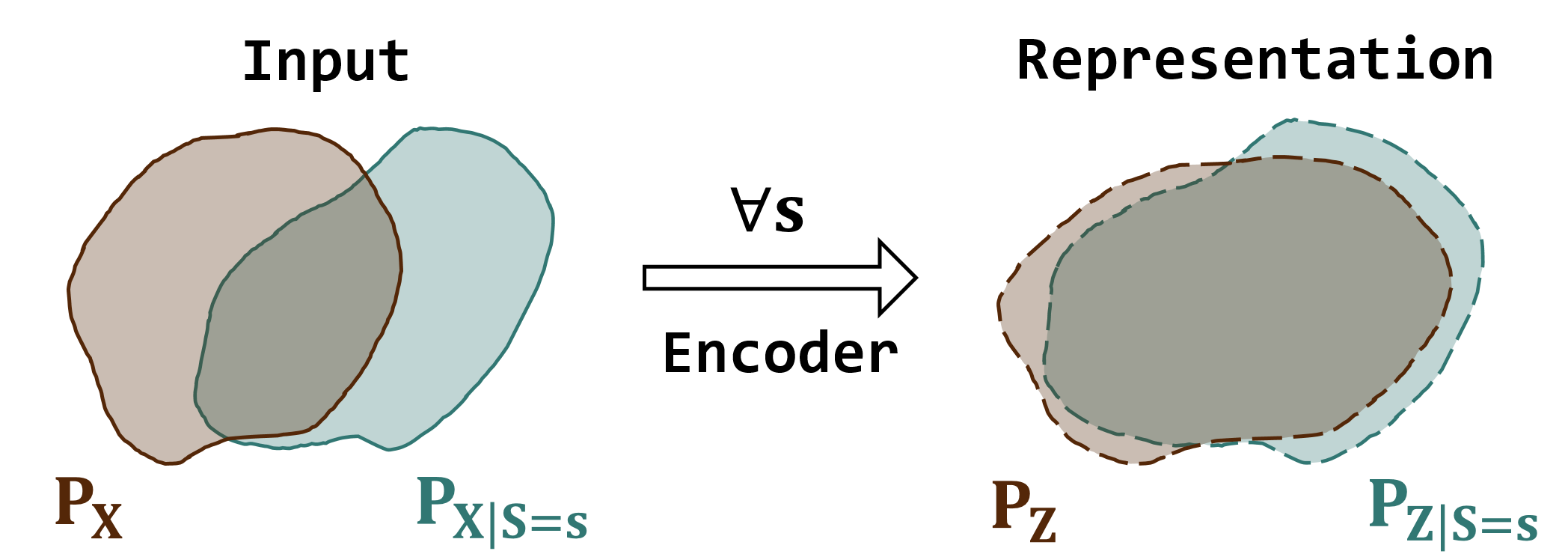}
    \caption{\textbf{A framework diagram of FRL for continuous sensitive attributes.}}
    \label{fig:illustration2}
    \vskip -0.2in
\end{figure}

In many cases, continuous sensitive attributes such as age, income and weight are frequently observed.
However, most fair algorithms have primarily focused on binary or categorical sensitive attributes, such as gender or race.
These algorithms could be applied to continuous sensitive attributes by transferring continuous sensitive attributes to categorical one by binning.
However, the optimal selection of the bins would not be easy and the performance of the algorithms could be significantly affected by this selection \cite{mary2019fairness}.
To address this critical issue, there have been endeavors to develop fair algorithms for continuous sensitive attributes  \cite{mary2019fairness, grari2021fairness, jiang2021generalized, pmlr-v202-giuliani23a}.

The aim of this paper is to develop an FRL algorithm for continuous sensitive attributes (Fig. \ref{fig:illustration2}).
We emphasize that existing FRL algorithms are limited to categorical sensitive attributes \cite{pmlr-v28-zemel13, louizos2015variational, 9aa5ba8a091248d597ff7cf0173da151_2016, calmon2017optimized, Madras2018LearningAF, Quadrianto_2019_CVPR, pmlr-v89-song19a, ruoss2020learning, gupta2021controllable,  pmlr-v130-gitiaux21a, zeng2021fair, kim2022learning,  shui2022fair, oh2022learning, guo2022learning, pmlr-v202-jovanovic23a, deka2023mmd, shen2023fair}, though these algorithms can be applied continuous sensitive attributes by binning. 
A key technical difficulty in learning a fair representation with continuous sensitive attributes
lies on estimating the conditional distribution of a representation vector given the sensitive attribute, which
is needed to measure the level of fairness of a given representation vector.
When the sensitive attribute is continuous, estimation of the conditional distribution 
cannot be done simply by using the empirical distribution for each value of sensitive attribute because 
at most one observation exists for each value in the training data.

To develop an FRL algorithm for continuous sensitive attributes,
we first introduce a new metric designed to quantify the disparity in the conditional and marginal distributions.
Specifically, we propose to use the expectation of Integral Probability Metric (IPM) between the conditional and marginal distributions of a given representation.
We refer to this metric as the \textbf{E}xpectation value of \textbf{IPM}s (EIPM).
EIPM has a desirable property that if the distribution of a representation has a low EIPM value, then the predictions of any head built over the representation become fair, regardless of the selection of the head.

To use EIPM in FRL, we need to estimate it based on training data.
For this purpose, we devise a weighted empirical distribution of the representation using the kernel smoothing technique to propose an EIPM estimator. 
We provide theoretical justifications including the asymptotic convergence rate of our proposed EIPM estimator. 

An important contribution of this paper is to develop a new technique to derive the convergence rate of the kernel smoothed EIPM estimator.
Note that the standard technique for theoretical study of kernel smoothed estimators
(e.g., Nadaraya-Watson estimator)
is to calculate the bias and variance of the corresponding estimator with respect to the sample size and bandwidth.
Since EIPM involves the {\it sup} operation in its definition which is nonlinear, the calculation of the bias and variance is not an easy task.
We develop a new and novel technique to derive the convergence rate of the kernel smoothed EIPM estimator.

Based on the EIPM and its estimation, we propose a new fair representation learning algorithm called \textbf{F}air \textbf{R}epresentation using \textbf{E}IP\textbf{M} (FREM).
Experimental results confirm that FREM outperforms various baseline methods that could be categorized into: (1) regularization methods that directly learn fair prediction models \cite{jiang2021generalized, mary2019fairness} and 
(2) variants of FRL methods for categorical sensitive attributes\cite{10.1145/3278721.3278779, Madras2018LearningAF, kim2022learning, deka2023mmd}.

Our contributions are categorized as follows.
\begin{itemize}
    \item We introduce a new fairness measure called EIPM, which has desirable properties to be used for FRL.
    \item We propose an estimator for EIPM having desirable statistical properties.
    \item Based on the EIPM, we develop a new Fair Representation Learning (FRL) algorithm for continuous sensitive attributes, called FREM.
    \item Experiments demonstrate that FREM outperforms existing state-of-the-art methods in terms of the fairness-prediction trade-off.
\end{itemize}

\section{Preliminary}
\label{not_pre}

\subsection{Notation}
Let $\mathbb{R}$ and $\mathbb{N}$ be the sets of real numbers and natural numbers, respectively.
We denote $\mathbb{R}^{+}_0 := \{ x \in \mathbb{R} : x \geq 0\}$ and $\mathbb{N}_0 := \{0\} \cup \mathbb{N}$.
Let $[N] := \{ 1, \ldots, N \}$ for $N \in \mathbb{N}.$
A capital letter denotes a random variable, and a vector is denoted by a bold letter.
Let  $\bm{X} \in \mathcal{X} \subset \mathbb{R}^{d}$ and $S \in \mathcal{S} \subset \mathbb{R}$ be the non-sensitive random input vector and sensitive random variable.
Let $Y \in \mathcal{Y}$ be the output variable, which can be a binary or numerical variable. 
Also let $\bm{Z} := h(\bm{X})$ be the representation of an input vector obtained by a given encoding function $h:\cX \to\cZ \subset \mathbb{R}^{m}$.  
For readability, we will interchangeably use $\bm{Z}$ and $h(\bm{X})$ unless there is any confusion.
We denote $\mbP_{\bm{X}}$, $\mbP_S$ and $\mbP_{\bm{Z}} (= \mbP_{h(\bm{X})})$ as the distribution of $\bm{X}$, $S$ and $\bm{Z}$, respectively.
Also, we denote $\mbP_{\bm{X},S}$ as the joint distribution of $(\bm{X},S)$, and $\mbP_{\bm{Z}|S=s} (= \mbP_{h(\bm{X})|S=s})$ as the conditional distribution of $\bm{Z}$ on $S=s$. 
We denote $f: \cZ \to \mathcal{Y}$ be a prediction head built on the representation space $\cZ$, and $g = f \circ h: \mathcal{X} \rightarrow \mathcal{Y}$ as a full prediction function.

For a distribution $\mbP$ and given $\bm{Z}_1, \dots, \bm{Z}_{n} \overset{i.i.d.}{\sim} \mbP$, the empirical distribution of $\mbP$ is defined by $\hat{\mbP} := \frac{1}{n} \sum_{i=1}^{n} \delta(\bm{Z}_i)$, where $\delta(\cdot)$ is the Dirac delta function.
For given measures $\mu^0$ and $\mu^1$, $\mu^0 \otimes \mu^1$ denotes the product measure of $\mu^0$ and $\mu^1$. Also, we write $\mu^0 << \mu^1$ if $\mu^0$ is dominated by $\mu^1$.
For a real valued function $v : \mathcal{Z} \to \mathbb{R}$, we denote $||v||_{\infty} := \sup_{\bm{z} \in \mathcal{Z}} |v(\bm{z})|$ as the infinite norm of $v$.
For $\epsilon>0$ and a set of functions $\mathcal{V}$, 
we denote $\mathcal{N}(\epsilon, \mathcal{V}, || \cdot ||_{\infty})$
as the smallest number of $\epsilon$-cover of $\mathcal{V}$ with respect to the infinite norm.

\subsection{Fair prediction models}

We say that a prediction model is fair when certain statistics regarding to the prediction (e.g., the proportion of being positive for classification and the mean prediction for regression) for each protected group are similar. 
To learn fair prediction models, we have to choose two things - fairness measure and learning algorithm.

\textbf{Fairness measures}
Let $\phi : \mathbb{R} \to \mathbb{R}$ be a measurable function.
For a given prediction model $g : \cX \to \mathcal{Y}$, the Demographic Parity (DP) for a binary sensitive attribute $S \in \{0,1\}$ is defined as
$$
\Delta \texttt{DP}_{\phi}(g) = 
\left\vert \mathbb{E}_{\bm{X}} \left( \phi \circ g(\bm{X}) | S = 1 \right) - \mathbb{E}_{\bm{X}} \left( \phi \circ g(\bm{X})  | S = 0 \right) \right\vert.
$$
Various fairness measures can be represented by choosing $\phi.$
For example of the binary classification, $\phi (w) = \mathbbm{1}(w \ge 0)$ corresponding to the original DP measure \cite{calders2009building, barocas2016big} and $\phi (w) = w$ leads to the mean DP \cite{Madras2018LearningAF, chuang2021fair}.

When the sensitive attribute is multinary, i.e., $S \in [C]$ with $C > 2$, the definition of DP is modified to 
the difference w.r.t. demographic parity (DDP) \cite{cho2020fair}, which is defined as
$$
\Delta \texttt{DDP}_{\phi}(g) = 
\sum_{s \in [C]} \left\vert \mathbb{E}_{\bm{X}} \left( \phi \circ g(\bm{X}) | S = s \right) - \mathbb{E}_{\bm{X}} \left( \phi \circ g(\bm{X}) \right) \right\vert.
$$

Note that both DP and DDP are not applicable to the case of continuous sensitive attributes.
To devise a fair learning algorithm for continuous sensitive attributes, \cite{mary2019fairness} and \cite{grari2021fairness} propose to estimate the Hirschfeld-Gebelein-Rényi (HGR) maximal correlation coefficient between $S$ and $g(\bm{X})$.
However, this measure 
is computationally involved so that its accurate estimation for continuous sensitive attributes is not known \cite{jiang2021generalized, pmlr-v202-giuliani23a}.
To mitigate this issue, as an alternative, \cite{jiang2021generalized} proposes Generalized Demographic Parity (GDP), which is defined as
$$
\Delta \textup{\texttt{GDP}}_{\phi}(g) =  \mathbb{E}_{S} \Big| \mathbb{E}_{\bm{X} } ( \phi \circ g(\bm{X})  | S ) - \mathbb{E}_{\bm{X}} ( \phi \circ g(\bm{X}) ) \Big|,
$$
which can be estimated by the standard kernel smoothing technique \cite{nadaraya1964estimating, watson1964smooth}.
Note that GDP shares a similar concept with DDP in the sense that it examines the difference between the conditional expectation of the prediction value with respect to the sensitive attribute and the marginal expectation of the prediction value.

As \cite{jiang2021generalized} did, we consider the identity function for $\phi$ in this paper, and we drop the subscript $\phi$
in $\Delta \texttt{GDP}_\phi$ unless there is any confusion. That is, we let
$$
\Delta \texttt{GDP} (g) =  \mathbb{E}_{S} \Big| \mathbb{E}_{\bm{X} } ( g(\bm{X})  | S ) - \mathbb{E}_{\bm{X}} ( g(\bm{X}) ) \Big|.
$$

\textbf{Learning algorithms}
Various approaches have been proposed to obtain fair prediction models.
In general, a given model $g$ is said to be fair when it has a small value of $\Delta (g)$ for a pre-specified fairness measure $\Delta.$
Existing algorithms for finding such fair models can be categorized into three groups:
(i) pre-, (ii) post-, and (iii) in-processing.

The pre-processing algorithms are to remove unfair biases in the training data before learning prediction models, and use the debiased training data to learn prediction models \cite{kamiran2012data, pmlr-v28-zemel13, feldman2015certifying, calmon2017optimized, webster2018mind, xu2018fairgan, pmlr-v97-creager19a, quadrianto2019discovering}.
The post-processing methods try to transform unfair prediction models to be fair \cite{corbett2017algorithmic, wei20a, pmlr-v115-jiang20a}.

The in-processing approach has been mostly explored among the three, which attempts to find an accurate model 
among fair prediction models.
For binary sensitive attributes, various algorithms \cite{goh2016satisfying, zafar2017fairness, donini2018empirical, chuang2021fair} have been suggested. 
For continuous sensitive attributes, \cite{mary2019fairness, grari2021fairness, jiang2021generalized} have proposed fairness constraints under which fair models are learned by minimizing given objective functions.

\subsection{Fair Representation Learning}

FRL aims to build a fair representation whose distributions for each protected group are similar.
Then, the learned fair representation could be used as new data for downstream tasks such as
constructing prediction models \cite{pmlr-v28-zemel13, Madras2018LearningAF, Quadrianto_2019_CVPR, kim2022learning, deka2023mmd}.
Since the representation is fair, any prediction model built upon the top of the representation would be fair.
A main theme of FRL is to choose a metric to measure a similarity between the distributions of each protected group.

\textbf{Binary sensitive attribute case}
For a binary sensitive attribute $S \in \{0,1\}$, FRL aims to find an encoding function $h$ such that 
\begin{align}
    \mbP_{h(\bm{X}) | S=0} \approx \mbP_{h(\bm{X}) | S=1}. \label{FRL_concept_binary}
\end{align}
Then,
$
\mathbb{E}_{\bm{X}} \left( f \circ h(\bm{X}) | S = 1 \right) \approx \mathbb{E}_{\bm{X}} \left( f \circ h(\bm{X})  | S = 0 \right)
$
holds for any prediction head $f$. 
Several FRL algorithms have been introduced in an extensive amount of literature
\cite{pmlr-v28-zemel13, louizos2015variational, 9aa5ba8a091248d597ff7cf0173da151_2016, calmon2017optimized, Madras2018LearningAF, Quadrianto_2019_CVPR, pmlr-v89-song19a, ruoss2020learning, gupta2021controllable,  pmlr-v130-gitiaux21a, zeng2021fair, kim2022learning,  shui2022fair, oh2022learning, guo2022learning, pmlr-v202-jovanovic23a, deka2023mmd, shen2023fair}. 

The key of FRL is the choice of a deviance measure that quantifies the dissimilarity between the two distributions.
Once the deviance measure is chosen, a fair representation is constructed by minimzing the deviance measure
in the learning phase.
Examples of the deviance measure are Kullback-Leibler (KL) divergence \cite{10.1007/978-3-030-61527-7_38}, Jensen-Shannon (JS) divergence \cite{9aa5ba8a091248d597ff7cf0173da151_2016, Madras2018LearningAF}, 
Integral Probability Metric (IPM) \cite{lee2022fast, deka2023mmd, kim2022learning}, etc.

Among these various possible choices, IPM has been received much attention in recent works partly because
it does not require the existence of the density.
Let $\cV$ be a set of discriminators  from $\cZ$ to $\mathbb{R}$, where $||v||_{\infty} \leq 1$ holds for every $v \in \cV$. 
The IPM (with respect to $\mathcal{V}$) for given two distributions $\mbP^0$ and $\mbP^1$ is defined as
$$\textup{IPM}_{\cV}(\mbP^0,\mbP^1) := \sup_{v\in \cV} \left| \int v(\bm{z}) (d\mbP^0(\bm{z})-d\mbP^1(\bm{z}))\right|.$$
When $\cV$ is the $1$-Lipschitz function space\footnote{A given function $v$ defined on $\cZ$ is a Lipschitz function with the Lipschitz constant $L$ if
$|v(\bz_1)-v(\bz_2)| \le L \|\bz_1-\bz_2\|$ for all $\bz_1,\bz_2\in \cZ,$
where $\|\cdot\|$ is certain norm defined on $\cZ.$}, the IPM becomes the well-known Wasserstein distance \cite{KR:58}. 
The IPM has been popularly used in various applications including the generative model and distributional robustness analysis \cite{10.5555/3305381.3305404, husain2020distributional}.

\textbf{Continuous sensitive attribute case}
For sensitive attributes, there exist infinitely many $s \in \mathcal{S}$ as well as infinitely many conditional distributions of $h(\bm{X})$ given $S=s$ (i.e., $\mbP_{h(\bm{X}) | S=s}$).
Following the concepts of DDP and GDP, FRL aims to ensure that each conditional distribution is closely similar to the marginal distribution of $h(\bm{X})$.
That is, instead of (\ref{FRL_concept_binary}), the goal is to find an encoding function $h$ such that 
\begin{align}
    \mbP_{h(\bm{X}) | S=s} \approx \mbP_{h(\bm{X})}, \forall s \in \mathcal{S}. \label{FRL_concept_con}
\end{align}
However, since $S$ is a continuous variable, there exists at most one sample such that $S_i = s$ for each $s \in \mathcal{S}$.
This fact makes estimating the conditional distribution $\mbP_{h(\bm{X}) | S=s}$ very difficult, and therefore quantifying the similarity between the two distributions in (\ref{FRL_concept_con}) also becomes challenging.
Learning fair representation on continuous sensitive attributes poses a significant challenge, and to the best of our knowledge, there is no existing work for FRL for continuous sensitive attributes without binning.

In the next section, we propose a new quantity to measure the level of fairness in representation when the sensitive attribute is continuous.


\section{FREM: A fair representation learning algorithm for continuous sensitive attributes}

In this section, we develop an FRL algorithm for continuous sensitive attributes.
In Section \ref{subsec_defEIPM}, we define a new fairness measure of a given representation with respect to a continuous sensitive attribute, called the \textbf{E}xpectation of \textbf{IPM}s (EIPM) and
provide a relation between EIPM and the fairness of 
a prediction model built upon the fair representation.
Then, we propose an estimator of EIPM using the weighted empirical distribution in Section \ref{subsec_defEIPM_est}, 
and develop an FRL algorithm for sensitive attributes based
on the estimated EIPM in Section \ref{sec_MMD}.
An extension of the FRL algorithm for equal opportunity is discussed in Section \ref{subsec_eo}.

EIPM is an extension of IPM for continuous sensitive attributes. 
It would be possible to consider other deviances such as the KL (Kullback-Leibler) and JS (Jensen-Shannon) divergences rather than IPM. 
However, in this paper, we focus on IPM since we succeed in developing a computationally feasible and theoretically sound estimator of EIPM.  
Apparently, it would not be easy to modify the KL and JS divergences to be easily estimable for continuous sensitive attributes since they require the estimation of the conditional density instead of the conditional distribution.

\subsection{The Expectation of IPMs (EIPM)} \label{subsec_defEIPM}

For a continuous sensitive attribute $S$, the IPMs between $\mbP_{\bm{Z}| S=s}$ and $\mbP_{\bm{Z}}$ vary across $s \in \mathcal{S}$.  
Due to this reason, we use the \textbf{E}xpectation value of \textbf{IPM}s (EIPM) for the deviance measure between the conditional distributions and the marginal distribution.
That is, we use
\begin{align}
    \textup{EIPM}_\mathcal{V} (\bm{Z} ; S) := \mathbb{E}_{S} \left[ \textup{IPM}_\mathcal{V} (\mbP_{\bm{Z}|S} ,\mbP_{\bm{Z}}) \right]. \label{def_EIPM}
\end{align}

EIPM possesses several desirable properties for FRL. 
First, we give a basic property that the EIPM value being zero guarantees perfect fairness. 
\begin{theorem}[Perfect fairness] \label{pro_basic}
    Assume that a set of discriminator $\mathcal{V}$ is large enough for $\textup{IPM}_\mathcal{V}$ to be a metric on the space of probabilities on $\mathcal{Z}$. 
    Then, $\textup{EIPM}_\mathcal{V} (\bm{Z}; S) = 0$ implies $\Delta \textup{\texttt{GDP}}( f \circ h ) = 0$ for any (bounded) prediction head $f$.
\end{theorem}
The assumption regarding $\mathcal{V}$ in Theorem \ref{pro_basic} is a standard one for the IPM, which is satisfied by most of the commonly used discriminator sets \cite{muller1997integral}.
Since there always exists a trade-off between the level of fairness and the prediction performance of the model, we are more interested in achieving a certain level of fairness instead of perfect fairness. 
Under this context, as proved in Theorem \ref{GDP_upper} below, an important property of EIPM is that the level of GDP of any given prediction head can be controlled by the level of EIPM as long as the prediction head is included in a properly defined function class. This result indicates that the fair representation learned by FREM can be used
as new input data for various downstream tasks requiring fairness.

\begin{theorem}[Controlling the level of fairness by EIPM]\label{GDP_upper}
    Let $\bm{Z}=h(\bm{X})$ be a representation corresponding to an encoding function $h$.
    For a class $\mathcal{V}$ of discriminators and a class $\mathcal{F}$ of prediction heads, we have the following results.
    \begin{enumerate}
        \item Let $\kappa := \sup_{f \in \mathcal{F}} \inf_{v \in \mathcal{V}} || f - v ||_{\infty}$. 
        Then for any prediction head $f \in \mathcal{F}$, we have
    $$\Delta \textup{\texttt{GDP}}(f \circ h) \leq \textup{EIPM}_{\mathcal{V}} (\bm{Z}; S) + 2\kappa.$$
        \item Assume there exists an increasing concave function $\xi : [0,\infty) \to [0,\infty)$ such that $\lim_{r \downarrow 0} \xi(r) = 0$ and $\textup{IPM}_{\mathcal{F}}(\mbP^0, \mbP^1) \leq \xi(\textup{IPM}_{\mathcal{V}}(\mbP^0, \mbP^1))$ for any two probability measures $\mbP^0$ and $\mbP^1$.  
        Then for any prediction head $f \in \mathcal{F}$, we have
        $$\Delta \textup{\texttt{GDP}}(f \circ h) \leq \xi(\textup{EIPM}_{\mathcal{V}} (\bm{Z}; S)).$$
    \end{enumerate}
\end{theorem}

\noindent Note that if $\mathcal{V}$ is sufficiently large so that $\mathcal{F} \subseteq \mathcal{V}$, we directly obtain $\Delta \textup{\texttt{GDP}}(f \circ h) \leq \textup{EIPM}_{\mathcal{V}} (\bm{Z}, S)$.
There exists, however, an interesting example of $\mathcal{V}$ such that $\mathcal{V}$ is fairly small (e.g., $\mathcal{V} \subset \mathcal{F}$) but one of assumptions in Theorem \ref{GDP_upper} holds. 
We provide certain representative examples of $\mathcal{V}$ below.   
\begin{example}[H\"older smooth functions]
    Let $\mathcal{F}$ be the $\beta$-Hölder function class\footnote{$\beta$-Hölder norm is defined by $||f||_{\mathcal{H}^\beta} :=  \sum_{\bm{\alpha} : |\bm{\alpha}|_1 \leq \beta} \left\|\partial^{\bm{\alpha}} f\right\|_{\infty} 
+\sum_{\bm{\alpha}:|\bm{\alpha}|_1 =\lfloor\beta\rfloor} \sup _{\bm{z}_1 \ne \bm{z}_2}
\frac{\left|\partial^{\bm{\alpha}} f(\bm{z}_1)-\partial^{\bm{\alpha}} f(\bm{z}_2)\right|}{|\bm{z}_1-\bm{z}_2|_{\infty}^{\beta-\lfloor\beta\rfloor}}$. $\beta$-Hölder class is the set of bounded $\beta$-Hölder norm.}. 
    For any sufficiently large $M$, consider the DNN class $\mathcal{V}$ with $\text{depth} \propto \log_2 M$ and $\text{width} \propto M^d$. Then, we have $\kappa \propto 1/M^{2 \beta}$ \cite{kohler2021rate}.
\end{example}

\begin{example}[Lipschitz continuous functions]
     Let $\mathcal{V}$ and $\mathcal{F}$ be the Lipschitz function class with Lipschitz constant $1$ and $L$, respectively. Then, we have $\xi(r) = Lr$.
\end{example}

\begin{remark}
GDP is the discrepancy between the conditional expectation and the marginal expectation of model output, whereas EIPM is the difference between the conditional distribution and the marginal distribution of representation.
The biggest challenge in measuring the fairness level in the representation with respect to continuous sensitive attributes is that, simply matching the expectations of the distributions in the representation space does not guarantee the fairness of the final prediction model (even in terms of GDP).
Of course, making the conditional distributions be similar is more difficult than making the conditional expectations similar, in particular for continuous sensitive attributes because there are infinitely many conditional distributions should be considered simultaneously.
\end{remark}

\subsection{Estimation of EIPM}
\label{subsec_defEIPM_est}

For a binary sensitive attribute,
when we do not know the population distributions 
$\mbP_{\bm{Z}|S=0}$ and $\mbP_{\bm{Z}|S=1}$ 
but we observe random samples
$\bm{Z}^{0}_1, \dots, \bm{Z}^{0}_{n_0} \sim \mbP_{\bm{Z}|S=0}$ and $\bm{Z}^{1}_1, \dots, \bm{Z}^{1}_{n_1} \sim \mbP_{\bm{Z}|S=1}$,
$\textup{IPM}_{\cV}(\mbP_{\bm{Z}|S=0},\mbP_{\bm{Z}|S=1})$ can be easily estimated by $\textup{IPM}_{\cV}(\hat{\mbP}_{\bm{Z}|S=0},\hat{\mbP}_{\bm{Z}|S=1})$,
where $\hat{\mbP}_{\bm{Z}|S=s}, s\in \{0,1\}$ are the empirical distributions of $\mbP_{\bm{Z}|S=s}, s\in \{0,1\},$  respectively.
That is
 $\textup{IPM}_{\cV}(\hat{\mbP}_{\bm{Z}|S=0},\hat{\mbP}_{\bm{Z}|S=1})=\sup_{v\in \cV} \left| \frac{1}{n_0}\sum_{i=1}^{n_0} v(\bm{Z}^{0}_{i}) - \frac{1}{n_1}\sum_{i=1}^{n_1} v(\bm{Z}^{1}_{i}) \right|.$ 
 
When $S$ is a multinary categorical variable, 
a natural estimator of the EIPM
\begin{align*}
    \textup{EIPM}_\mathcal{V} (\bm{Z} ; S) =  \int_{s \in \mathcal{S}}  \textup{IPM}_\mathcal{V} \left(\mbP_{\bm{Z}|S=s} ,\mbP_{\bm{Z}}\right) \mathbb{P}_S (ds)
\end{align*}
is
\begin{align}
    \hat{\textup{EIPM}}_\mathcal{V}^{\textup{cat}} (\bm{Z}; S) :=& \int_{s \in \mathcal{S}}  \textup{IPM}_\mathcal{V} \left(\hat{\mbP}^{\textup{cat}}_{\bm{Z}|S=s} ,\hat{\mbP}_{\bm{Z}}\right) \hat{\mathbb{P}}_S (ds), \label{discrete_empdis}    
\end{align}
where ‘cat’ in the superscript is a short for 'categorical',
\begin{align*}
    \hat{\mbP}^{\textup{cat}}_{\bm{Z}|S=s}  := \frac{1}{|\{j : S_j = s\}|}\sum_{j : S_j = s}\delta(\bm{Z}_j),
\end{align*}
$\hat{\mbP}_{\bm{Z}} := \frac{1}{n} \sum_{j=1}^n \delta(\bm{Z}_j)$
and $\hat{\mbP}_{S} := \frac{1}{n} \sum_{i=1}^n \delta(S_i)$. 

In case of continuous sensitive attributes, however, estimating $\mathbb{P}_{\bm{Z}|S=s}$ with $\hat{\mbP}^{\textup{cat}}_{\bm{Z}|S=s}$ would  not be appropriate.
It is because $\hat{\mbP}^{\textup{cat}}_{\bm{Z}|S=s}$ is not well-defined when there exists
no $S_j$ equal to $s.$ 
Moreover, $|\{j : S_j = s\}|$ is at most 1 and thus 
$\hat{\mbP}^{\textup{cat}}_{\bm{Z}|S=s}$ is not even statistically consistent. 
In general, estimation of the conditional distribution
is challenging and smoothing techniques are typically employed.  
The aim of this subsection is to propose a consistent estimator of EIPM for sensitive attributes by use of a new kernel smoothing technique.

To present the new smoothing technique, we first let $K_{\gamma} : \mathcal{S} \times \mathcal{S} \to \mathbb{R}$ be a kernel function on $\mathcal{S}$ with bandwidth $\gamma,$
which measures similarity between any pair of sensitive attribute $s, s' \in \mathcal{S}$.
We assume that $K_{\gamma}$ satisfies the following assumption.
\begin{assumption} \label{assum_kernel}
    There exists a function $k : \mathbb{R} \to \mathbb{R}^{+}_0$ such that $K_{\gamma}(s, s') = k\left( \frac{s - s'}{\gamma} \right)$, $||k||_{\infty} < \infty$, $\int k(s)ds = 1$, $\int s^2k(s)ds < \infty$ and $k(s) = k(-s)$ for every $s, s' \in \mathcal{S}$.
\end{assumption}
A popular choice of the kernel is Radial Basis Function (RBF) kernel, defined by
$K_{\gamma}(s, s') := \frac{1}{\sqrt{2 \pi}}\exp \left( - \frac{ (s-s')^{2} }{ 2 \gamma^{2} } \right).$
Any kernel satisfying Assumption \ref{assum_kernel} can be used (e.g., triangle, Epanechnikov) and we experimentally compare the three kernels in Section \ref{exp:vary_kernel}.

Then, to estimate $\mbP_{\bm{Z}|S=S_i},$ we propose to use
$$\hat{\mbP}^{(-i), \gamma}_{\bm{Z}|S=S_i}:= \sum_{j \neq i} \hat{w}_{\gamma}(j; i) \delta(\bm{Z}_j)$$ for each $i \in \{1,\dots,n\}.$
where the weights $\hat{w}_{\gamma}(1 ; i)$,$\dots$, $\hat{w}_{\gamma}(n ; i)$ are defined by
\begin{align*}
    \hat{w}_{\gamma}(j ; i) := \frac{K_{\gamma}(S_j, S_i)}{\sum_{j \neq i} K_{\gamma}(S_j, S_i)}, && j \in \{ 1, \ldots, n \}.
\end{align*}
Note that for $i \in \{1,\dots,n\}$, $\sum_{j \neq i} \hat{w}_{\gamma}(j ; i)=1$ and the index $j$ with $S_j \approx S_i$ has a larger value of $\hat{w}_{\gamma}(j ; i).$
In other words, $\hat{\mbP}^{(-i), \gamma}_{\bm{Z}|S=S_i}$ is a weighted empirical distribution on $\mathcal{Z}$, which gives more weights for samples closer to $S_i$.
Similarly, to estimate $\mbP_{\bm{Z}},$ we use $\hat{\mbP}^{(-i)}_{\bm{Z}} := \frac{1}{n-1} \sum_{j \neq i} \delta(\bm{Z}_j).$ 
Then, the IPM between $\hat{\mbP}^{(-i), \gamma}_{\bm{Z}|S=S_i}$ and $\hat{\mbP}^{(-i)}_{\bm{Z}}$ for $i \in \{1,\dots,n\}$ becomes
\begin{align*}
    \textup{IPM}_\mathcal{V} \left(\hat{\mbP}^{(-i), \gamma}_{\bm{Z}|S=S_i} ,\hat{\mbP}^{(-i)}_{\bm{Z}}\right) = \sup_{v \in \mathcal{V}} 
\left| \sum_{j \neq i} \left(\hat{w}_\gamma (j;i) - \frac{1}{n-1}\right)v(\bm{Z}_j)\right|. 
\end{align*}
The resulting proposed EIPM estimator is given as
\begin{align}
    \hat{\textup{EIPM}}^{\gamma}_{\mathcal{V}} (\bm{Z}; S) := \frac{1}{n} \sum_{i=1}^n  \textup{IPM}_{\mathcal{V}} \left(\hat{\mbP}^{(-i), \gamma}_{\bm{Z}|S=S_i} ,\hat{\mbP}^{(-i)}_{\bm{Z}}\right). \label{eq_sampled_EIPM}
\end{align}
As introduced above, we exclude the $i$th sample in the estimate of $\mathbb{P}_{\bm{Z}|S=S_{i}}$ for technical simplicity; this convention makes theoretical studies of the EIPM estimator be easier.

One may raise a question that the weighted empirical distribution
$\hat{\mbP}^{(-i), \gamma}_{\bm{Z}|S=S_i}$ is similar to the Nadaraya–Watson conditional density estimator \cite{de2003conditional}.
However, it differs in that it applies the kernel smoothing to $S$ only, while using a sum of Dirac delta functions for $\bm{Z}$. 
We apply this approach because EIPM depends on the conditional expectation of $\bm{Z}$ instead of the conditional density, which makes the EIPM estimator be free from the curse of dimensionality (with respect to the dimension of $\bm{Z}$).
See Appendix \ref{Appen:syn_multi} for a numerical discussion of this claim.

\begin{remark}
One may consider using the `expectation of KL divergence' $(\mathbb{E}_{S} \left[ \textup{KL}(\mbP_{\bm{Z}|S} ,\mbP_{\bm{Z}}) \right])$ or similar measures, instead of EIPM.
However, the technique we develop (using
weighted empirical distribution) cannot be applied to KL divergence-based methods.
Note that $\textup{KL} (\hat{\mbP}^{(-i), \gamma}_{\bm{Z}|S=S_i} ,\hat{\mbP}^{(-i)}_{\bm{Z}}) = \sum_{j=1}^n \hat{w}_\gamma (j;i) \log( (n-1)\hat{w}_\gamma (j;i))$ does not depend on $\{\bm{Z}_i\}_{i=1}^n$, which means that this value is quite different from $\textup{KL}(\mbP_{\bm{Z}|S=S_i} ,\mbP_{\bm{Z}})$.
\end{remark}

\subsection{A choice of the discriminator}

For feasible estimation of EIPM in (\ref{eq_sampled_EIPM}), a careful selection of the set of discriminators (i.e., $\mathcal{V}$) should be done.
There are several candidates of the set of discriminators for IPM
including the 1-Lipschitz function class \cite{KR:58}, the parametric family proposed by \cite{kim2022learning} and the RKHS unit ball \cite{gretton2012kernel}, which correspond to the Wasserstein distance, the sigmoid IPM and the Maximum Mean Discrepancy (MMD), respectively.
Straightforward application of these discriminators to EIPM would face computational difficulties.
In particular, any set of discriminator that requires a numerical maximization to calculate the IPM value would be
prohibited for EIPM since $n$ many maximizations should be done to calculate the EIPM value.
The computations of Wasserstein distance and sigmoid IPM require such sets of discriminators while the MMD has a closed-form solution and thus we can learn $h$ and $f$ without the adversarial learning (i.e., numerical maximization to compute the IPM value). 
Thus, we choose the MMD as the IPM in our proposed FRL algorithm.

To explain more details of the MMD, let $\kappa : \mathcal{Z} \times \mathcal{Z} \to \mathbb{R}$ be a positive definite kernel function on $\mathcal{Z}$. 
For the Reproducing Kernel Hilbert Space (RKHS) $(\mathcal{V}_{\kappa}(\mathcal{Z}), ||\cdot||_{\mathcal{V}_{\kappa}(\mathcal{Z})})$ corresponding to $\kappa$, we consider the unit ball in the RKHS 
$\mathcal{V}_{\kappa, 1}=\{v \in \mathcal{V}_{\kappa}(\mathcal{Z}): ||v||_{\mathcal{V}_{\kappa}(\mathcal{Z})} \le 1\}$
for the set of discriminator used for EIPM.
The IPM employing this set of discriminators is referred to as the MMD, which is widely used across various domains \cite{li2017mmd, lee2022fast, pmlr-v202-kong23d}.

We are now ready to introduce the closed-form formula of our proposed estimator based on MMD, denoted by $\hat{\textup{EIPM}}^{\gamma}_{\mathcal{V}_{\kappa, 1}} (\bm{Z}; S),$ which is given in the following proposition. 

\begin{proposition} \label{pro_estimator_derivation}
    For given $\gamma>0$, $h \in \mathcal{H}$, $\{\bm{X}_i , S_i\}_{i=1}^n$ and $\bm{Z}_i = h(\bm{X}_i)$,  $\hat{\textup{EIPM}}^{\gamma}_{\mathcal{V}_{\kappa, 1}} (\bm{Z}; S)$ is given as
    \begin{align*}
    \hat{\textup{EIPM}}^{\gamma}_{\mathcal{V}_{\kappa, 1}} (\bm{Z}; S) = \frac{1}{n} \sum_{i=1}^n \left[ \sum_{j,k \neq i} [A_{\gamma}]_{i,j} [A_{\gamma}]_{i,k}  \kappa(\bm{Z}_j, \bm{Z}_k) 
    \right]^{\frac{1}{2}},
    \end{align*}
    where $A_{\gamma}$ is the $n \times n$ matrix defined by
    $$[A_{\gamma}]_{i,j} = \frac{K_{\gamma}(S_i , S_j)}{\sum_{j \neq i} K_{\gamma}(S_i , S_j) } - \frac{1}{n-1}.$$
\end{proposition}

Note that it is theoretically well-known that many RKHSs including the Gaussian and Laplace RKHS can encompass or approximate a wide range of function \cite{steinwart2001influence, steinwart2008support, gretton2012kernel}.
Hence, using $\mathcal{V}_{\kappa, 1}$ as the set of discriminator for EIPM ensures low GDP values for a wide range of prediction heads, including nonlinear ones,
Also, ${\mathcal{V}_{\kappa,1}}$ usually has a model complexity similar to that of a parametric family, which 
result in good finite sample performances of the estimated EIPM 
(see details in Section \ref{sec_tho}).

\subsection{Learning a fair representation for continous sensitive attributes} \label{sec_MMD}

The aim of this subsection is to choose a good encoder $h$ among those satisfying $\textup{EIPM}_{\mathcal{V}} (h(\bm{X}), S) \leq \delta$ for a pre-specified $\delta$.
There are two approaches to achieve this goal: supervised and unsupervised.
For unsupervised FRL, Auto-Encoder is typically used for a learning framework \cite{Madras2018LearningAF, kim2022learning}.
On the other hand, the supervised approach learns the encoder and prediction head simultaneously, provided that observations of the output $Y$ are available.
In these days, supervised FRL is more popular partly because supervised pre-trained models can be successively transferred to various downstream tasks (e.g., GPT).
Moreover, recent FRL algorithms such as LAFTR \cite{Madras2018LearningAF} and sIPM-LFR \cite{kim2022learning} also considered supervised learning for their numerical studies.
Thus, we focus on the supervised learning since it is more popular and widely used \cite{Quadrianto_2019_CVPR,gupta2021controllable,ruoss2020learning,oh2022learning,guo2022learning,deka2023mmd}.

For supervised FRL, we learn a fair representation by solving
\begin{align}
    \argmin_{h \in \mathcal{H}, f \in \mathcal{F}} & \mathcal{L}_{\textup{sup}}(f \circ h) && \textup{ s.t. } \textup{EIPM}_{\mathcal{V}} (h(\bm{X}); S) \leq \delta, \label{argmin_1} 
\end{align}
where $\mathcal{L}_{\textup{sup}}$ is a given supervised risk such as the cross-entropy  for classification or MSE (Mean Squared Error) for regression.
That is, the algorithm finds a good encoder $h$ among those satisfying the fairness constraint by minimizing the supervised risk with respect to $h$ and $f$ jointly.

In practice, however, the value of $\textup{EIPM}_{\mathcal{V}}$ in (\ref{argmin_1}) is not available. 
A simple remedy is to replace it by its estimator 
$\hat{\textup{EIPM}}^{\gamma}_{\mathcal{V}_{\kappa, 1}}$ (provided in Proposition \ref{pro_estimator_derivation}) to find the solution of  
\begin{align}
    \argmin_{h \in \mathcal{H}, f \in \mathcal{F}} & \mathcal{L}_{\textup{sup}}(f \circ h) && \textup{ s.t. } \hat{\textup{EIPM}}^{\gamma}_{\mathcal{V}_{\kappa, 1}} (h(\bm{X}), S) \leq \delta \label{argmin_2} 
\end{align}
with a properly chosen bandwidth $\gamma$.
Based on the foregoing discussions, we arrive at a new algorithm of FRL for continuous sensitive attributes.
Specifically, we solve the Lagrangian dual problem of our objective in (\ref{argmin_2}). 
That is, we solve
\begin{align}
    \argmin_{h \in \mathcal{H}, f \in \mathcal{F}}  \mathcal{L}_{\textup{sup}}(f \circ h) 
    + \lambda \hat{\textup{EIPM}}^{\gamma}_{\mathcal{V}_{\kappa, 1}} (h(\bm{X}), S),
\end{align}
where the multiplier $\lambda \geq 0$ is a hyper-parameter controlling the relative magnitude of the fairness constraint.
We call this proposed algorithm as the
\textbf{F}air \textbf{R}epresentation using \textbf{E}IP\textbf{M} (FREM), which is summarized in Algorithm \ref{alg:eipm} of Appendix.

\subsection{Extension to Equal Opportunity} \label{subsec_eo}

Equal Opportunity (EO) \cite{hardt2016equality} is another important group fairness notion, besides DP.
The FREM algorithm can be modified easily for Generalized Equal Opportunity (GEO), which is defined as
\begin{equation}\label{def:deltageo}
    \Delta \texttt{GEO} := \mathbb{E}_{S} \Big| \mathbb{E}_{\bm{X} } ( g(\bm{X})  | S, Y = 1 ) - \mathbb{E}_{\bm{X}} ( g(\bm{X}) | Y = 1 ) \Big|.
\end{equation}
Instead of (\ref{def_EIPM}), we consider
$$\textup{EIPM}_\mathcal{V} (\bm{Z}; S | Y=1) := \mathbb{E}_{S} \left[ \textup{IPM}_\mathcal{V} (\mbP_{\bm{Z}|S, Y=1} ,\mbP_{\bm{Z}|Y=1}) \right],$$
and we estimate it by $ \hat{\textup{EIPM}}^{\gamma}_{\mathcal{V}_{\kappa, 1}} (\bm{Z}; S | Y=1) := $
\begin{align*}
    \frac{1}{n_1} \sum_{i : Y_i=1}  \textup{IPM}_{\mathcal{V}_{\kappa, 1}} \left(\hat{\mbP}^{(-i), \gamma}_{\bm{Z}|S=S_i, Y=1} ,
    \hat{\mbP}^{(-i)}_{\bm{Z}|Y=1}\right),
\end{align*}
where $n_1:= |\{i : Y_i = 1\}|$ and $\hat{\mbP}^{(-i), \gamma}_{\bm{Z}|S=S_i, Y=1}$ 
is defined as
$$\hat{\mbP}^{(-i), \gamma}_{\bm{Z}|S=S_i, Y=1}:= \sum_{j \neq i} \frac{K_{\gamma}(S_i, S_j) \mathbb{I}(Y_j = 1)}{\sum_{j \neq i} K_{\gamma}(S_i, S_j) \mathbb{I}(Y_j = 1)} \delta(Z_i)$$
and $\hat{\mbP}^{(-i)}_{\bm{Z}|Y=1} := \frac{1}{n_1 - 1}\sum_{j \neq i} \mathbb{I}(Y_j = 1)\delta(\bm{Z}_j)$.
Then, all the theorems and algorithms discussed in the previous subsections can be modified accordingly without much hamper.
Details are provided in Appendix \ref{App_EO}.

\section{Theoretical analysis} \label{sec_tho}

We study theoretical properties of the estimated EIPM given in (\ref{eq_sampled_EIPM}) 
which in turn provides theoretical guarantees of the fairness level of the learned fair representation by FREM.

\subsection{Convergence rate of the estimated EIPM}
\label{subsec__est_conv}

In this subsection, we derive the convergence rate of
$\hat{\textup{EIPM}}^{\gamma}_{\mathcal{V}}(\bm{Z}; S)$ as the sample size increases.
Even though the EIPM estimator looks similar to the Nadaraya–Watson estimator, the $\sup$ operation in the definition of EIPM makes EIPM be nonlinear with respect to the conditional distribution of $\bm{Z}$ given $S$ and thus standard techniques to study the Nadaraya–Watson estimator (e.g., calculation of the bias and variance) are not directly applicable. 
To resolve this issue, we develop novel techniques to verify several asymptotic properties of the EIPM estimator. 
We assume the following mild regularity conditions.
\begin{assumption} \label{assum_ps}
    $\mbP_{S}$ admits a density $p(s)$ with respect to Lebesgue measure $\mu_{S}$ on $\mathcal{S}$. Also, there exist $0<L_p<U_p<\infty$ such that $L_p < p(s) < U_p$ on $s \in \mathcal{S}$.    
\end{assumption}

\begin{assumption} \label{assum_pzs}
Suppose that there exists a $\sigma$-finite measure $\mu_{\bm{X}}$ on $\mathcal{X}$
such that $\mbP_{\bm{X},S} << \mu_{\bm{X}} \otimes \mu_{S},$ where $\mu_{S}$ is the Lebesgue measure on $\mathcal{S}.$
    We denote $p(\bm{x},s) := \frac{d\mbP_{\bm{X},S}}{d(\mu_{\bm{X}} \otimes  \mu_{S})}(\bm{x},s)$ as the Radon-Nikodym derivative of $\mbP_{\bm{X},S}$ with respect to $\mu_{\bm{X}} \otimes  \mu_{S}$.     
    For every $\bm{x} \in \mathcal{X}$, $p(\bm{x},s)$ is twice differentiable with respect to $s$ and has a bounded second derivative.
\end{assumption}
Assumption \ref{assum_ps} implies that $S$ admits a bounded density function (w.r.t. Lebesgue measure).
This is a very mild one, because $\mathcal{S}$ is usually a bounded set.
Assumption \ref{assum_pzs} is about the smoothness of the joint density function with respect to $S.$
Note that there is no smoothness condition on $\bm{X}.$

\begin{assumption} \label{assum_gamma}
    The bandwidth of the kernel
    satisfies $\gamma_n \to 0$ and $n \gamma_n \to \infty$ as $n \to \infty$. 
\end{assumption}
Assumption $\ref{assum_gamma}$ is necessary for convergence of kernel estimators \cite{rosenblatt1969conditional}.
This assumption implies that a smaller bandwidth should be used as the number of samples increases, but the rate of decrease should not be too rapid. The following theorem is the main result of this paper.

\begin{theorem}[Convergence of proposed estimator] \label{EIPM_conv}
Let $h$ be a bounded measurable encoder.
Suppose that Assumption \ref{assum_kernel}, \ref{assum_ps}, \ref{assum_pzs} and \ref{assum_gamma} hold.
Then, for 
$$
\epsilon_n = \gamma_n^2 + \frac{\log n}{\sqrt{n \gamma_n}} \left( 1 + \log \mathcal{N} \left(\sqrt{\frac{\gamma_n}{n}}, \mathcal{V}, ||\cdot||_{\infty}\right) \right)^{\frac{1}{2}},
$$
we have
\begin{align*}
    \left|\hat{\textup{EIPM}}^{\gamma_n}_{\mathcal{V}} (\bm{Z} ; S) - \textup{EIPM}_{\mathcal{V}} (\bm{Z} ; S) \right| 
    < c \epsilon_n
\end{align*}
for sufficiently large $n$ with probability at least $1-\frac{4}{n}$, where $c$ is the constant not depending on $n$ and $m$. 
\end{theorem}

Theorem \ref{EIPM_conv} provides the statistical convergence rate of the EIPM estimator with respect to the sample size $n.$
If $\mathcal{V}$ consists of a single function,
the error rate becomes a well-known upper bound of the root mean square error of the Nadaraya–Watson non-parametric regression estimator \cite{nadaraya1964estimating, watson1964smooth, rosenblatt1969conditional}.
However, as the size of the discriminator set becomes larger, the error rate becomes slower, which is consistent with the known property when estimating the IPM with finite samples \cite{sriperumbudur2012empirical}.
Note that the model complexity of $\mathcal{V}_{\kappa, 1}$ is much smaller than that of most other sets of discriminators (e.g., the Lipschitz function class), which results in a faster convergence rate.

On the contrary, one can raise a concern regarding the curse of dimensionality when estimating EIPM on high-dimensional representations, which is one of the common challenges associated with kernel estimation methods \cite{bengio2005curse}.
However, the convergence rate does not depend on the dimension of $\bm{Z}$ and hence we are able to handle representations of high dimension.
This is because our estimator is devised to estimate the conditional expectation directly by employing the kernel method only for a sensitive attribute.

\begin{remark}\label{rmk:bandwidth}
    Similar to other kernel methods, the convergence rate of our estimator depends on $\gamma_n$.
    Up to a logarithmic factor, the optimal $\gamma_n$ for given $\mathcal{V}$ is 
    $$\gamma_n^{opt} \propto \left( \frac{1}{n} \log \mathcal{N} \left( \sqrt{\frac{1}{n}}, \mathcal{V}, ||\cdot||_{\infty} \right) \right)^{1/5},$$ 
    which yields
    $$\epsilon_n^{opt} = \frac{\log n}{n^{2/5}} \left( 1 + \log \mathcal{N} \left(\sqrt{\frac{1}{n}}, \mathcal{V}, ||\cdot||_{\infty} \right) \right)^{2/5}.$$ 
\end{remark}

\begin{figure*}[h]
    \centering
    \includegraphics[scale=0.26]{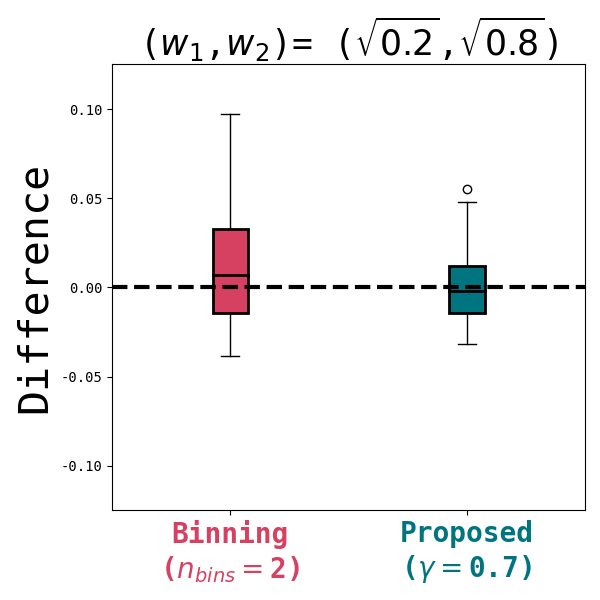}
    \includegraphics[scale=0.26]{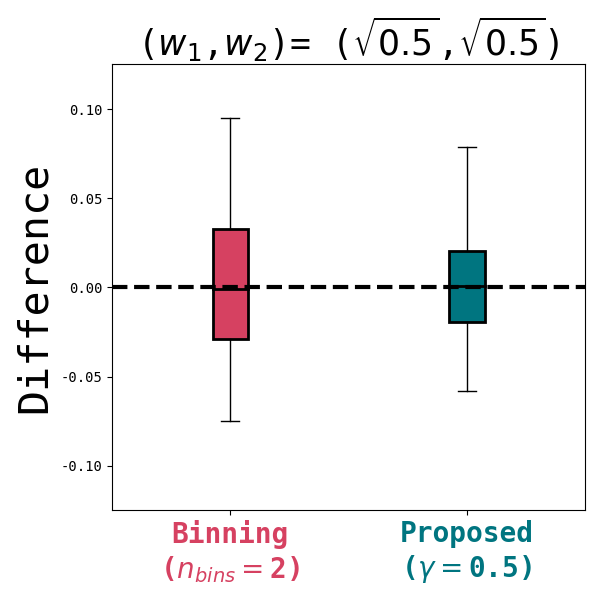}
    \includegraphics[scale=0.26]{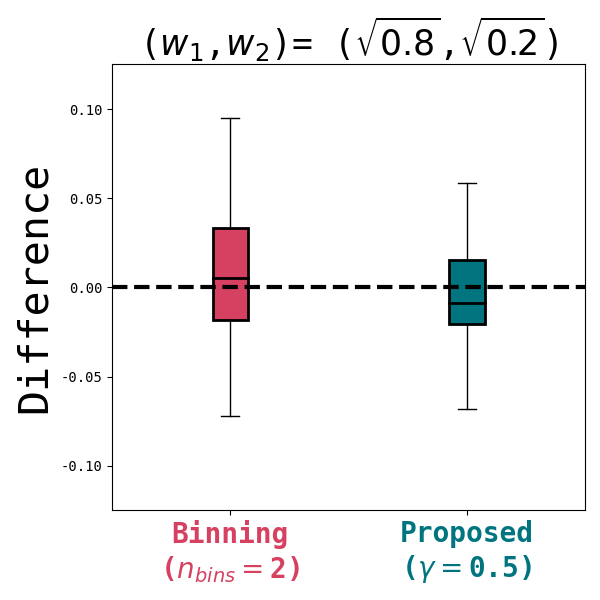}
    \caption{
    \textbf{Simulation results}:
    Box plots of the differences between the true EIPM and the two estimators (the best binning estimator and the best proposed estimator).
    It is clear that our proposed estimator is more accurate. 
    \textbf{(Left)} $w_{1} = \sqrt{0.2}, w_{2} = \sqrt{0.8}.$
    \textbf{(Center)} $w_{1} = \sqrt{0.5}, w_{2} = \sqrt{0.5}.$
    \textbf{(Right)} $w_{1} = \sqrt{0.8}, w_{2} = \sqrt{0.2}.$
    }
    \label{fig:synthetic1}
    \vskip -0.2in
\end{figure*}

\subsection{Theoretical guarantees for the fairness level of the estimated fair representation}

A statistical question is whether the two constraints in (\ref{argmin_1}) and (\ref{argmin_2}) become similar as $n$ increases.
Note that Theorem \ref{EIPM_conv} itself does not guarantee the convergence since the result holds for a fixed $h.$
The following theorem ensures that the two constraints are asymptotically equivalent.
\begin{theorem}[Asymptotically equivalence of the constraints] \label{thm_encoder_set}
    Let $\mathcal{H}$ and $\mathcal{V}$ be a set of encoders and the set of discriminators, respectively, where the elements of $\mathcal{V}$ are Lipschitz  with the Lipschitz constant $L>0$.
    For $\delta>0$, we define 
    $$\mathcal{H}_{\mathcal{V}}(\delta) := \{ h \in \mathcal{H} : \textup{EIPM}_{\mathcal{V}} (h(\bm{X}); S) \leq \delta\}$$
    and 
    $$\hat{\mathcal{H}}^{\gamma}_{\mathcal{V}}(\delta) := \{ h \in \mathcal{H} : \hat{\textup{EIPM}}^{\gamma}_{\mathcal{V}} (h(\bm{X}); S) \leq \delta\}$$
    as the set of encoders whose representation spaces satisfy the fairness constraints defined by $\textup{EIPM}_{\mathcal{V}}$ and $\hat{\textup{EIPM}}^{\gamma}_{\mathcal{V}}$, respectively.
    Suppose that Assumption \ref{assum_kernel}, \ref{assum_ps}, \ref{assum_pzs} and \ref{assum_gamma} hold.
    Then, for every $\delta>0$ and 
    \begin{align*}
        \epsilon_n = \gamma_n^2 + \frac{\log n}{\sqrt{n \gamma_n}} \Bigg( 1 &+ \log \mathcal{N} \left(\frac{1}{2}\sqrt{\frac{\gamma_n}{n}}, \mathcal{V}, ||\cdot||_{\infty}\right) \\
        &+ \log \mathcal{N} \left(\frac{1}{2L}\sqrt{\frac{\gamma_n}{n}}, \mathcal{H}, ||\cdot||_{\infty}\right)
        \Bigg)^{\frac{1}{2}},
    \end{align*}
    we have
    \begin{align*}        
    \mathcal{H}_{\mathcal{V}}(\delta- c\epsilon_n) 
    \subseteq \hat{\mathcal{H}}^{\gamma_n}_{\mathcal{V}}(\delta) 
    \subseteq  \mathcal{H}_{\mathcal{V}}(\delta + c\epsilon_n)
    \end{align*}
    for sufficiently large $n$ with probability at least $1-\frac{4}{n}$, where $c$ is a constant not depending on $n$, $m$ and $\delta$.
\end{theorem}

Theorem \ref{thm_encoder_set} implies that for any $\delta>0$, $\hat{\mathcal{H}}^{\gamma}_{\mathcal{V}}(\delta)$ converges to $\mathcal{H}_{\mathcal{V}}(\delta)$, with the specified convergence rate.
Compared to Theorem \ref{EIPM_conv}, the model complexity of $\mathcal{H}$ is additionally incorporated in the convergence rate.
This is due to the necessity of uniform convergence of $\hat{\textup{EIPM}}^{\gamma}_{\mathcal{V}}(h(\bm{X}), S)$
to $\textup{EIPM}_{\mathcal{V}}(h(\bm{X}), S)$ with respect to $h\in \mathcal{H}.$
With Theorem \ref{thm_encoder_set}, we can ensures that every encoder in $\hat{\mathcal{H}}^{\gamma}_{\mathcal{V}}(\delta)$ has an EIPM value of at most $\delta + c \epsilon_n,$ and so does the estimated fair representation by FREM.

\section{Experiments}
\label{exps}

This section presents the results of numerical experiments.
In Section \ref{sec:simul}, we provide empirical evidences for 
inferior performances of the estimation of EIPM by use of the simple binning technique, which
supports that our proposed estimator is necessary.
In Section \ref{sec:realdata}, we investigate the performance of FREM 
compared with existing state-of-art algorithms
by analyzing several benchmark tabular datasets and graph datasets, respectively.
In Section \ref{sec:expimp}, we summarize the implications of the experimental studies.
In all experiments, we use the RBF kernel function with scale parameter $\sigma > 0$ for $\kappa$
(i.e., $\kappa(\bm{z}, \bm{z}^{\prime}) = \exp(-\|\bm{z}-\bm{z}^{\prime}\|^2 / 2\sigma^2 )$.

\subsection{Synthetic dataset: estimation of EIPM}\label{sec:simul}

We empirically compare our proposed estimator of EIPM with those obtained through the simple binning technique
by analyzing a synthetic dataset.
For the joint distribution of $\bm{X} = [X^{(1)}, X^{(2)}]^{\top}$ and $S,$ we consider:
$$
\begin{pmatrix}
    S \\
    X^{(1)} \\
    X^{(2)} \\
\end{pmatrix}
\sim 
N \left( 
\begin{bmatrix}
    0   \\
    0 \\
    0 \\
\end{bmatrix},
\begin{bmatrix}
    1 & \rho & 0  \\
    \rho & 1 & 0  \\
    0 & 0 & 1  \\
\end{bmatrix}
\right),
$$
where $\rho \ge 0$ is a given correlation between $X^{(1)}$ and $S.$
For the encoder function, we consider a linear functions:
$$ h(\cdot) = \bm{w}^{\top} \cdot, \bm{w} = \left[ w_{1}, w_{2} \right]^{\top}, \Vert \bm{w} \Vert_{2} = 1.$$
Our goal is to estimate $\textup{EIPM}_{\mathcal{V}_{\kappa,1}}(h(\bm{X}); S)$.
Using the fact that the marginal and conditional distributions of the representation are also Gaussian distributions, we can obtain the true EIPM value analytically, whose details are given in Appendix \ref{appen:syn_true_1}.

For the simulation, we consider the three true encoder functions corresponding to 
$$(w_1, w_2) \in \{ (\sqrt{0.2}, \sqrt{0.8}), (\sqrt{0.5}, \sqrt{0.5}), (\sqrt{0.8}, \sqrt{0.2}) \}$$
with the fixed correlation $\rho = 0.4$. 
Synthetic data of the size $n = 100$ are generated from each of the three true  probabilistic models and the proposed estimator $\hat{\textup{EIPM}}_{\mathcal{V}_{\kappa,1}}^{\gamma} (h(\bm{X}); S)$ is computed using Proposition \ref{pro_estimator_derivation}
with the bandwidth $\gamma$ selected from $\{ 0.3, 0.5, 0.7 \}.$
We also consider the binning estimator, that is, we first categorize $S$ by quantiles and then calculate $\hat{\textup{EIPM}}_{\mathcal{V}_{\kappa,1}}^{\textup{cat}} (h(\bm{X}); S)$ in equation (\ref{discrete_empdis}).
We vary the number of bins ($n_{bins}$) over $\{ 2, 3, 4 \}.$  

For each probabilistic model,
we generate synthetic datasets 100 times and obtain 100 EIPM estimates.
Fig. \ref{fig:synthetic1} displays the box plots of the 100 differences of
the estimated and true EIPM values for each probabilistic model with $n_{bins}$ and $\gamma$
selected on the test data.
The biases, MAE (Mean Absolute Error)s and RMSE (Root Mean Squared Error)s of the estimates with various $n_{\textup{bins}}$ or $\gamma$s are provided in Table \ref{tab:synthetic1} of Appendix \ref{Appen:syn_result_1}.
First of all, the results confirm that our proposed estimator dominates the binning estimator with large margins.
In addition, another interesting observation is that the bias of the binning keeps increasing as $n_{bins}$ increases. 
One would think that a large $n_{bins}$ leads over-parametrization which results in small bias but large variance.
This conjecture, however, is not valid for the binning estimator. 
This would be partly due to the non-linearity of EIPM with respect to the conditional distributions.
Additional results on various representation dimensions and number of samples are provided in Appendix \ref{Appen:syn_multi}.

\subsection{Real data analysis}\label{sec:realdata}

We compare the performance of FREM with existing state-of-the-art baselines by analyzing three benchmark real datasets
- two tabular datasets and two graph datasets. 
To save the space, we present the results only for the two tabular data in the main manuscript and defer the results for the graph data to Appendix \ref{appen:exp_results}.
Even though the graph data are more complex, the results are similar to those for the tabular data, which confirms that FREM works well regardless of data domain.

\textbf{Datasets}
(1) Classification: We use \textsc{Adult}\footnote{\url{https://archive.ics.uci.edu/ml/machine-learning-databases/adult}} and two graph datasets, \textsc{Pokec-n} and \textsc{Pokec-z}, which are constructed from Slovakia's social network called Pokec\footnote{\url{https://snap.stanford.edu/data/soc-Pokec.html}}.
In \textsc{Adult} dataset, the target label is whether the income of an individual exceeds $50\$$k or not and the continuous sensitive attribute is the age.
In \textsc{Pokec-n} and \textsc{Pokec-z} datasets, the target label is the (binarized) working field of an individual and the continuous sensitive attribute is the age.
(2) Regression: We use \textsc{Crime} dataset\footnote{\url{https://archive.ics.uci.edu/ml/datasets/communities+and+crime}}.
In \textsc{Crime} dataset, the target response is the number of crimes per population in US communities and the continuous sensitive attribute is the black group ratio of a given community. 
Details about the datasets including an example of the dataset bias are provided in Appendix \ref{appen:data}.

\textbf{Performance measures of prediction models}
To evaluate the prediction performance, we use two measures for each task.
For classification, we consider accuracy (\texttt{Acc}) and average precision (\texttt{AP})\footnote{\texttt{AP} is the area under the precision-recall curve.}.
For regression, we use the mean squared error (\texttt{MSE}) and the mean absolute error (\texttt{MAE}).

For fairness evaluation,
we mainly employ the (kernel-based) Generalized Demographic Parity ($\Delta \texttt{GDP}$, \cite{jiang2021generalized})
on the test data.
If it is exactly zero, $\hat{Y}$ is independent of $S,$ implying perfect fairness.
In addition, as alternatives to $\Delta \texttt{GDP},$ we consider two additional fairness measures: $\Delta \texttt{HGR}$ \cite{mary2019fairness} and $\texttt{MI}(\hat{Y}, S)$ \cite{MacKay2003}.
$\Delta \texttt{GDP}$ and $\Delta \texttt{HGR}$ are calculated directly following the estimators provided by \cite{jiang2021generalized} and \cite{mary2019fairness}, respectively, and
see Appendix \ref{sec:appen-fairness_measures} for computation of $\texttt{MI}(\hat{Y}, S).$

\begin{figure*}[ht]
    \centering
    \fbox{
    \includegraphics[width=0.23\textwidth]{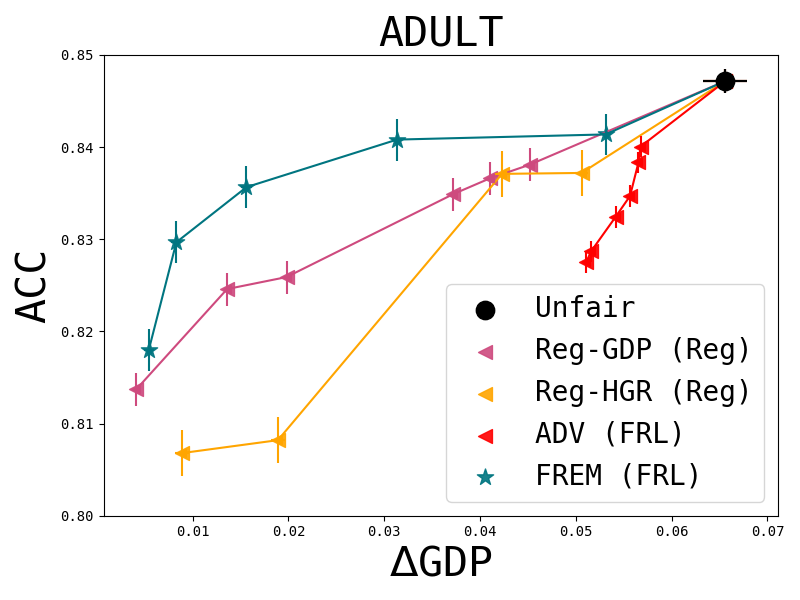}
    }
    \fbox{
    \includegraphics[width=0.23\textwidth]{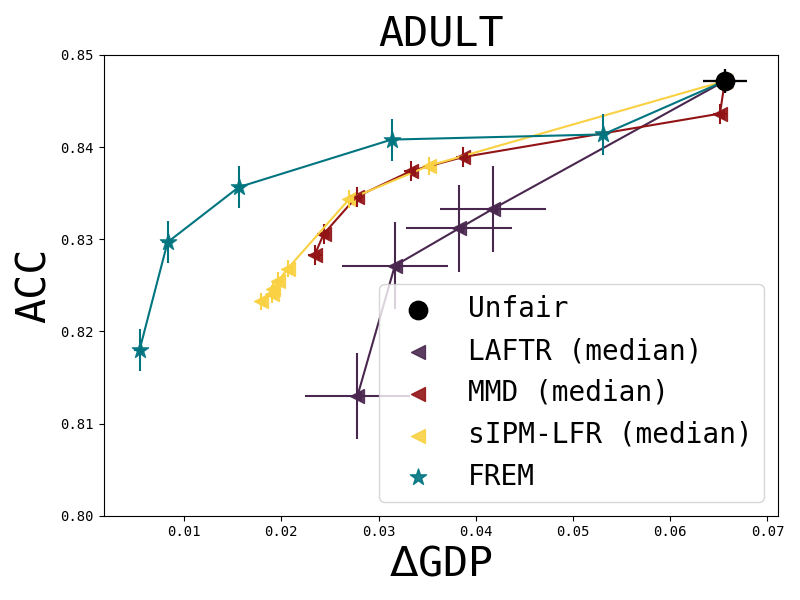}
    \includegraphics[width=0.23\textwidth]{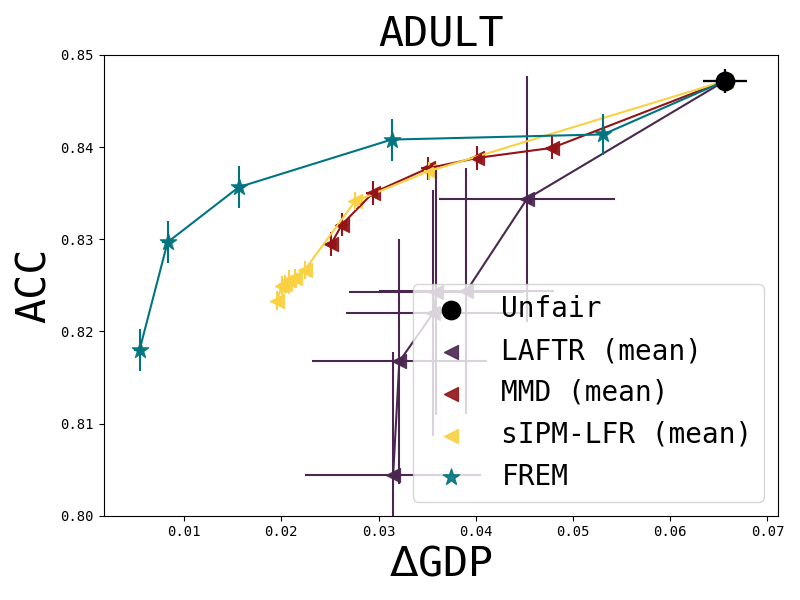}
    }
    \\~\\
    \fbox{
    \includegraphics[width=0.23\textwidth]{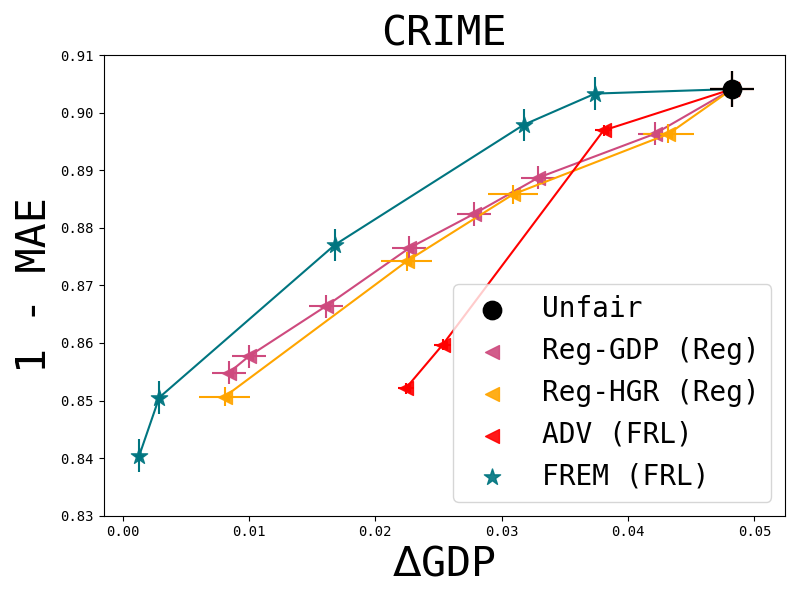}
    }
    \fbox{
    \includegraphics[width=0.23\textwidth]{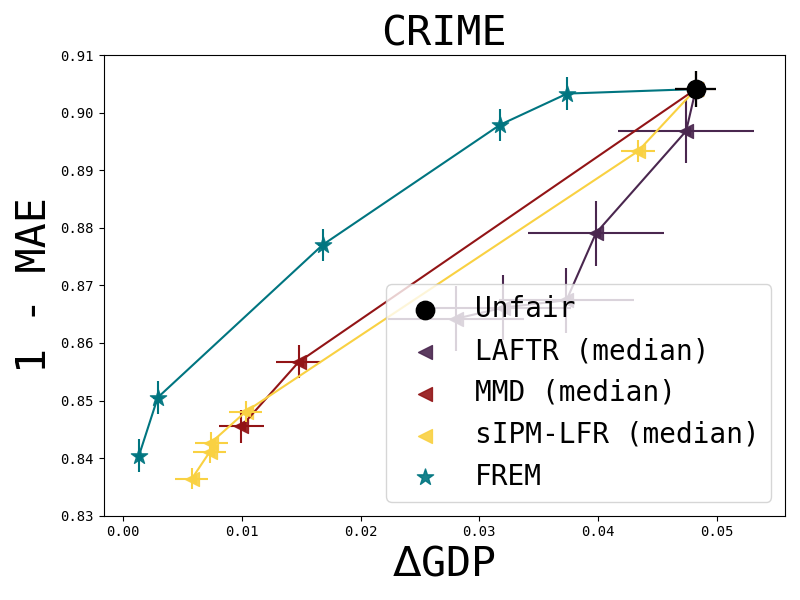}
    \includegraphics[width=0.23\textwidth]{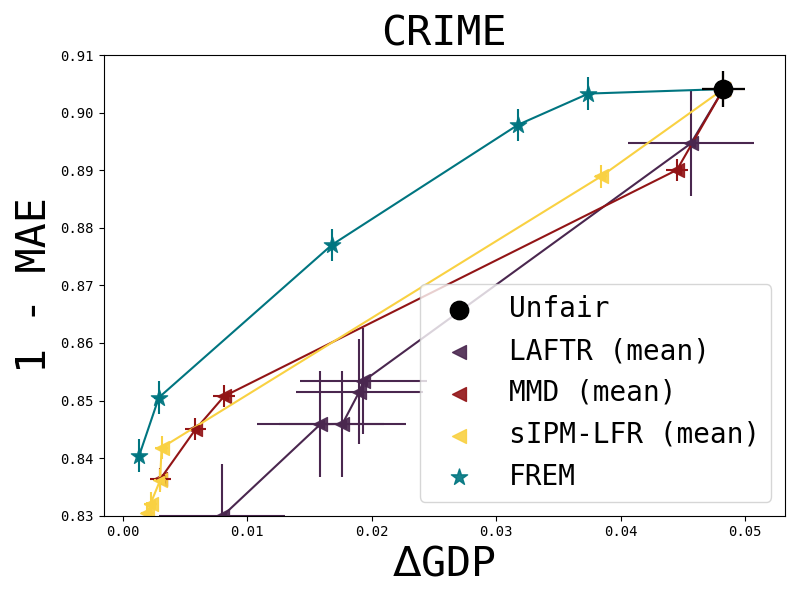}
    }
    \caption{\textbf{Demographic Parity}: Pareto-front lines for fairness-prediction trade-off.
    (Top)
    \textsc{Adult} dataset, $\Delta \texttt{GDP}$ vs. \texttt{ACC}.
    (Bottom)
    \textsc{Crime} dataset, $\Delta \texttt{GDP}$ vs. 1 - \texttt{MAE}.
    $\bullet$: Unfair,
    \textcolor{Regkernelcolor}{--$\blacktriangleleft$--}: {Reg-GDP},
    \textcolor{orange}{--$\blacktriangleleft$--}: {Reg-HGR},
    \textcolor{red}{--$\blacktriangleleft$--}: ADV,
    \textcolor{sIPMcolor}{--$\blacktriangleleft$--}: sIPM-LFR,
    \textcolor{MMDcolor}{--$\blacktriangleleft$--}: MMD,
    \textcolor{LAFTRcolor}{--$\blacktriangleleft$--}: LAFTR,
    \textcolor{Ourcolor}{--$\mathbf{\filledstar}$--}: FREM.
    }
    \label{fig:tradeoff-dp}
    \vskip -0.2in
\end{figure*}

\textbf{Implementation details}
For the encoder network $h,$ we adopt a two-layer neural network with $m=50$, the selu activation \cite{NIPS2017_5d44ee6f} and the size of nodes at the hidden layer being $50$.
For the prediction head, we use a linear layer on the $50$-dimensional representation space.
This architecture is consistent with those in the previous studies that have dealt with continuous sensitive attributes \cite{mary2019fairness, jiang2021generalized}.
For FREM, we use the RBF kernel function $\kappa$ with scale parameter $\sigma = 1.0$ after the max-min scaling of input data.

We split the entire dataset randomly into $80\% / 20\%$ for training/test datasets, and 
repeat it five times.
The average performance with the standard error is calculated on the test dataset over the five trials.
Additionally, we randomly extract $20\%$ from the training dataset for the validation dataset to select the bandwidth $\gamma.$
After selecting the bandwidth, we add the validation data back into the training data, and 
train a fair model with the FREM algorithm. 

In all cases, we use the min-max scaling to standardize $S$ to the range $[0, 1],$ and train the networks for $200$ epochs, after which we evaluate the performance of the model on the test data.
For more details with \texttt{Pytorch}-style pseudo-code, refer to Appendix \ref{appen:impl}.

\subsubsection{Comparison with existing fair algorithms for continuous sensitive attributes}\label{exp:vsregs}

For baseline methods, we consider Reg-GDP \cite{jiang2021generalized} and Reg-HGR \cite{mary2019fairness} which
are algorithms to learn fair prediction models with respect to a given continuous sensitive attribute.
These approaches employ specific regularizers that serve as a proxy of $\Delta \texttt{GDP}$ and $\Delta \texttt{HGR},$ respectively.
In addition, we consider an adversarial learning approach for continuous sensitive attributes, ADV, which is an ad-hoc modification of \cite{10.1145/3278721.3278779}.
ADV trains an encoder to make it difficult to predict $S$ from $\bm{Z}.$
Details of these algorithms are provided in Appendix \ref{appen:ADV}.

We display the Pareto-front lines for \textsc{Adult} and \textsc{Crime} datasets
to show the fairness-prediction trade-off, as depicted in the left side of Fig. \ref{fig:tradeoff-dp}.
FREM clearly outperforms all baseline methods consistently on both datasets.
In particular, superior performance of FREM over ADV suggests that theoretical soundness 
is desirable for practical purposes.
For \texttt{AP} and \texttt{MSE}, we report the results in Fig. \ref{fig:tradeoff-dp-appen} of Appendix, which still show consistent outperformance of FREM.
For the two additional fairness measures, i.e., $\Delta \texttt{HGR}$ and $\texttt{MI}(\bm{Z}, S),$ the Pareto-front lines are presented in Table \ref{tab:other} and Fig. \ref{fig:tradeoff-hgr}-\ref{fig:tradeoff-mi2} of Appendix, whose implications are similar - FREM is superior.
The results for equal opportunity can be found in Fig. \ref{fig:tradeoff-eqopp} of Appendix \ref{appen:exp_results}.


An interesting but mysterious observation from Fig. \ref{fig:tradeoff-dp} is that FREM outperforms the regularization methods (i.e., Reg-GDP and Reg-HGR).
Note that the regularization methods learn a fair prediction model directly without considering the fairness
of the representation and thus are expected to be better for prediction. 
In fact, this conjecture is true for training data.
Fig. \ref{fig:tradeoff-training} of Appendix \ref{frem_vs_gdp} shows that Reg-GDP has a lower value of the training loss than that of FREM.
Higher training performance but lower test performance of the regularization methods suggests that overfitting occurs on the tabular datasets.

In contrast, for the graph datasets, FREM and the regularization methods perform similarly
on both the test and training datasets
(see Fig. \ref{fig:graph_trade_off}-\ref{fig:graph_trade_off3} in Appendix \ref{sec:appen-graph_data} for the test datasets and Fig. \ref{fig:tradeoff-training_graph} in Appendix \ref{frem_vs_gdp} for the training datasets).
These results suggest that overfitting would not occur for the graph datasets.
Note that tabular datasets are usually simpler objects than graph datasets and thus overfitting would occur more easily for tabular datasets.

Along with prediction performance, we compare the fairness of the representations learned by FREM and the regularization methods, because the fairness of the representation
is related to the fairness of downstream tasks as proved in Theorem \ref{GDP_upper}.
As shown in Fig. \ref{fig:mi_z_s_reg} and \ref{fig:mi_z_s_reg2} in Appendix \ref{frem_vs_reg_mi_z_s}, FREM achieves better trade-offs between $\texttt{MI}(\bm{Z}, S)$ (i.e., fairness level of representation) and accuracy, outperforming the regularization methods on both tabular and graph datasets.

\subsubsection{Comparison with existing FRL methods for binary sensitive attributes}\label{exp:vsfrls}

We compare FREM with state-of-art FRL algorithms designed for binary sensitive attributes.
The baseline FRL methods for binary sensitive attributes considered in the experiments are LAFTR \cite{Madras2018LearningAF}, MMD \cite{deka2023mmd}, and sIPM-LFR \cite{kim2022learning}.
We categorize the given continuous sensitive attribute into a binary one ($0$ and $1$) by binning, and apply the three FRL methods on the 
binned binary sensitive attributes.
The purpose of this comparison is to demonstrate that FRL for binary sensitive attributes do not generalize well to continuous sensitive attributes, which indicates that fair algorithms specifically designed for continuous sensitive attributes are necessary and FREM is such a learning algorithm.
  
The right side of Fig. \ref{fig:tradeoff-dp} shows that FREM outperforms binary FRL methods with large margins. 
For binary FRL methods, we observe that $\Delta \texttt{GDP}$ in \textsc{Adult} dataset is not reduced further after a certain level of fairness. 
Even though this result is not surprising since reducing $\Delta \texttt{DP}$ does not guarantee to reduce $\Delta \texttt{GDP},$ it amply demonstrates that binary FRL algorithms are not suitable for continuous sensitive attributes.
We present the comparison results for \texttt{AP} and \texttt{MSE} in Fig. \ref{fig:tradeoff-dp-appen} of Appendix, which still shows the outperformance of FREM.

\begin{figure}[ht]
\vskip -0.05in
    \centering
    \includegraphics[scale=0.20]{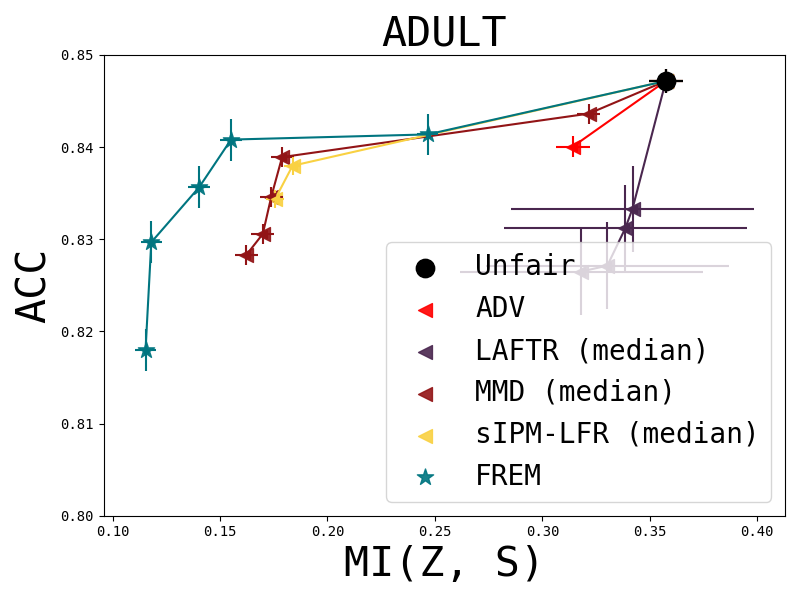}
    \includegraphics[scale=0.20]{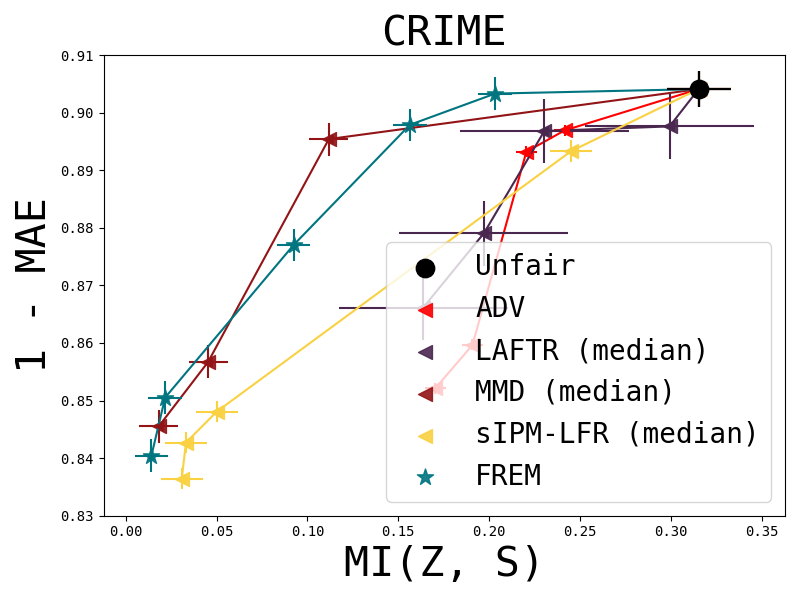}
    \caption{
    \textbf{Comparison of FRL methods in terms of fairness of learned representations}: Mutual information between $\bm{Z}$ and $S$ for five FRL methods.
    (Left) \textsc{Adult} (Right) \textsc{Crime}.
    Categorized by median for LAFTR, MMD, and sIPM-LFR.
    }
    \label{fig:compare_mi_z_s}
\end{figure}

Not only focusing on the fairness of the final prediction $\hat{Y},$ 
we also investigate how fair the learned representation $\bm{Z}$ is.
We consider the Mutual Information (MI) \cite{MacKay2003} as the measure of fairness of the learned representation.
We compare FREM with four FRL methods in terms of the trade-off between the prediction performance (e.g., \texttt{Acc}) and  fairness of the representation (i.e., $\texttt{MI}(\bm{Z}, S)$). 
Fig. \ref{fig:compare_mi_z_s} clearly shows that FREM is good at learning fair representations.

\subsubsection{Comparison with FRL methods with multinary sensitive attributes}\label{exp:vsmmds}

One may argue that using existing FRL methods with a binned multinary sensitive attribute would
improve the FRL methods with the binned binary sensitive attribute. For FRL methods involving
adversarial learning (i.e., the maximization step with respect to discriminators)
such as LAFTR \cite{Madras2018LearningAF} and sIPM-LFR \cite{kim2022learning}, however, is computationally 
demanding and unstable since multiple adversarial learnings, each of which corresponds to each category of a multinary sensitive attribute, are required. Thus, we only consider  MMD \cite{deka2023mmd} with a binned multinary sensitive attribute
as a competitor of FREM.
For a given number of bins $J,$ we first make $J$-many bins
$b_{1}, \cdots, b_{J}$ based on the corresponding quantiles.
Then, we use the fairness regularization term defined by $\sum_{j=1}^{J} \textup{MMD}(\mathbb{P}_{\bm{Z}}, \mathbb{P}_{\bm{Z}|S \in b_{j}}).$
Note that the main difference between FREM and MMD with the binned multinary sensitive attribute
is whether the kernel smoothing is used or not.

\begin{figure}[h]
    \centering
    \includegraphics[width=0.47\linewidth]{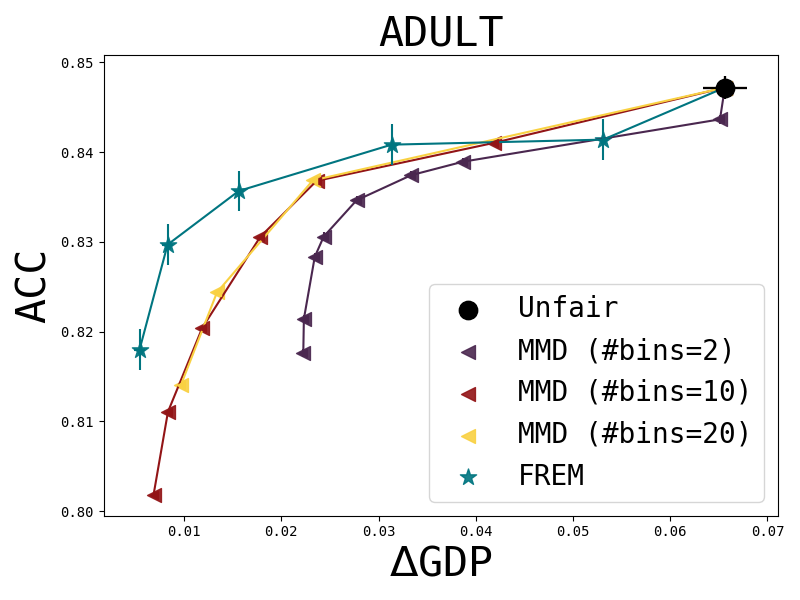}
    \includegraphics[width=0.47\linewidth]{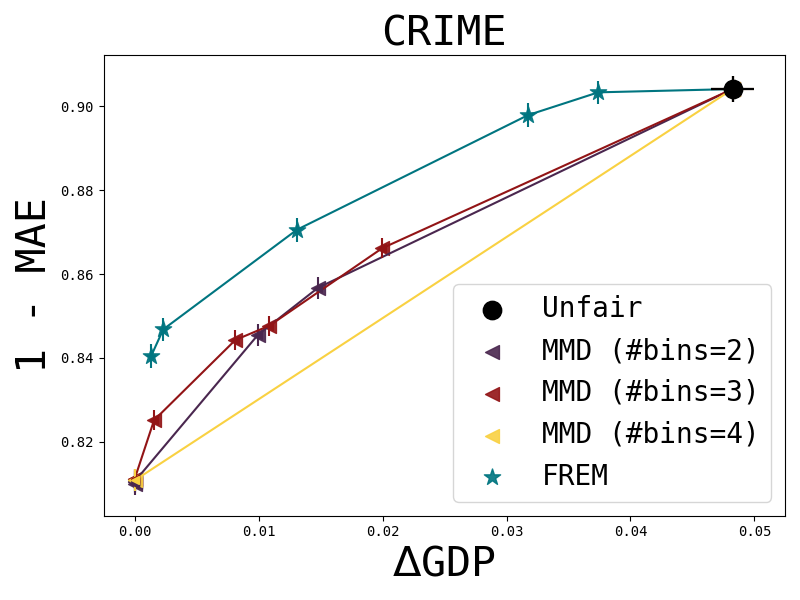}
    \caption{ {Comparison between FREM and MMD with various numbers of bins.
    (Left) \textsc{Adult}, (Right) \textsc{Crime}.}}
    \label{fig:mmds}
    \vskip -0.1in
\end{figure}

The results are presented in Fig. \ref{fig:mmds}, which show that FREM still outperforms MMD with multinary sensitive attributes.
Furthermore, the performance of MMD is not improved even when the number of bins exceed a specific value (i.e., 10 for \textsc{Adult}, 3 for \textsc{Crime}), which
implies that our proposed smoothing technique is necessary.

\subsection{Ablation studies for the choice of kernel functions}\label{exp:vary_kernel}

\textbf{(1)  Choice of kernel $K_{\gamma}$:}
As theoretically discussed in Section \ref{subsec_defEIPM_est}, any kernel function satisfying Assumption \ref{assum_kernel} can be employed for $K_{\gamma}.$
To investigate how much the choice of kernel affect finite sample performances,
we compare the three kernels: (i) RBF, (ii) Triangular, and (iii) Epanechnikov, in terms of the fairness-prediction trade-off.
The results are presented in Fig. \ref{fig:vary_kernel}, which show that the influence of the choice of kernel is minimal.
Auxillary results with other prediction measures (i.e., \texttt{AP} and \texttt{MSE}) and fairness measures (i.e., $\Delta \texttt{HGR}$ and \texttt{MI}) are given in Fig. \ref{fig:tradeoff-kernel_1} and \ref{fig:tradeoff-kernel_2} of Appendix.

\begin{figure}[ht]
    \vskip -0.1in
    \centering
    \includegraphics[scale=0.20]{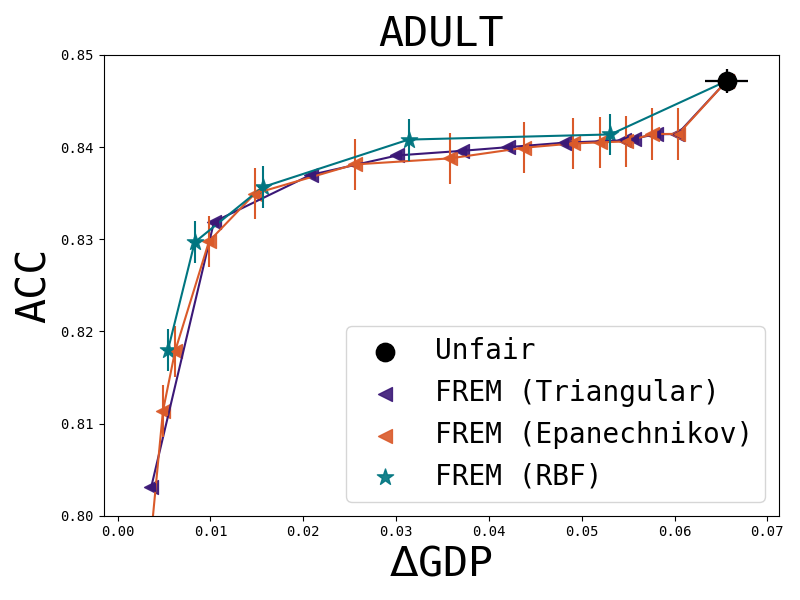}
    \includegraphics[scale=0.20]{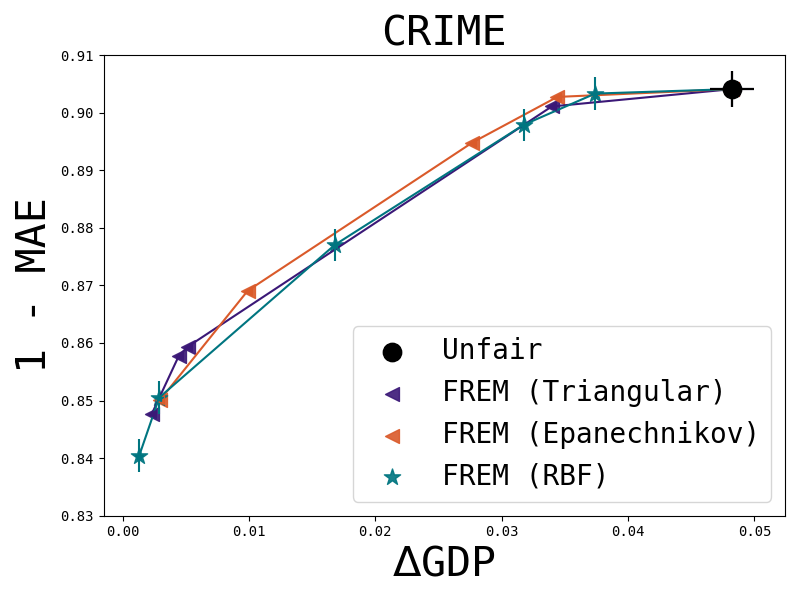}
    \caption{ \textbf{FREM with various kernels}: Comparison of kernels $K_{\gamma}.$
    (Left) \textsc{Adult} (Right) \textsc{Crime}.
    $\bullet$: Unfair,
    \textcolor{Trianglecolor}{--$\blacktriangleleft$--}: FREM with Triangular kernel,
    \textcolor{Epanecolor}{--$\blacktriangleleft$--}: FREM with Epanechnikov kernel,
    \textcolor{Ourcolor}{--$\mathbf{\filledstar}$--}: FREM with RBF kernel.
    }
    \label{fig:vary_kernel}
    \vskip -0.1in
\end{figure}

\textbf{(2) Choice of Kernel $\kappa$ in MMD:}
In addition, we analyze the impact of the choice of kernel used in MMD.
Fig. \ref{fig:tradeoff-kernel_z_1} and \ref{fig:tradeoff-kernel_z_2} of Appendix show that the three 
aforementioned kernels offer similar performances overall. 
That is, FREM is also robust to the choice of kernel in MMD.

\textbf{(3) Choice of the scale parameter $\sigma$ in MMD}
Recall that the results in Sections \ref{exp:vsregs}, \ref{exp:vsfrls}, \ref{exp:vsmmds} are obtained with fixed $\sigma=1.0$. 
To investigate the sensitivity of the FREM to the choice of $\sigma$,
we evaluate the performances of FREM with various values of $\sigma.$
Fig. \ref{fig:bandwidths_rep} of Appendix indicates that that performance of FREM is not sensitive to choice of $\sigma$ in MMD.
Particularly, the choices of $\sigma$ within an appropriate range, such as $[0.8, 1.5],$ yield similar results.

\subsection{Implications of the numerical experiments} \label{sec:expimp}

We can summarize the implications obtained from the numerical studies as follows.

\textbf{1.} The proposed EIPM estimator works well while estimators obtained by the simple binning technique
     are inferior, which implies that the smoothing technique is necessary for accurate estimation of EIPM.  
     
\textbf{2.} FREM with the estimated EIPM learns fair representations successfully. In particular, FREM dominates
      Reg-GDP and Reg-HGR which estimate fair prediction models without learning fair representations, which suggests that FRL is an useful regularization for learning fair models.

\textbf{3.} Existing FRL methods with binned sensitive attributes are not competitive to FREM, which confirms
   that the smoothing technique is a key in the success of FREM.

\textbf{4.} The performance of FREM is not sensitive to the choice of the kernels in EIPM and MMD, and hence
     can be used in practice without much difficulty.

\section{Discussion}
\label{conclusion}

There are several possible future works related to FREM.
We only consider one-dimensional continuous sensitive attributes. Extensions for multivariate continuous or mix-typed (e.g., some are categorical and others are continuous) sensitive attributes would be useful. 
For multivariate sensitive attributes, we should define a fairness measure carefully because requiring all conditional
distributions are similar would be too strong. 

Another issue is to explore other scalable sets of discriminators other than MMD for FREM.
It is known that RKHS usually includes highly smooth functions and thus all the results for FREM would be only valid for smooth
prediction models. It would be useful to construct a set of discriminators such that it includes less smooth functions 
but computation of FREM is feasible.

We develop EIPM based on IPM. We do not claim that IPM is optimal for FRL. We use IPM
mainly because the estimation of EIPM is possible and feasible. 
As discussed earlier, KL or JS divergences would be good alternatives, however, at this point we do not know how to estimate them for continuous sensitive attributes.
MI is also another potential option.
However, its finite sample version (i.e., the estimator) would be computationally difficult to be used in the training phase.
In addition, its theoretical properties are largely unknown.
Searching for other fairness measures for FRL with continuous sensitive attributes is worth pursuing.

\section*{Acknowledgments}
This work was supported by the National Research Foundation of Korea (NRF) grant funded by the Korea government (MSIT) (No. 2022R1A5A7083908)
and 
National Research Foundation of Korea (NRF) grant funded by the Korea government (MSIT) (No. 2020R1A2C3A01003550). 

\bibliography{eipm}
\bibliographystyle{IEEEtran}

\section{Biography Section}
\begin{IEEEbiography}[{\includegraphics[width=1in,height=1.25in,clip,keepaspectratio]{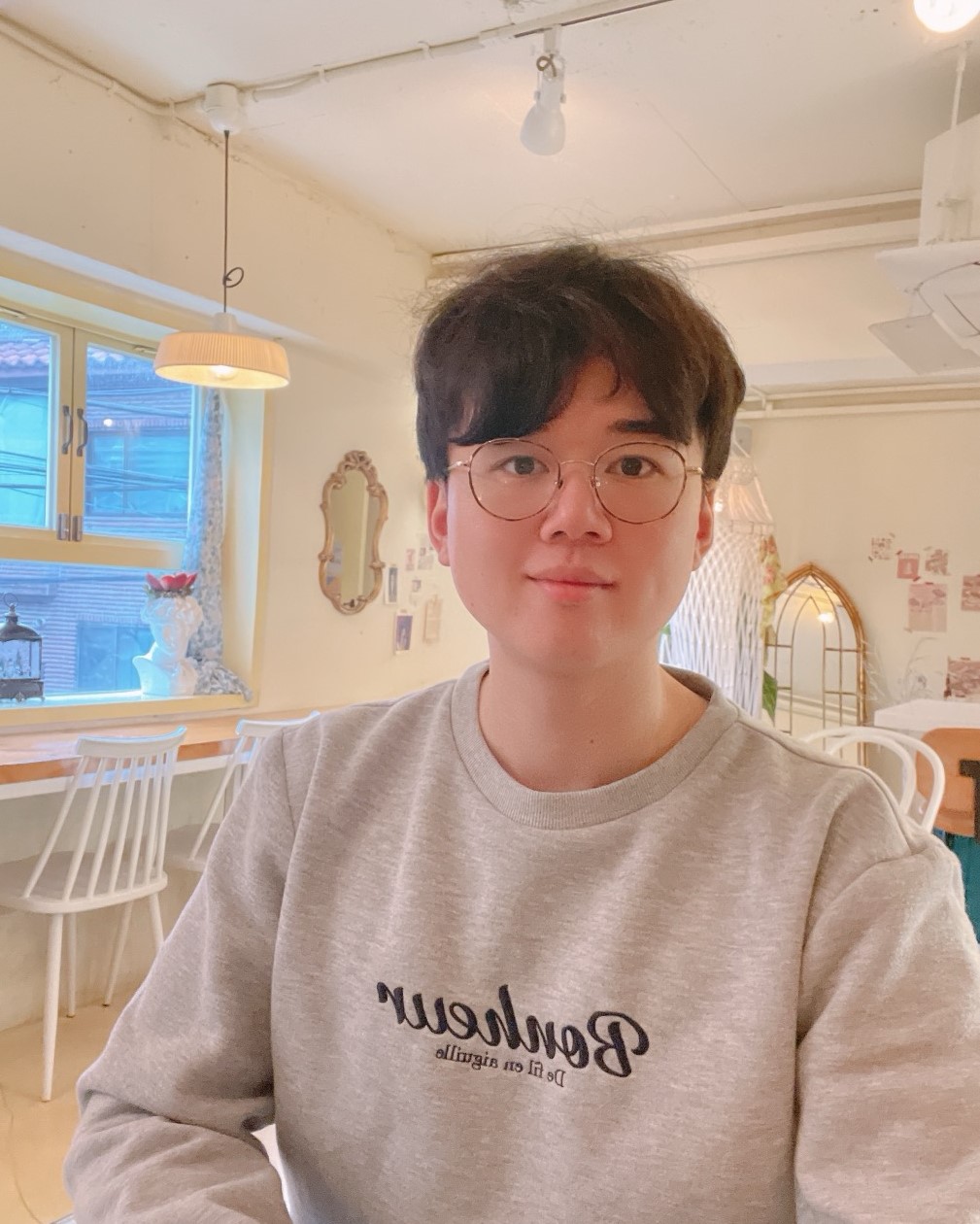}}]{Insung Kong}
received B.S. and Ph.D. degree in statistics from Seoul National University, South Korea, in  2018 and 2024, respectively. 
He is currently a postdoctoral researcher at the University of Twente, Netherlands.
His research focuses on statistical machine learning, deep learning theory, Bayesian statistics and trustworthy AI.
\end{IEEEbiography}
\vspace{-1.cm}
\begin{IEEEbiography}[{\includegraphics[width=1in,height=1.25in,clip,keepaspectratio]{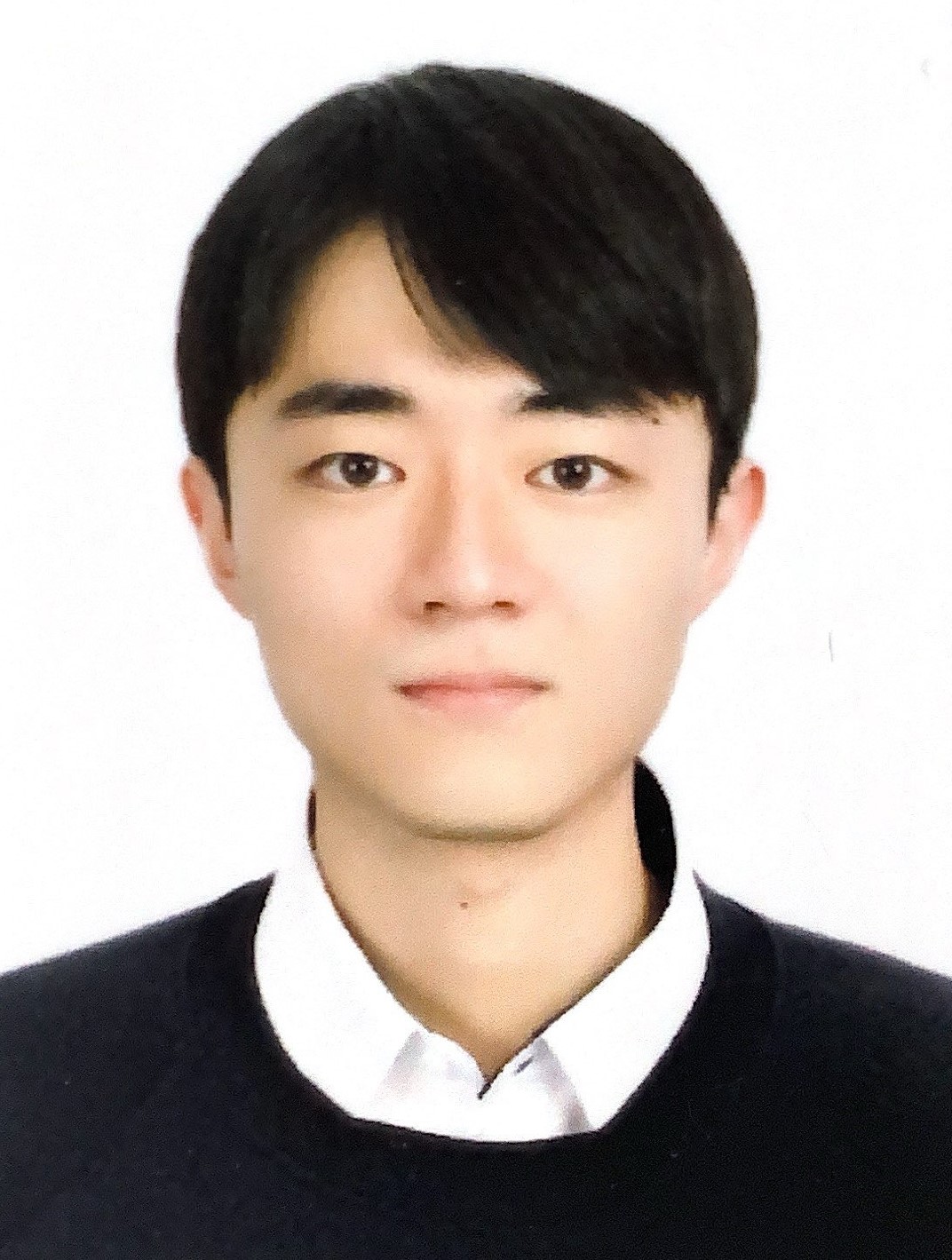}}]{Kunwoong Kim}
received B.S. a degree in energy resources engineering from Seoul National University, South Korea, in 2019. 
He is currently a Ph.D. candidate at the Department of Statistics, Seoul National University. 
His research focuses on deep representation learning and self-supervised learning.
\end{IEEEbiography}
\vspace{-1.cm}
\begin{IEEEbiography}[{\includegraphics[width=1in,height=1.25in,clip,keepaspectratio]{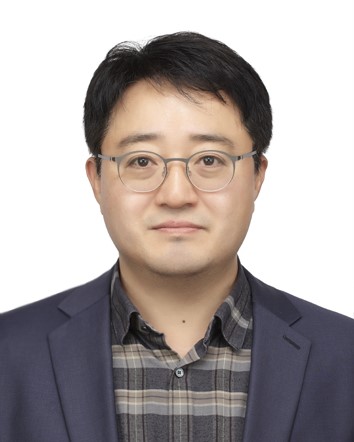}}]{Yongdai Kim} received the B.S. and M.S.
degrees in statistics from Seoul National University (SNU) in 1991 and 1993, respectively, and the Ph.D. degree of statistics from Ohio State University, Columbus, OH, U.S.A., in 1997. 
He is currently a Professor at Department of Statistics, Seoul National University. Before joining Seoul National University, he was at the National Institutes of Health (NIH), USA. 
His research area includes deep learning, machine learning and Bayesian statistics.
\end{IEEEbiography}

\vfill
\clearpage

\appendices
\onecolumn
\section{Proof for main theorems}

\numberwithin{equation}{section}
\renewcommand{\theequation}{A.\arabic{equation}}

\subsection{Additional notations and technical Lemmas}
For two positive sequences $(a_n)_{n \in \mathbb{N}}$ and $(b_n)_{n \in \mathbb{N}}$, we write $a_n = o(b_n)$ if $\lim_{n \to \infty} a_n / b_n = 0$.
We denote $\mathbb{P}_{S}$, $\mathbb{E}_{S}$ and $\mathbb{V}_{S}$ as the  probability measure, expectation and variance with respect to $S$, respectively. 
We denote
$\mathbb{P}_{\bm{X}}$, $\mathbb{E}_{\bm{X}}$, $\mathbb{P}_{\bm{Z}}$, $\mathbb{E}_{\bm{Z}}$ $\mathbb{P}_{\bm{Z},S}$, $\mathbb{E}_{\bm{Z},S}$ and $\mathbb{V}_{\bm{Z},S}$ similarly.  
We denote $\mathbb{P}^{(n)}$ as the joint probability measure of $(\bm{X}_1, S_1),\dots, (\bm{X}_n, S_n),$
where $(\bm{X}_i,S_i), i=1,\ldots,n$ are independent realizations of $(\bm{X},S).$
Also, we denote $\mathbb{P}^{(n)}_{-i}$ as the joint probability measure of $(\bm{X}_1, S_1),\dots,(\bm{X}_{i-1}, S_{i-1}),(\bm{X}_{i+1}, S_{i+1}),\dots,(\bm{X}_n, S_n)$.
For the function $k : \mathbb{R} \to \mathbb{R}$ defined on Assumption \ref{assum_kernel}, we denote $M_k := ||k||_{\infty}$ and $\kappa := \int s^2k(s)ds.$ 
Hence, we have
\begin{align*}
    \int k(s)ds &= 1 \\
    \int sk(s)ds &= 0 \\
    \int s^2k(s)ds &= \kappa.
\end{align*}

\begin{lemma}[Hoeffding inequality] \label{ineq_hoeff} 
    Let $X_1 , \dots, X_n$ be i.i.d random variables such that $a \leq X_i \leq b$ almost surely. For $M=b-a$ and for all $t>0$,
    \begin{align*}
        \mathbb{P}\left( \left| \frac{1}{n} \sum_{i=1}^n X_i - \mathbb{E}(X_1) \right| \geq \epsilon \right) \leq 2\exp \left(-\frac{2n\epsilon^2}{ M^2 } \right).
    \end{align*}    
\end{lemma}

\begin{lemma}[Bernstein inequality] \label{ineq_bern} 
    Let $X_1 , \dots, X_n$ be i.i.d random variables such that $a \leq X_i \leq b$ almost surely.
    For $M=b-a$ and for all $t>0$,
    \begin{align*}
        \mathbb{P}\left( \left| \frac{1}{n} \sum_{i=1}^n X_i - \mathbb{E}(X_1) \right| \geq \epsilon \right) \leq 2\exp \left(-\frac{n\epsilon^2}{ 2\mathbb{V}(X_1) + \frac{1}{3} M \epsilon } \right).
    \end{align*}    
\end{lemma}

\subsection{Proof for Theorem \ref{pro_basic}}
\begin{proof}
    From $\textup{EIPM}_\mathcal{V} (\bm{Z}; S) = \mathbb{E}_{S} \left[ \textup{IPM}_\mathcal{V} (\mbP_{\bm{Z}|S} ,\mbP_{\bm{Z}}) \right] = 0$, we have $\textup{IPM}_\mathcal{V} (\mbP_{\bm{Z}|S} ,\mbP_{\bm{Z}}) = 0$ almost surely (with respect to probability of $S$).
Since $\textup{IPM}_\mathcal{V}$ is a metric on the probability space of $\mathcal{Z}$, we have   
$\mbP_{\bm{Z}|S} \equiv \mbP_{\bm{Z}}$ almost surely.
Hence, for any bounded prediction head $f$, we get
\begin{align*}
    \Delta \textup{\texttt{GDP}}(f \circ h) =&  \mathbb{E}_{S} \Big| \mathbb{E}_{\bm{Z} } ( f(\bm{Z})  | S ) - \mathbb{E}_{\bm{Z}} ( f(\bm{Z}) ) \Big| \\
    \leq & \int_{S} \int_{\bm{Z}} \left|f(\bm{Z}) 
    \left(d\mbP_{\bm{Z}|S} - d\mbP_{\bm{Z}}\right) \right| d\mbP_S \\
    = & 0.
\end{align*}
\end{proof}

\subsection{Proof for Theorem \ref{GDP_upper}}
\begin{proof}
1) For any $f \in \mathcal{F}$, there exists $v \in \mathcal{V}$ such that 
$|| f - v ||_{\infty} \leq \kappa$. Then,
    \begin{align*}
     \Delta \textup{\texttt{GDP}}(f \circ h) 
     & =  \mathbb{E}_{S} \ \left| \mathbb{E}_{\bm{Z}}( f(\bm{Z})  | S) - \mathbb{E}_{\bm{Z}} ( f(\bm{Z}) ) \right| \\
    & =  \mathbb{E}_{S } \left| \int f(\bm{z}) d{\mbP}_{\bm{Z}|S}(\bm{z}) - \int f(\bm{z}) d \mbP_{\bm{Z}}(\bm{z})   \right| \\
     & \leq \mathbb{E}_{S } \left| \int v(\bm{z}) d{\mbP}_{\bm{Z}|S}(\bm{z}) - \int v(\bm{z}) d \mbP_{\bm{Z}}(\bm{z})   \right| + 2\kappa \\
    & \leq  \mathbb{E}_{S} \sup_{v \in \mathcal{V}} \left|  
    \int v(\bm{z}) d \mbP_{\bm{Z}|S}(\bm{z}) - \int v(\bm{z}) d \mbP_{\bm{Z}}(z)   \right| + 2\kappa \\
    & =  \textup{EIPM}_\mathcal{V} (\bm{Z} ; S) + 2\kappa.
\end{align*}

2) For any $f \in \mathcal{F}$,
    \begin{align*}
     \Delta \textup{\texttt{GDP}}(f \circ h) 
     & =  \mathbb{E}_{S} \ \left| \mathbb{E}_{\bm{Z}}( f(\bm{Z})  | S) - \mathbb{E}_{\bm{Z}} ( f(\bm{Z}) ) \right| \\
    & =  \mathbb{E}_{S } \left| \int f(\bm{z}) d{\mbP}_{\bm{Z}|S}(\bm{z}) - \int f(\bm{z}) d \mbP_{\bm{Z}}(\bm{z})   \right| \\
    & \leq  \mathbb{E}_{S} \sup_{f \in \mathcal{F}} \left|  
    \int f(\bm{z}) d \mbP_{\bm{Z}|S}(\bm{z}) - \int f(\bm{z}) d \mbP_{\bm{Z}}(z)   \right|  \\
    & =  \mathbb{E}_{S} \textup{IPM}_{\cF}(\mbP_{\bm{Z}|S} ,\mbP_{\bm{Z}})  \\
    & \leq  \mathbb{E}_{S} \xi(\textup{IPM}_{\cV}(\mbP_{\bm{Z}|S} ,\mbP_{\bm{Z}}))  \\ 
    & \leq  \xi(\textup{EIPM}_\mathcal{V} (\bm{Z} ; S)),
\end{align*}
where the last inequality holds by the Cauchy inequality.
\end{proof}

\subsection{Proof for Proposition \ref{pro_estimator_derivation}}
\begin{proof}
    Lemma 6 of \cite{gretton2012kernel} states that for independent random variables $U, U^{\prime} \sim \mbP^U$
    and for independent random variables $T, T^{\prime} \sim \mbP^T$, 
$$\textup{IPM}_{\mathcal{V}_{\kappa,1}}(\mbP^{U},\mbP^{T}) = \sqrt{\mathbb{E} \Big[ \kappa (U, U^{\prime}) + \kappa (T, T^{\prime}) - 2 \kappa (U, T)\Big]},$$
where the expectation is respect to $U$, $U^{\prime}$, $T$ and $T^{\prime}$.
Hence, for $\sigma>0$ and $\gamma>0$, we obtain
\begin{align*}
    \hat{\textup{EIPM}}^{\gamma}_{\mathcal{V}_{\kappa, 1}} (\bm{Z}; S)
    &= \frac{1}{n} \sum_{i=1}^n  \textup{IPM}_{\mathcal{V}_{\kappa,1}} \left(\hat{\mbP}^{(-i), \gamma}_{\bm{Z}|S=S_i} ,\hat{\mbP}^{(-i)}_{\bm{Z}}\right) \\
    &= \frac{1}{n} \sum_{i=1}^n \left[ 
    \sum_{j,k \neq i} \hat{w}_\gamma (j;i) \hat{w}_\gamma (k;i) \kappa(\bm{Z}_j, \bm{Z}_k) 
    + \sum_{j,k \neq i} \frac{1}{(n-1)^2} \kappa(\bm{Z}_j, \bm{Z}_k) 
    - 2 \sum_{j,k \neq i} \frac{1}{n-1} \hat{w}_\gamma (j;i) \kappa(\bm{Z}_j, \bm{Z}_k) \right]^{\frac{1}{2}} \\
    &= \frac{1}{n} \sum_{i=1}^n \left[ 
    \sum_{j,k \neq i} \left( \hat{w}_\gamma (j;i) \hat{w}_\gamma (k;i) 
    +\frac{1}{(n-1)^2}  
    - \frac{2}{n-1} \hat{w}_\gamma (j;i) \right)  \kappa(\bm{Z}_j, \bm{Z}_k) \right]^{\frac{1}{2}} \\
    &= \frac{1}{n} \sum_{i=1}^n \left[ 
    \sum_{j,k \neq i} \left( \hat{w}_\gamma (j;i) \hat{w}_\gamma (k;i) 
    +\frac{1}{(n-1)^2}  
    - \frac{1}{n-1} \hat{w}_\gamma (j;i) - \frac{1}{n-1} \hat{w}_\gamma (k;i) \right)  \kappa(\bm{Z}_j, \bm{Z}_k) \right]^{\frac{1}{2}} \\
    &= \frac{1}{n} \sum_{i=1}^n \left[ 
    \sum_{j,k \neq i} \left( \hat{w}_\gamma (j;i) - \frac{1}{n-1} \right) 
    \left( \hat{w}_\gamma (k;i) - \frac{1}{n-1} \right)  
    \kappa(\bm{Z}_j, \bm{Z}_k) \right]^{\frac{1}{2}} \\
    &= \frac{1}{n} \sum_{i=1}^n \left[ \sum_{j,k \neq i} [A_{\gamma}]_{i,j} [A_{\gamma}]_{i,k}  \kappa(\bm{Z}_j, \bm{Z}_k) 
    \right]^{\frac{1}{2}},
\end{align*}
where we use the fact that $\kappa(\bm{Z}_j, \bm{Z}_k) = \kappa(\bm{Z}_k, \bm{Z}_j)$ for the fourth equality.
\end{proof}

\subsection{Proof for Theorem \ref{EIPM_conv}} \label{sec_proof_eipmconv}

\begin{lemma} \label{lemma_aboutp}
Under Assumptions \ref{assum_kernel}, \ref{assum_ps}, \ref{assum_pzs} and \ref{assum_gamma},
    \begin{enumerate}
        \item $p(s)$ is twice differentiable and has bounded second derivative.
        
        \item For any $s \in \mathcal{S}$, we have
        $$\frac{1}{\gamma_n} \mathbb{E}_{S} (K_{\gamma_n}(S, s)) = p(s) + \frac{\kappa p''(s)}{2}\gamma_n^2 + o(\gamma_n^2)$$
        and        
        $$\frac{1}{\gamma_n } \mathbb{E}_{S} \left[ K_{\gamma_n}(S , s)^2\right] \leq  p(s) + \frac{\kappa p''(s)}{2}\gamma_n^2 + o(\gamma_n^2) $$
        for sufficiently large $n$.
        
        \item For any $s \in \mathcal{S}$, we have
        \begin{align*}
        \left| \frac{1}{\gamma_n} \frac{1}{n-1} \sum_{j \neq i} K_{\gamma_n}(S_j , s) - p(s) \right| 
        \leq \frac{1}{\sqrt{n \gamma_n}} \log(n) + \kappa|p''(s)|\gamma_n^2
   \end{align*}
   for sufficiently large $n$ with probability at least $1-\frac{1}{n^2}$. 
    \end{enumerate}
\end{lemma}
\begin{proof}
    \begin{enumerate}
        \item We have
    \begin{align*}
        \left| p''(s) \right| 
        =& \left| \int_{\bm{x} \in \mathcal{X}} \frac{\partial^2}{\partial s^2} p(\bm{x}, s) d\mu_{\bm{X}}(\bm{x}) \right|  \\
        \leq & \int_{\bm{x} \in \mathcal{X}} \left|\frac{\partial^2}{\partial s^2} p(\bm{x}, s) \right| d\mu_{\bm{X}}(\bm{x}),    
    \end{align*}
    which implies that $p(s)$ is twice differentiable and has a bounded second derivative by Assumption \ref{assum_pzs}.
    
        \item By applying Taylor expansion to $p(s+t\gamma_n)$, we obtain
        \begin{align*}
        \frac{1}{\gamma_n } \mathbb{E}_{S} (K_{\gamma_n}(S, s))
        & =  \mathbb{E}_{S} \left[ \frac{1}{\gamma_n } k\left(\frac{S - s}{\gamma_n}\right) \right] \nonumber \\
        & =  \int \frac{1}{\gamma_n } k\left(\frac{u - s}{\gamma_n}\right) p(u) du \nonumber \\
        & = \int k(t) p(s+t\gamma_n) dt \nonumber \\
        & = \int k(t) \left[ p(s) + t\gamma_n p'(s) + (t\gamma_n)^2 \frac{p''(s)}{2}  \right] dt  + o(\gamma_n^2)  \nonumber \\
        & = p(s) + \frac{\kappa p''(s)}{2}\gamma_n^2 + o(\gamma_n^2).
    \end{align*}
        Similarly, we have
        \begin{align*}
        \frac{1}{\gamma_n } \mathbb{E}_{S} \left[ K_{\gamma_n}(S , s)^2\right] 
        & = \frac{1}{\gamma_n}  \mathbb{E}_{S} \left( k\left(\frac{S - s}{\gamma_n}\right)^2 \right) \\
        & \leq \frac{1}{\gamma_n} \mathbb{E}_{S} \left( k\left(\frac{S - s}{\gamma_n}\right) \right) \\
        & \leq  
        \int \frac{1}{\gamma_n} k\left(\frac{u - s}{\gamma_n}\right) p(u) du \\ 
        & = p(s) + \frac{\kappa p''(s)}{2}\gamma_n^2 + o(\gamma_n^2). 
    \end{align*} 
    
        \item Define $\hat{p}_n(s)$ as
        $$\hat{p}_n(s) := \frac{1}{\gamma_n} \frac{1}{n-1} \sum_{j \neq i} K_{\gamma_n}(S_j , s).$$        
        From the fact $0 \leq \frac{1}{\gamma_n} K_{\gamma_n}(S_j , s) \leq \frac{M_k}{\gamma_n}$, we have
    \begin{align*}
        &\mbP\Bigg( \left|\hat{p}_n(s) - \mathbb{E}_{S} \left[\frac{1}{\gamma_n} K_{\gamma_n}(S , s)\right] \right| 
        > \frac{2 M_k }{3(n-1)\gamma_n}\log(2n^2) + \sqrt{ \mathbb{V}_{S} \left[\frac{1}{\gamma_n} K_{\gamma_n}(S , s)\right] \frac{4\log(2n^2)}{n-1}}
        \Bigg) 
        \leq \frac{1}{n^2}
    \end{align*}
    by the Bernstein's inequality (Lemma \ref{ineq_bern}). Also, we have
    $$\mathbb{E}_{S} \left[\frac{1}{\gamma_n} K_{\gamma_n}(S , s)\right] = p(s) + \frac{\kappa p''(s)}{2}\gamma_n^2 + o(\gamma_n^2)$$
    and
    \begin{align*}
        \mathbb{V}_{S} \left[\frac{1}{\gamma_n} K_{\gamma_n}(S , s)\right]
        & = \frac{1}{\gamma_n^2 } \mathbb{V}_{S} \left[ K_{\gamma_n}(S , s)\right] \nonumber \\
        & \leq \frac{1}{\gamma_n^2 } \mathbb{V}_{S} \left[ K_{\gamma_n}(S , s)^2 \right] \nonumber \\
        & \leq \frac{1}{\gamma_n } \left(p(s) + \frac{\kappa p''(s)}{2}\gamma_n^2 + o(\gamma_n^2)\right). 
    \end{align*}     
    To sum up, we have
   \begin{align*}
   \left| \frac{1}{\gamma_n} \frac{1}{n-1} \sum_{j \neq i} K_{\gamma_n}(S_j , s) - p(s) \right| 
   = & | \hat{p}_n(s) - p(s) | \\
        \leq& \left| \hat{p}_n(s) - \mathbb{E}_{S} \left[\frac{1}{\gamma_n} K_{\gamma_n}(S , s)\right] \right| 
        + \left| \mathbb{E}_{S} \left[\frac{1}{\gamma_n } K_{\gamma_n}(S , s)\right] - p(s) \right| \nonumber \\
        \leq &
        \frac{1}{\sqrt{n \gamma_n}} \log(n) + \kappa |p''(s)|\gamma_n^2
   \end{align*}
   for sufficiently large $n$ with probability at least $1-\frac{1}{n^2}$.
    \end{enumerate}
\end{proof}

\begin{lemma} \label{lemma_aboutm}
    For a given encoder function $h$ and a given real-valued function $v$ such that $||v||_{\infty} \leq 1$, 
    we define $m(s)$ for $s \in \mathcal{S}$ as $$m(s) := \mathbb{E}_{\bm{X}} (v \circ h(\bm{X})|S=s).$$
    Under Assumptions \ref{assum_kernel}, \ref{assum_ps}, \ref{assum_pzs} and \ref{assum_gamma},
    \begin{enumerate}
        \item $m(s)$ is twice differentiable and has a bounded second derivative.
        \item For any $s \in \mathcal{S}$, we have 
        \begin{align*}
            \frac{1}{\gamma_n} 
            \mathbb{E}_{\bm{Z},S} (  K_{\gamma_n}(S, s) v(\bm{Z}))
            = m(s)p(s) + \kappa \gamma_n^2 \left( \frac{m''(s)p(s)}{2} + \frac{m(s)p''(s)}{2} + m'(s)p'(s)  \right) + o(\gamma_n^2).
   \end{align*}
    \end{enumerate}
\end{lemma}

\begin{proof}
\begin{enumerate}
    \item We have
\begin{align*}
    m(s) & =\int_{\bm{x} \in \mathcal{X}} v(h(\bm{X})) d\mbP_{\bm{X}|S=s} \\
    & = \frac{\int_{\bm{x} \in \mathcal{X}} v(h(\bm{X})) p(\bm{x}, s) d\mu_{\bm{X}}(\bm{x})}
    {\int_{\bm{x} \in \mathcal{X}} p(\bm{x}, s) d\mu_{\bm{X}}(\bm{x})} \\
    & = \frac{\int_{\bm{x} \in \mathcal{X}} v(h(\bm{X})) p(\bm{x}, s) d\mu_{\bm{X}}(\bm{x})}
    {p(s)}.
\end{align*}
We denote $A(s) := \int_{\bm{x} \in \mathcal{X}} v(h(\bm{X})) p(\bm{x}, s) d\mu_{\bm{X}}(\bm{x})$.
From
\begin{align*}
    \left| A''(s) \right| 
    =& \left| \int_{\bm{x} \in \mathcal{X}} v(h(\bm{X})) \frac{\partial^2}{\partial s^2}p(\bm{x}, s) d\mu_{\bm{X}}(\bm{x}) \right| \\
    \leq &  \int_{\bm{x} \in \mathcal{X}} \left| v(h(\bm{X})) \frac{\partial^2}{\partial s^2}p(\bm{x}, s) \right| d\mu_{\bm{X}}(\bm{x})  \\
    \leq & \int_{\bm{x} \in \mathcal{X}} \left| \frac{\partial^2}{\partial s^2}p(\bm{x}, s) \right| d\mu_{\bm{X}}(\bm{x})
\end{align*}
and Assumption \ref{assum_pzs}, $A(s)$ is twice differentiable and has a bounded second derivative.
Hence, 
\begin{align*}
    m''(s) = & \left(\frac{A(s)}{p(s)}\right)^{''} \\
    = \  & \frac{A''(s) p(s) - A(s)p''(s)}{p(s)^4} 
     - \frac{2p(s)(A'(s)p(s)-A(s)p'(s))}{p(s)^4}
\end{align*}
is also bounded by Assumption \ref{assum_ps}. 

\item By applying Taylor expansion to $m(s + t\gamma_n)$ and $p(s + t\gamma_n)$, we obtain
    \begin{align*}
        & \frac{1}{\gamma_n } \mathbb{E}_{\bm{Z},S} (  K_{\gamma_n}(S, s) v(\bm{Z})) \\
        & = \frac{1}{\gamma_n } \mathbb{E}_{S} (  K_{\gamma_n}(S, s) \mathbb{E}_{Z}(v(\bm{Z})|S)) \\
        & = \frac{1}{\gamma_n } \mathbb{E}_{S} (K_{\gamma_n}(S, s) m(S) )  \\
        &= \mathbb{E}_{S} \left[ \frac{1}{\gamma_n} k\left( \frac{S - s}{\gamma_n} \right)   m(S) \right] \\
        & = \int  \frac{1}{\gamma_n} k\left( \frac{u - s}{\gamma_n} \right)   m(u) p(u) du \\
        & = \int  k(t)  m(s + t\gamma_n) p(s + t\gamma_n) dt \\
        & = \int  k(t) \left( m(s) + t\gamma_n m'(s) + (t\gamma_n)^2 \frac{m''(s)}{2}  \right) \left( p(s) + t\gamma_n p'(s) + (t\gamma_n)^2 \frac{p''(s)}{2}  \right) dt + o(\gamma_n^2)  \\
        & = m(s)p(s) + \kappa \gamma_n^2 \left( \frac{m''(s)p(s)}{2} + \frac{m(s)p''(s)}{2} + m'(s)p'(s)  \right) 
        + o(\gamma_n^2). 
    \end{align*}     
\end{enumerate}
\end{proof}

\subsubsection{Proof for Theorem \ref{EIPM_conv}}
\begin{proof}
    For given $i \in \{1,\dots,n\}$ and given $S_i = s$, we have
    \begin{align}
        & \left |\textup{IPM}_{\mathcal{V}} (\hat{\mbP}^{(-i), \gamma_n}_{\bm{Z}|S=s} ,\hat{\mbP}^{(-i)}_{\bm{Z}}) - \textup{IPM}_{\mathcal{V}} (\mbP_{\bm{Z}|S=s} ,\mbP_{\bm{Z}}) \right| \nonumber \\
        & = \Bigg| \sup_{v \in \mathcal{V}} \left| \sum_{j \neq i}  \hat{w}_{\gamma_n}(j;i) v(\bm{Z}_j) -  \frac{1}{n-1} \sum_{j \neq i} v(\bm{Z}_j) \right| - \sup_{v \in \mathcal{V}} \Big| \mathbb{E}_{\bm{Z}}(v(\bm{Z})|S=s) -  \mathbb{E}_{\bm{Z}} v(\bm{Z}) \Big|
        \Bigg|  \nonumber \\
        & \leq  \sup_{v \in \mathcal{V}} \Bigg| \sum_{j \neq i}  \hat{w}_{\gamma_n}(j;i) v(\bm{Z}_j) -  \frac{1}{n-1} \sum_{j \neq i} v(\bm{Z}_j) - \mathbb{E}_{\bm{Z}} (v(\bm{Z})|S=s) + \mathbb{E}_{\bm{Z}} v(\bm{Z}) \Bigg| \label{for_suph_tmp1} \\
        & \leq  \sup_{v \in \mathcal{V}} \Bigg| \sum_{j \neq i}   \frac{K_{\gamma_n}(S_j, s)}{\sum_{j \neq i} K_{\gamma_n}(S_j, s)} v(\bm{Z}_j)
        - \frac{\sum_{j \neq i}  K_{\gamma_n}(S_j, s) v(\bm{Z}_j)}{ (n-1)\mathbb{E}_{S} (K_{\gamma_n}(S, s))}
        \Bigg| \label{thm_prob_1} \\
         & \quad + \sup_{v \in \mathcal{V}} \Bigg| \frac{\sum_{j \neq i}  K_{\gamma_n}(S_j, s) v(\bm{Z}_j)}{ (n-1)\mathbb{E}_{S} (K_{\gamma_n}(S, s))}
        - \frac{\mathbb{E}_{\bm{Z},S} (  K_{\gamma_n}(S, s) v(\bm{Z}) )}{\mathbb{E}_{S} (K_{\gamma_n}(S, s))}
        \Bigg| \label{thm_prob_2} \\
         & \quad + \sup_{v \in \mathcal{V}} \Bigg| \frac{\mathbb{E}_{\bm{Z},S} (  K_{\gamma_n}(S, s) v(\bm{Z})  )}{\mathbb{E}_{S} (K_{\gamma_n}(S, s))}
        - \mathbb{E}_{\bm{Z}}(v(\bm{Z})|S=s)
        \Bigg| \label{thm_prob_3} \\
        & \quad + \sup_{v \in \mathcal{V}} \left| \frac{1}{n-1} \sum_{j \neq i} v({\bm{Z}}_j) - \mathbb{E}_{\bm{Z}} v(\bm{Z}) \right|, \label{thm_prob_4}     
    \end{align}
    where the first and the second inequality hold  by the fact that $|\sup |a| - \sup |b|| \leq \sup|a-b|$ and $\sup |a+b+c+d| \leq \sup |a| + \sup |b| + \sup |c| + \sup |d|$, respectively.    
    For 
    $\epsilon_n = \gamma_n^2 + \frac{\log n}{\sqrt{n \gamma_n}} \left(1 + \log \mathcal{N} \left(\sqrt{\frac{\gamma_n}{n}}, \mathcal{V}, ||\cdot||_{\infty}\right) \right)^{\frac{1}{2}}$,
    we will show that there exist positive constants $c_1 , c_2, c_3$ and $c_4$ that do not depend on $n$, $m$ and $s$ such that  
    \begin{align*}
        \mathbb{P}^{(n)}_{-i} \left( (\ref{thm_prob_1}) > c_1 \epsilon_n \right) \leq & \ \frac{1}{n^2},\\
        \mathbb{P}^{(n)}_{-i} \left( (\ref{thm_prob_2})  > c_2 \epsilon_n \right) \leq & \ \frac{1}{n^2}, \\
        \mathbb{P}^{(n)}_{-i} \left( (\ref{thm_prob_3}) > c_3 \epsilon_n \right) = & \ 0, \\
        \mathbb{P}^{(n)}_{-i} \left( (\ref{thm_prob_4}) > c_4 \epsilon_n \right) \leq & \ \frac{1}{n^2}        
    \end{align*}
    for all sufficiently large $n$.

\textbf{Bound for (\ref{thm_prob_1})} : 
For sufficiently large $n$ with probability at least $1-\frac{1}{n^2}$, we have
\begin{align}
    (\ref{thm_prob_1}) &= \sup_{v \in \mathcal{V}} \Bigg| \sum_{j \neq i}   \frac{K_{\gamma_n}(S_j, s)}{\sum_{j \neq i} K_{\gamma_n}(S_j, s)} v(\bm{Z}_j)
        - \frac{\sum_{j \neq i}  K_{\gamma_n}(S_j, s) v(\bm{Z}_j)}{ (n-1)\mathbb{E}_{S} (K_{\gamma_n}(S, s))}
        \Bigg| \nonumber \\
    & = \left| 1 - \frac{\sum_{j \neq i} K_{\gamma_n}(S_j, s)}{(n-1)\mathbb{E}_{S} (K_{\gamma_n}(S, s))} \right| \cdot
    \sup_{v \in \mathcal{V}} \left| \frac{\sum_{j \neq i}   K_{\gamma_n}(S_j, s) v(\bm{Z}_j)}{\sum_{j \neq i} K_{\gamma_n}(S_j, s)}  \right| \nonumber \\
    & \leq \left| 1 - \frac{\sum_{j \neq i} K_{\gamma_n}(S_j, s)}{(n-1)\mathbb{E}_{S} (K_{\gamma_n}(S, s))} \right| \nonumber \\
    & = \left|\frac{\frac{1}{\gamma_n } \mathbb{E}_{S} (K_{\gamma_n}(S, s)) - 
    \frac{1}{\gamma_n } \frac{1}{n-1} \sum_{j \neq i} K_{\gamma_n}(S_j, s)
    }{ \frac{1}{\gamma_n } \mathbb{E}_{S} (K_{\gamma_n}(S, s))}\right| \nonumber \\
    & \leq \left|\frac{ \frac{\kappa |p''(s)|}{2}\gamma_n^2 + o(\gamma_n^2) + \frac{1}{\sqrt{n \gamma_n}} \log(n) + \kappa |p''(s)|\gamma_n^2
    }{p(s) + \frac{\kappa p''(s)}{2}\gamma_n^2 + o(\gamma_n^2)}\right| \nonumber \\
    & \leq \frac{2}{L_p} \cdot \left( \frac{1}{\sqrt{n \gamma_n}} \log(n) + 2 \kappa |p''(s)|\gamma_n^2 \right) \nonumber \\
    & \leq c_1 \epsilon_n , \label{thm_prob_bound1}
\end{align}
     where the first inequality holds by $||v||_{\infty} \leq 1$ and the second inequality holds by Lemma \ref{lemma_aboutp}.

\textbf{Bound for (\ref{thm_prob_2})} : we have
\begin{align*}
    (\ref{thm_prob_2}) =& \sup_{v \in \mathcal{V}} \Bigg| \frac{\sum_{j \neq i}  K_{\gamma_n}(S_j, s) v(\bm{Z}_j)}{ (n-1)\mathbb{E}_{S} (K_{\gamma_n}(S, s))}
        - \frac{\mathbb{E}_{\bm{Z},S} (  K_{\gamma_n}(S, s) v(\bm{Z}) )}{\mathbb{E}_{S} (K_{\gamma_n}(S, s))}
        \Bigg| \\
        =& \frac{1}{\mathbb{E}_{S} (K_{\gamma_n}(S, s))} \sup_{v \in \mathcal{V}} \Bigg| \frac{1}{n-1} \sum_{j \neq i}  K_{\gamma_n}(S_j, s) v(\bm{Z}_j)
        - \mathbb{E}_{\bm{Z},S} (  K_{\gamma_n}(S, s) v(\bm{Z}) )
        \Bigg| \\
        \leq & \frac{2 }{L_p \gamma_n } \sup_{v \in \mathcal{V}} \Bigg| \frac{1}{n-1} \sum_{j \neq i}  K_{\gamma_n}(S_j, s) v(\bm{Z}_j)
        - \mathbb{E}_{\bm{Z},S} (  K_{\gamma_n}(S, s) v(\bm{Z}) )
        \Bigg|, 
\end{align*}
where the inequality holds by Lemma \ref{lemma_aboutp}.
Let $M := \mathcal{N}(\sqrt{\frac{\gamma_n}{n}}, \mathcal{V}, ||\cdot||_{\infty})$, 
and let $v_1 , \dots, v_M$ be the centers of $\sqrt{\frac{\gamma_n}{n}}$-cover of $\mathcal{V}$ (They do not need to be in $\mathcal{V}$, but we assume that their infinite norm is bounded by 1).
Since $\frac{\epsilon_n}{M_k} > \frac{1}{\sqrt{n \gamma_n}}$ implies $\frac{\gamma_n \epsilon_n}{M_k}  > \sqrt{\frac{\gamma_n}{n}}$, for every $v \in \mathcal{V}$ there exists $m \in \{1,\dots,M\}$ such that $||v- v_m ||_{\infty} \leq \frac{\gamma_n \epsilon_n}{M_k} $ and hence 
\begin{align*}
    \left|\frac{1}{n-1} \sum_{j \neq i}  K_{\gamma_n}(S_j, s) v(\bm{Z}_j) - 
\frac{1}{n-1} \sum_{j \neq i}  K_{\gamma_n}(S_j, s) v_m(\bm{Z}_j) \right| \leq& \gamma_n \epsilon_n 
\end{align*}
and
\begin{align*}
    \left|\mathbb{E}_{\bm{Z},S} (  K_{\gamma_n}(S, s) v(\bm{Z}) ) - 
\mathbb{E}_{\bm{Z},S} (  K_{\gamma_n}(S, s) v_m(\bm{Z}) ) \right| \leq& \gamma_n \epsilon_n 
\end{align*}
hold.
Also, from Lemma \ref{lemma_aboutp} we have
\begin{align*}
    \mathbb{V}_{\bm{Z},S} (  K_{\gamma_n}(S, s) v_m(\bm{Z}) )
    & \leq \mathbb{E}_{\bm{Z},S} (  K_{\gamma_n}(S, s)^2 v_m(\bm{Z})^2 )\\
    & \leq \mathbb{E}_{\bm{Z},S} (  K_{\gamma_n}(S, s)^2 )\\
    & \leq 2 \gamma_n p(s). 
\end{align*}
To sum up, by choosing $c_2 := \frac{6}{L_p}$ we obtain
\begin{align}
    \mathbb{P}^{(n)}_{-i} \left( (\ref{thm_prob_2}) > c_2 \epsilon_n \right) 
    \leq & \mathbb{P}^{(n)}_{-i} \left( \sup_{v \in \mathcal{V}} \Bigg| \frac{1}{n-1} \sum_{j \neq i}  K_{\gamma_n}(S_j, s) v(\bm{Z}_j)
        - \mathbb{E}_{\bm{Z},S} (  K_{\gamma_n}(S, s) v(\bm{Z}) )
        \Bigg|
        > 3 \gamma_n \epsilon_n \right) \nonumber \\
    \leq & \mathbb{P}^{(n)}_{-i} \left( \bigcup_{m=1}^M \left\{ \Bigg| \frac{1}{n-1} \sum_{j \neq i}  K_{\gamma_n}(S_j, s) v_m(\bm{Z}_j)
        - \mathbb{E}_{\bm{Z},S} (  K_{\gamma_n}(S, s) v_m(\bm{Z}) )
        \Bigg|
        > \gamma_n \epsilon_n \right\} \right) \nonumber \\
    \leq & \sum_{m=1}^M 
    \mathbb{P}^{(n)}_{-i} \left( \Bigg| \frac{1}{n-1} \sum_{j \neq i}  K_{\gamma_n}(S_j, s) v_m(\bm{Z}_j)
        - \mathbb{E}_{\bm{Z},S} (  K_{\gamma_n}(S, s) v_m(\bm{Z}) )
        \Bigg|
        > \gamma_n \epsilon_n \right) \nonumber\\
    \leq & 2M \exp \left(-\frac{n \gamma_n^2 \epsilon_n^2   }{4 p(s) \gamma_n + \frac{2}{3} M_k \gamma_n \epsilon_n } \right) \nonumber\\
    \leq & 2M \exp \left(-\frac{1}{5 U_p} n \gamma_n \epsilon_n^2  \right) \nonumber \\
    \leq & \frac{1}{n^2}, \label{thm_prob_bound2}
\end{align}
where we use the Bernstein inequality (Lemma \ref{ineq_bern}) for the fourth inequality
and use the fact $\log M = \log \mathcal{N}(\sqrt{\frac{\gamma_n}{n}}, \mathcal{V}, ||\cdot||_{\infty}) \leq \frac{n \epsilon_n^2 \gamma_n}{\log n}$ for the last inequality.

\textbf{Bound for (\ref{thm_prob_3})} :
By Lemma \ref{lemma_aboutp} and Lemma \ref{lemma_aboutm}, 
there exist positive constants $c_{3,1}$ and $c_{3,2}$ not depending on $n$, $m$ and $s$ such that
\begin{align*}
    \left|\frac{1}{\gamma_n} \mathbb{E}_{S} (K_{\gamma_n}(S, s)) - p(s) \right| \leq c_{3,1} \gamma_n^2 
\end{align*}
and
\begin{align*}
    \left| \frac{1}{\gamma_n } \mathbb{E}_{\bm{Z},S} (  K_{\gamma_n}(S, s) v(\bm{Z})) - m(s)p(s) \right|
    \leq c_{3,2} \gamma_n^2
\end{align*}
hold. Hence, there exists a constant $c_3 > 0$ not depending on $n$, $m$ and $s$ such that
\begin{align}
    (\ref{thm_prob_3}) = & \sup_{v \in \mathcal{V}} \Bigg| \frac{\mathbb{E}_{\bm{Z},S} (  K_{\gamma_n}(S, s) v(\bm{Z})  )}{\mathbb{E}_{S} (K_{\gamma_n}(S, s))}
    - \mathbb{E}_{\bm{Z}}(v(\bm{Z})|S=s) \Bigg| \nonumber \\
    = & \sup_{v \in \mathcal{V}} \Bigg| \frac{\mathbb{E}_{\bm{Z},S} (  K_{\gamma_n}(S, s) v(\bm{Z})  )}{\mathbb{E}_{S} (K_{\gamma_n}(S, s))}
    - m(s) \Bigg| \nonumber \\
    \leq & c_3 \gamma_n^2.
    \label{thm_prob_bound3}
\end{align}
\textbf{Bound for (\ref{thm_prob_4})} : 
Let $M := \mathcal{N}(\sqrt{\frac{\gamma_n}{n}}, \mathcal{V}, ||\cdot||_{\infty})$, 
and let $v_1 , \dots, v_M$ be the centers of $\sqrt{\frac{\gamma_n}{n}}$-cover of $\mathcal{V}$ (They do not need to be in $\mathcal{V}$, but we assume that their infinite norm is bounded by 1).
Since $\epsilon_n  > \sqrt{\frac{\gamma_n}{n}}$, for every $v \in \mathcal{V}$ there exists $m \in \{1,\dots,M\}$ such that $||v- v_m ||_{\infty} \leq \epsilon_n $.
Hence, we obtain
\begin{align}
    \mathbb{P}^{(n)}_{-i} \left( (\ref{thm_prob_4}) > 3 \epsilon_n \right) 
    = & \mathbb{P}^{(n)}_{-i} \left( \sup_{v \in \mathcal{V}} \left| \frac{1}{n-1} \sum_{j \neq i} v({\bm{Z}}_j) - \mathbb{E}_{\bm{Z}} v(\bm{Z}) \right|
        > 3 \epsilon_n \right) \nonumber \\
    \leq & \mathbb{P}^{(n)}_{-i} \left( \bigcup_{m=1}^M \left\{ \Bigg| \frac{1}{n-1} \sum_{j \neq i}  v_m(\bm{Z}_j)
        - \mathbb{E}_{\bm{Z}}  v_m(\bm{Z}) 
        \Bigg|
        > \epsilon_n \right\} \right) \nonumber \\
    \leq & \sum_{m=1}^M 
    \mathbb{P}^{(n)}_{-i} \left( \Bigg| \frac{1}{n-1} \sum_{j \neq i}  v_m(\bm{Z}_j)
        - \mathbb{E}_{\bm{Z}} v_m(\bm{Z} )
        \Bigg|
        > \epsilon_n \right) \nonumber \\
    \leq & 2M \exp \left(-\frac{n \epsilon_n^2 }{2} \right) \nonumber \\
    \leq & \frac{1}{n^2}, \label{thm_prob_bound4}
\end{align}
where we use the Hoeffding inequality (Lemma \ref{ineq_hoeff}) for the third inequality.
   
    In conclusion, by (\ref{thm_prob_bound1}), (\ref{thm_prob_bound2}), (\ref{thm_prob_bound3}) and (\ref{thm_prob_bound4}), we have
    \begin{align}        
         \mathbb{P}^{(n)}_{-i} \Bigg( \left|\textup{IPM}_{\mathcal{V}} (\hat{\mbP}^{(-i), \gamma_n}_{\bm{Z}|S=s} ,\hat{\mbP}^{(-i)}_{\bm{Z}}) - \textup{IPM}_{\mathcal{V}} (\mbP_{\bm{Z}|S=s} ,\mbP_{\bm{Z}}) \right|        
        > c' \epsilon_n \Bigg)
         \leq \frac{3}{n^2}  \label{for_suph_tmp2}
    \end{align}
    for every $s \in \mathcal{S}$, where $c'$ is the constant not depending on $n$, $m$ and $s$.  
    Hence, we get
    \begin{align}
        \left|\hat{\textup{EIPM}}^{\gamma_n}_{\mathcal{V}} (\bm{Z} ; S)
        - \frac{1}{n} \sum_{i=1}^n \left[ \textup{IPM}_\mathcal{V} (\mbP_{\bm{Z}|S = S_i} ,\mbP_{\bm{Z}}) \right] \right| \nonumber
        = & \left| \frac{1}{n} \sum_{i=1}^n \textup{IPM}_{\mathcal{V}} (\hat{\mbP}^{(-i), \gamma_n}_{\bm{Z}|S=S_i} ,\hat{\mbP}^{(-i)}_{\bm{Z}})
        - \frac{1}{n} \sum_{i=1}^n \left[ \textup{IPM}_\mathcal{V} (\mbP_{\bm{Z}|S = S_i} ,\mbP_{\bm{Z}}) \right] \right| \nonumber\\
        \leq & \frac{1}{n} \sum_{i=1}^n \left|  \textup{IPM}_{\mathcal{V}} (\hat{\mbP}^{(-i), \gamma_n}_{\bm{Z}|S=S_i} ,\hat{\mbP}^{(-i)}_{\bm{Z}})
        -  \left[ \textup{IPM}_\mathcal{V} (\mbP_{\bm{Z}|S = S_i} ,\mbP_{\bm{Z}}) \right] \right| \nonumber \\
        \leq & c' \epsilon_n \label{thm_prob_bound5}
    \end{align}
    for sufficiently large n with probability at least $1-\frac{3}{n}$.
    Finally, since $0 \leq \textup{IPM}_\mathcal{V} (\mbP_{\bm{Z}|S = S_i} ,\mbP_{\bm{Z}}) \leq 1$ , we have
    \begin{align}
        \mathbb{P}^{(n)} \Bigg( \Bigg| \frac{1}{n} \sum_{i=1}^n \left[ \textup{IPM}_\mathcal{V} (\mbP_{\bm{Z}|S = S_i} ,\mbP_{\bm{Z}}) \right]
        - \mathbb{E}_{S} \left[ \textup{IPM}_\mathcal{V} (\mbP_{\bm{Z}|S} ,\mbP_{\bm{Z}}) \right] \Bigg| > \sqrt{\frac{\log (2n)}{2n}} \Bigg)
         \leq \frac{1}{n} \label{thm_prob_bound6}.
    \end{align}
    by the Hoeffding's inequality (Lemma \ref{ineq_hoeff}).    
    By (\ref{thm_prob_bound5}) and (\ref{thm_prob_bound6}), we obtain the assertion.
\end{proof}

\subsection{Lemma and Proof for Theorem \ref{thm_encoder_set}} \label{sec_proof_encoderconv}
\begin{lemma} \label{cover_composite}
    Consider the sets of functions $\mathcal{H} \subseteq \{ h : \mathcal{X} \to \mathcal{Z} \}$ and $\mathcal{V} \subseteq \{ v : \mathcal{Z} \to \mathcal{Y} \}$.
    Assume that the elements of $\mathcal{V}$ are Lipschitz functions with a Lipschitz constant $L>0$.
    For every $\epsilon_1 >0$ and $\epsilon_2 >0$, we have
    $$\mathcal{N}(\epsilon_1 + L\epsilon_2, \mathcal{V} \circ \mathcal{H}, ||\cdot||_{\infty}) 
    \leq \mathcal{N}(\epsilon_1 , \mathcal{V}, ||\cdot||_{\infty})
    \mathcal{N}(\epsilon_2 , \mathcal{H}, ||\cdot||_{\infty}) 
    .$$
\end{lemma}
\begin{proof}
    Let $M_{1} := \mathcal{N}(\epsilon_1, \mathcal{V}, ||\cdot||_{\infty})$, 
and let $v_1 , \dots, v_{M_{1}}$ be the centers of $\epsilon_1$-cover of $\mathcal{V}$. 
    Similarly, we define $M_{2} := \mathcal{N}(\epsilon_2, \mathcal{H}, ||\cdot||_{\infty})$, and let $h_1 , \dots, h_{M_2}$ be the centers of $\epsilon_2$-cover of $\mathcal{H}$.
    For any $v \in \mathcal{V}$ and $h \in \mathcal{H}$, there exist $m_1 \in \{1,\dots,M_1 \}$ and $m_2 \in \{1,\dots,M_2 \}$  
    such that $||v - v_{m_1}||_{\infty} \leq \epsilon_1$
    and $||h - h_{m_2}||_{\infty} \leq \epsilon_2$.
    By the definition of infinite norm, we have
    $$||(v \circ h_{m_2}) - (v_{m_1} \circ h_{m_2})||_{\infty} \leq ||v - v_{m_1}||_{\infty}.$$
    Also, by Lipschitz condition we have 
    $$||(v \circ h) - (v \circ h_{m_2})||_{\infty} \leq L ||h - h_{m_2}||_{\infty}.$$
    To sum up, we obtain
    \begin{align*}
        ||(v \circ h) - (v_{m_1} \circ h_{m_2})||_{\infty} 
        \leq & ||(v \circ h) - (v \circ h_{m_2})||_{\infty} 
        + ||(v \circ h_{m_2}) - (v_{m_1} \circ h_{m_2})||_{\infty} \\
        \leq & L ||h - h_{m_2}||_{\infty} + ||v - v_{m_1}||_{\infty}  \\
        \leq & \epsilon_1 + L\epsilon_2
    \end{align*}    
    This implies that balls with centered at $(v_1 \circ h_1), (v_1 \circ h_2), \dots, (v_{m_1} \circ h_{m_2})$ with radius $\epsilon_1 + L\epsilon_2$ can cover $\mathcal{V} \circ \mathcal{H}$.     
\end{proof}

\subsubsection{Proof for Theorem \ref{thm_encoder_set}}
\begin{proof}
    For given $i \in \{1,\dots,n\}$ and given $S_i = s$, we have
    \begin{align}
        & \sup_{h \in \mathcal{H}} \left |\textup{IPM}_{\mathcal{V}} (\hat{\mbP}^{(-i), \gamma_n}_{\bm{Z}|S=s} ,\hat{\mbP}^{(-i)}_{\bm{Z}}) - \textup{IPM}_{\mathcal{V}} (\mbP_{\bm{Z}|S=s} ,\mbP_{\bm{Z}}) \right| \nonumber \\
        & = \sup_{h \in \mathcal{H}} \Bigg| \sup_{v \in \mathcal{V}} \left| \sum_{j \neq i}  \hat{w}_{\gamma_n}(j;i) v(\bm{Z}_j) -  \frac{1}{n-1} \sum_{j \neq i} v(\bm{Z}_j) \right| - \sup_{v \in \mathcal{V}} \Big| \mathbb{E}_{\bm{Z}}(v(\bm{Z})|S=s) -  \mathbb{E}_{\bm{Z}} v(\bm{Z}) \Big|
        \Bigg|  \nonumber \\
        & \leq  \sup_{h \in \mathcal{H}} \sup_{v \in \mathcal{V}} \Bigg| \sum_{j \neq i}  \hat{w}_{\gamma_n}(j;i) v(\bm{Z}_j) -  \frac{1}{n-1} \sum_{j \neq i} v(\bm{Z}_j) - \mathbb{E}_{\bm{Z}} (v(\bm{Z})|S=s) + \mathbb{E}_{\bm{Z}} v(\bm{Z}) \Bigg| \nonumber \\  
        & =  \sup_{h \in \mathcal{H}} \sup_{v \in \mathcal{V}} \Bigg| \sum_{j \neq i}  \hat{w}_{\gamma_n}(j;i) v(h(\bm{X}_j)) -  \frac{1}{n-1} \sum_{j \neq i} v(h(\bm{X}_j)) - \mathbb{E}_{\bm{X}} (v(h(\bm{X}))|S=s) + \mathbb{E}_{\bm{X}} v(h(\bm{X}_j)) \Bigg| \nonumber \\
        & =  \sup_{g \in \mathcal{V} \circ \mathcal{H}} 
        \Bigg| \sum_{j \neq i}  \hat{w}_{\gamma_n}(j;i) g(\bm{X}_j) -  \frac{1}{n-1} \sum_{j \neq i} g(\bm{X}_j) - \mathbb{E}_{\bm{X}} g(\bm{X})|S=s) + \mathbb{E}_{\bm{X}} g(\bm{X}_j) \Bigg|, \label{eq_suph_tmp1}       
    \end{align}    
    where the inequality holds by the fact that $|\sup |a| - \sup |b|| \leq \sup|a-b|$.
    Since (\ref{eq_suph_tmp1}) has the same form as the expression of (\ref{for_suph_tmp1}), we can follow the proof of Theorem \ref{EIPM_conv} to obtain a similar result as (\ref{for_suph_tmp2}). In other words, we obtain
    \begin{align*}        
         \mathbb{P}^{(n)}_{-i} \Bigg( \sup_{h \in \mathcal{H}} \left|\textup{IPM}_{\mathcal{V}} (\hat{\mbP}^{(-i), \gamma_n}_{\bm{Z}|S=s} ,\hat{\mbP}^{(-i)}_{\bm{Z}}) - \textup{IPM}_{\mathcal{V}} (\mbP_{\bm{Z}|S=s} ,\mbP_{\bm{Z}}) \right|        
        > c' \epsilon'_n \Bigg)
         \leq \frac{3}{n^2}  
    \end{align*}
    for every $s \in \mathcal{S}$, where
    $$\epsilon'_n := \gamma_n^2 + \frac{\log n}{\sqrt{n \gamma_n}} \left[1 + \log \mathcal{N} \left(\sqrt{\frac{\gamma_n}{n}}, \mathcal{V} \circ \mathcal{H}, ||\cdot||_{\infty}\right) \right]^{\frac{1}{2}}$$
    and $c'$ is the constant not depending on $n$, $m$ and $s$.
    Also, using Lemma \ref{cover_composite}, we obtain
    \begin{align*}
        \epsilon'_n & = \gamma_n^2 + \frac{\log n}{\sqrt{n \gamma_n}} \left(1 + \log \mathcal{N} \left(\sqrt{\frac{\gamma_n}{n}}, \mathcal{V} \circ \mathcal{H}, ||\cdot||_{\infty}\right) \right)^{\frac{1}{2}} \\
        & \leq \gamma_n^2 + \frac{\log n}{\sqrt{n \gamma_n}} \left[1 + \log \mathcal{N} \left(\frac{1}{2}\sqrt{\frac{\gamma_n}{n}}, \mathcal{V}, ||\cdot||_{\infty}\right)
        + \log \mathcal{N} \left(\frac{1}{2L}\sqrt{\frac{\gamma_n}{n}}, \mathcal{H}, ||\cdot||_{\infty}\right)
        \right]^{\frac{1}{2}} \\
        & =: \epsilon_n.
    \end{align*}
    Hence, we get
    \begin{align}
        \sup_{h \in \mathcal{H}} \left|\hat{\textup{EIPM}}^{\gamma_n}_{\mathcal{V}} (\bm{Z} ; S)
        - \frac{1}{n} \sum_{i=1}^n \left[ \textup{IPM}_\mathcal{V} (\mbP_{\bm{Z}|S = S_i} ,\mbP_{\bm{Z}}) \right] \right| \nonumber
        = & \sup_{h \in \mathcal{H}} \left| \frac{1}{n} \sum_{i=1}^n \textup{IPM}_{\mathcal{V}} (\hat{\mbP}^{(-i), \gamma_n}_{\bm{Z}|S=S_i} ,\hat{\mbP}^{(-i)}_{\bm{Z}})
        - \frac{1}{n} \sum_{i=1}^n \left[ \textup{IPM}_\mathcal{V} (\mbP_{\bm{Z}|S = S_i} ,\mbP_{\bm{Z}}) \right] \right| \nonumber\\
        \leq & \frac{1}{n} \sum_{i=1}^n \sup_{h \in \mathcal{H}} \left|  \textup{IPM}_{\mathcal{V}} (\hat{\mbP}^{(-i), \gamma_n}_{\bm{Z}|S=S_i} ,\hat{\mbP}^{(-i)}_{\bm{Z}})
        -  \left[ \textup{IPM}_\mathcal{V} (\mbP_{\bm{Z}|S = S_i} ,\mbP_{\bm{Z}}) \right] \right| \nonumber \\
        \leq & c' \epsilon'_n \nonumber \\
        \leq & c' \epsilon_n \label{thm_prob_bound5_tmp}
    \end{align}
    for sufficiently large n with probability at least $1-\frac{3}{n}$.

Let $M := \mathcal{N}(\frac{1}{2L} \sqrt{\frac{\gamma_n}{n}}, \mathcal{H}, ||\cdot||_{\infty})$, 
and let $h_1 , \dots, h_M$ be the centers of $\frac{1}{2L} \sqrt{\frac{\gamma_n}{n}}$-cover of $\mathcal{H}$
(They do not need to be in $\mathcal{H}$).
Since $\frac{\epsilon_n}{4L}  > \frac{1}{2L} \sqrt{\frac{\gamma_n}{n}}$, for every $h \in \mathcal{H}$ there exists $m \in \{1,\dots,M\}$ such that $||h - h_m ||_{\infty} \leq \frac{\epsilon_n}{4L}$.
For any $s \in \mathcal{S}$, we have
    \begin{align*}
        & \left| \textup{IPM}_\mathcal{V} (\mbP_{h(\bm{X})|S = s} ,\mbP_{h(\bm{X})}) - \textup{IPM}_\mathcal{V} (\mbP_{h_m(\bm{X})|S = s} ,\mbP_{h_m(\bm{X})}) \right| \\
        & = \left| \sup_{v \in \mathcal{V}} \left| \mathbb{E}_{\bm{X}}(v \circ h(\bm{X})|S = s) - \mathbb{E}_{\bm{X}}(v \circ h(\bm{X}))  \right| - \sup_{v \in \mathcal{V}} \left| \mathbb{E}_{\bm{X}}(v \circ h_m(\bm{X})|S = s) - \mathbb{E}_{\bm{X}}(v \circ h_m(\bm{X})) \right| \right| \\
        & \leq  \sup_{v \in \mathcal{V}} \Bigg| \mathbb{E}_{\bm{X}}(v \circ h(\bm{X})|S = s) - \mathbb{E}_{\bm{X}}(v \circ h(\bm{X}))  - \mathbb{E}_{\bm{X}}(v \circ h_m(\bm{X})|S = s) +\mathbb{E}_{\bm{X}}(v \circ h_m(\bm{X}))  \Bigg| \\
        & \leq  \sup_{v \in \mathcal{V}} \Bigg| \mathbb{E}_{\bm{X}}(v \circ h(\bm{X})|S = s) - \mathbb{E}_{\bm{X}}(v \circ h_m(\bm{X})|S = s)  \Bigg|  + \sup_{v \in \mathcal{V}} \Bigg| \mathbb{E}_{\bm{X}}(v \circ h(\bm{X}))  - \mathbb{E}_{\bm{X}}(v \circ h_m(\bm{X}))  \Bigg| \\
        & \leq L \cdot \frac{\epsilon_n}{4L} + L \cdot \frac{\epsilon_n}{4L} = \frac{\epsilon_n}{2},
    \end{align*}
    where the first inequality holds by the fact that $|\sup |a| - \sup |b|| \leq \sup|a-b|$ and the third inequality uses the Lipschitz condition. Using this result, we get
    \begin{align*}
        & \Bigg| \frac{1}{n} \sum_{i=1}^n \left[ \textup{IPM}_\mathcal{V} (\mbP_{h(\bm{X})|S = S_i} ,\mbP_{h(\bm{X})}) \right]
        - \mathbb{E}_{S} \left[ \textup{IPM}_\mathcal{V} (\mbP_{h(\bm{X})|S} ,\mbP_{h(\bm{X})}) \right] \Bigg|\\
        & \qquad -  \Bigg| \frac{1}{n} \sum_{i=1}^n \left[ \textup{IPM}_\mathcal{V} (\mbP_{h_m (\bm{X})|S = S_i} ,\mbP_{h_m (\bm{X})}) \right]
        - \mathbb{E}_{S} \left[ \textup{IPM}_\mathcal{V} (\mbP_{h_m (\bm{X})|S} ,\mbP_{h_m (\bm{X})}) \right] \Bigg| \\
        & \leq \Bigg| \frac{1}{n} \sum_{i=1}^n \left[ \textup{IPM}_\mathcal{V} (\mbP_{h(\bm{X})|S = S_i} ,\mbP_{h(\bm{X})}) \right]
        - \frac{1}{n} \sum_{i=1}^n \left[ \textup{IPM}_\mathcal{V} (\mbP_{h_m (\bm{X})|S = S_i} ,\mbP_{h_m (\bm{X})}) \right] \Bigg|\\
        & \qquad +  \Bigg| \mathbb{E}_{S} \left[ \textup{IPM}_\mathcal{V} (\mbP_{h(\bm{X})|S} ,\mbP_{h(\bm{X})}) \right] 
        - \mathbb{E}_{S} \left[ \textup{IPM}_\mathcal{V} (\mbP_{h_m (\bm{X})|S} ,\mbP_{h_m (\bm{X})}) \right] \Bigg| \\
        & \leq \frac{\epsilon_n}{2} + \frac{\epsilon_n}{2} = \epsilon_n,
    \end{align*}
    where the first inequality holds by the fact that $|a-b| - |c-d| \leq |a-c| + |b-d|$.
    Hence, we obtain
    \begin{align}
        &\mathbb{P}^{(n)} \Bigg( \sup_{h \in \mathcal{H}} \Bigg| \frac{1}{n} \sum_{i=1}^n \left[ \textup{IPM}_\mathcal{V} (\mbP_{\bm{Z}|S = S_i} ,\mbP_{\bm{Z}}) \right]
        - \mathbb{E}_{S} \left[ \textup{IPM}_\mathcal{V} (\mbP_{\bm{Z}|S} ,\mbP_{\bm{Z}}) \right] \Bigg| > 2 \epsilon_n  \Bigg) \nonumber \\
        & \leq \mathbb{P}^{(n)} \Bigg( \bigcup_{m=1}^M \Bigg| \frac{1}{n} \sum_{i=1}^n \left[ \textup{IPM}_\mathcal{V} (\mbP_{h_m(\bm{X})|S = S_i} ,\mbP_{h_m(\bm{X})}) \right]
        - \mathbb{E}_{S} \left[ \textup{IPM}_\mathcal{V} (\mbP_{h_m(\bm{X})|S} ,\mbP_{h_m(\bm{X})}) \right] \Bigg| > \epsilon_n \Bigg) \nonumber \\    
        & \leq \sum_{m=1}^M \mathbb{P}^{(n)} \Bigg( \Bigg| \frac{1}{n} \sum_{i=1}^n \left[ \textup{IPM}_\mathcal{V} (\mbP_{h_m(\bm{X})|S = S_i} ,\mbP_{h_m(\bm{X})}) \right]
        - \mathbb{E}_{S} \left[ \textup{IPM}_\mathcal{V} (\mbP_{h_m(\bm{X})|S} ,\mbP_{h_m(\bm{X})}) \right] \Bigg| > \epsilon_n \Bigg) \nonumber \\  
        & \leq 2M \exp\left( -\frac{n \epsilon_n^2}{2} \right) \nonumber \\
         &\leq \frac{1}{n} \label{thm_prob_bound6_tmp},
    \end{align}
    where we use the Hoeffding inequality (Lemma \ref{ineq_hoeff}) for the third inequality.
    By (\ref{thm_prob_bound5_tmp}) and (\ref{thm_prob_bound6_tmp}), we obtain the assertion.
\end{proof}

\clearpage

\section{Algorithm Details} \label{App_alg}

Algorithm \ref{alg:eipm} presents the overall algorithm of FREM.

\begin{algorithm}[h]
\caption{FREM}
\label{alg:eipm}
    \begin{algorithmic}[1]
        \REQUIRE 1. Network parameters.
        \\
        $\theta$: Parameter of the representation encoder $h.$
        \\
        $\phi$: Parameter of prediction head $f.$
        \REQUIRE 2. Hyper-parameters.
        \\
        $\lambda:$ Regularization parameter.
        \\
        $\textup{lr}:$ Learning rate.
        \\
        $T$: Training epochs.
        \\
        $n_{\textup{mb}}:$ Mini-batch size.
        \\
        $\gamma$ : Radius of kernel for EIPM estimation.
        \\
        \FOR{$t = 1, \cdots, T$}
        \STATE Randomly sample a mini-batch $(\bm{x}_{i}, y_{i}, s_{i})_{i=1}^{n_{\textup{mb}}}$
        \newline
        \STATE \textbf{(Compute task loss)}
        \\
        $\mathcal{L}_{\textup{sup}}(\theta, \phi) = \frac{1}{n_{\textup{mb}}} \sum_{i=1}^{n_{\textup{mb}}} l(y_{i}, f_{\phi}(h_{\theta}(\bm{x}_{i}))) $ 
        \newline
        \STATE \textbf{(Compute EIPM)}
        \\
        Compute $n_{\textup{mb}} \times n_{\textup{mb}}$ matrix $A_{\gamma}$ by
        $$[A_{\gamma}]_{i,j} = \frac{K_{\gamma}(s_i , s_j)}{\sum_{j \neq i} K_{\gamma}(s_i , s_j) } - \frac{1}{n_{\textup{mb}}-1}$$
        \STATE Compute
        $$
        \mathcal{L}_{\textup{fair}}(\theta) = \sum_{i=1}^{n_{\textup{mb}}} \left( \sum_{j,k \neq i} [A_\gamma]_{i,j}   [A_\gamma]_{i,k} \kappa(j, k) \right)^{\frac{1}{2}}
        $$
        where $\kappa(j, k) := \kappa(h_{\theta}(\bm{x}_{j}), h_{\theta}(\bm{x}_{k})).$
        \STATE Divide by mini-batch size
        $ \mathcal{L}_{\textup{fair}}(\theta) \leftarrow \frac{1}{n_{\textup{mb}}} \mathcal{L}_{\textup{fair}}(\theta) $
        \STATE \textbf{(Total loss)}
        \\
        $\mathcal{L}(\theta, \phi) = \mathcal{L}_{\textup{sup}}(\theta, \phi) + \lambda \mathcal{L}_{\textup{fair}}(\theta)$
        \STATE \textbf{(Parameter updates)}
        \\
        $\theta \leftarrow \theta - \textup{lr} \cdot\nabla_{\theta} \mathcal{L} (\theta, \phi)
        \newline
        \phi \leftarrow \phi - \textup{lr}\cdot\nabla_{\phi} \mathcal{L} (\theta, \phi)$ 
        \ENDFOR
        \newline
        \textbf{Return} $\theta$ and $\phi$
    \end{algorithmic}
\end{algorithm}

\clearpage
Also, 
Algorithm \ref{alg:ptlike} provides a \texttt{Pytorch}-style pseudo-code for computing EIPM.

\definecolor{commentcolor}{RGB}{110,154,155}  
\newcommand{\PyComment}[1]{\ttfamily\textcolor{commentcolor}{\# #1}}  
\newcommand{\PyCode}[1]{\ttfamily\textcolor{black}{#1}}

\begin{algorithm}[h]
    \SetAlgoLined
    \PyComment{Function that outputs the kernel matrix between v and w using RBF kernel with bandwidth a.} \\
    \PyCode{def Kernel\_matrix(Xi, Xj, a):} \\
    \Indp
    \PyComment{Xi, Xj: a pair of representation vectors of size [n, m]} \\
    \PyComment{a: bandwidth for kernel} \\
    \PyCode{m = - torch.cdist(Xi, Xj, p=2)**2} \\
    \PyCode{m = m / (2 * a**2)} \\
    \PyCode{m = torch.exp(m)} \\
    \PyCode{return m} \\
    \Indm
    \PyComment{Function that outputs EIPM value between z and s using bandwidths $\sigma$ and $\gamma$.} \\
    \PyCode{def compute\_EIPM(z, s, $\sigma$, $\gamma$):} \\
    \Indp
    \PyComment{z: representation vectors of size [n, m]} \\
    \PyComment{s: sensitive attributes of size [n, 1]} \\
    \PyComment{$\sigma$: bandwidth in MMD} \\
    \PyComment{$\gamma$: bandwidth of kernel estimator} \\
    \PyComment{} \\
    \PyComment{Kernel method for s} \\
    \PyCode{A = Kernel\_matrix(s, s, $\gamma$) - torch.eye(B)} \\
    \PyCode{A = A / A.sum(dim=0)} \\
    \PyCode{A = A - 1 / (n - 1)} \\
    \PyCode{A = A.fill\_diagonal\_(0)} \\
    \PyComment{MMD} \\
    \PyCode{K = Kernel\_matrix(z, z, $\sigma$)} \\
    \PyComment{Compute EIPM} \\
    \PyCode{EIPM = torch.einsum('ij, ik' -> 'ijk', A.T, A.T)} \\
    \PyCode{EIPM = torch.sum(EIPM * K.unsqueeze(dim=0), dim=(1, 2)).sum()} \\
    \PyCode{EIPM = EIPM / n} \\
    \PyComment{} \\
    \PyCode{return EIPM}
    \caption{\texttt{PyTorch}-style pseudo-code for computing EIPM}
    \label{alg:ptlike}
\end{algorithm}

\clearpage

\section{Extension to Equal Opportunity} \label{App_EO}
\renewcommand{\theequation}{C.\arabic{equation}}

\subsection{Definition and properties of EIPM for equal opportunity}

Equal Opportunity (EO) \cite{hardt2016equality} is another important group fairness measure, besides DP.
For a binary sensitive attribute $S \in \{0,1\}$, binary output $Y \in \{0,1\}$ and given prediction model $g : \cX \to  \{0,1\}$, EO is defined as
$$
\Delta \texttt{EO}(g) = 
\left\vert \mathbb{E}_{\bm{X}} \left( g(\bm{X}) | S = 1, Y = 1 \right) - \mathbb{E}_{\bm{X}} \left( g(\bm{X})  | S = 0, Y = 1 \right) \right\vert.
$$
Considering the concept of \cite{jiang2021generalized}, it is natural to define Generalized Equal Opportunity (GEO) as
$$\Delta \textup{\texttt{GEO}}(g) =  \mathbb{E}_S \Big| \mathbb{E}_{\bm{X} } ( g(\bm{X})  | S, Y=1 ) - \mathbb{E}_{\bm{X}} ( g(\bm{X}) | Y=1 ) \Big|$$
as an extension of EO for continuous sensitive attributes.

FREM algorithm (for DP) can be modified easily for EO.
Instead of (\ref{def_EIPM}), which is the definition of EIPM for DP, we consider
\begin{align}
    \textup{EIPM}_\mathcal{V} (\bm{Z}; S | Y=1) := \mathbb{E}_{S} \left[ \textup{IPM}_\mathcal{V} (\mbP_{\bm{Z}|S, Y=1} ,\mbP_{\bm{Z}|Y=1}) \right]. \label{def_EIPM_EO}
\end{align}
First, we give a basic property  that the zero EIPM (for GEO) value guarantees perfectly fair representation. 
\begin{theorem} \label{pro_basic_EO}
    Assume that a set of discriminator $\mathcal{V}$ is large enough for $\textup{IPM}_\mathcal{V}$ to be a metric on the probability space of $\mathcal{Z}$. 
    Then, $\textup{EIPM}_\mathcal{V} (\bm{Z}; S| Y=1) = 0$ implies $\Delta \textup{\texttt{GEO}}( f \circ h ) = 0$ for any (bounded) prediction head $f$.
\end{theorem}

\begin{proof}
    From $\textup{EIPM}_\mathcal{V} (\bm{Z}; S|Y=1) = \mathbb{E}_{S} \left[ \textup{IPM}_\mathcal{V} (\mbP_{\bm{Z}|S, Y=1} ,\mbP_{\bm{Z}|Y=1}) \right] = 0$, we have $\textup{IPM}_\mathcal{V} (\mbP_{\bm{Z}|S,Y=1} ,\mbP_{\bm{Z}|Y=1}) = 0$ almost surely (with respect to probability of $S$).
Since $\textup{IPM}_\mathcal{V}$ is a metric on the probability space of $\mathcal{Z}$, we have   
$\mbP_{\bm{Z}|S,Y=1} \equiv \mbP_{\bm{Z}|Y=1}$ almost surely.
Hence, for any bounded prediction head $f$, we get
\begin{align*}
    \Delta \textup{\texttt{GEO}}(f \circ h) =&  \mathbb{E}_{S} \Big| \mathbb{E}_{\bm{Z} } ( f(\bm{Z})  | S , Y=1 ) - \mathbb{E}_{\bm{Z}} ( f(\bm{Z}) | Y=1 ) \Big| \\
    \leq & \int_{S} \int_{\bm{Z}} \left|f(\bm{Z}) 
    \left(d\mbP_{\bm{Z}|S, Y=1} - d\mbP_{\bm{Z}|Y=1}\right) \right| d\mbP_S \\
    = & 0.
\end{align*}
\end{proof}

Another important property of (\ref{def_EIPM_EO}) is that the level of GEO of a given prediction head can be controlled by the level of EIPM (for GEO) as long as the prediction head is included in a properly defined function class.
\begin{theorem} \label{GDP_upper_EO}
    Let $\bm{Z}=h(\bm{X})$ be a representation corresponding to an encoding function $h$.
    For a class $\mathcal{V}$ of discriminators and a class $\mathcal{F}$ of prediction heads, we have the following results.
    \begin{enumerate}
        \item Let $\kappa := \sup_{f \in \mathcal{F}} \inf_{v \in \mathcal{V}} || f - v ||_{\infty}$. 
        Then for any prediction head $f \in \mathcal{F}$, we have
    $$\Delta \textup{\texttt{GEO}}(f \circ h) \leq \textup{EIPM}_{\mathcal{V}} (\bm{Z}; S| Y=1) + 2\kappa.$$
        \item Assume there exists an increasing concave function $\xi : [0,\infty) \to [0,\infty)$ such that $\lim_{r \downarrow 0} \xi(r) = 0$ and $\textup{IPM}_{\mathcal{F}}(\mbP^0, \mbP^1) \leq \xi(\textup{IPM}_{\mathcal{V}}(\mbP^0, \mbP^1))$ for any two probability measures $\mbP^0$ and $\mbP^1$.  
        Then for any prediction head $f \in \mathcal{F}$, we have
        $$\Delta \textup{\texttt{GEO}}(f \circ h) \leq \xi(\textup{EIPM}_{\mathcal{V}} (\bm{Z}; S| Y=1)).$$
    \end{enumerate}
\end{theorem}

\begin{proof}
1) For any $f \in \mathcal{F}$, there exists $v \in \mathcal{V}$ such that 
$|| f - v ||_{\infty} \leq \kappa$. Then,
    \begin{align*}
     \Delta \textup{\texttt{GEO}}(f \circ h) 
     & =  \mathbb{E}_{S} \ \left| \mathbb{E}_{\bm{Z}}( f(\bm{Z})  | S, Y=1) - \mathbb{E}_{\bm{Z}} ( f(\bm{Z})| Y=1 ) \right| \\
    & =  \mathbb{E}_{S } \left| \int f(\bm{z}) d{\mbP}_{\bm{Z}|S, Y=1}(\bm{z}) - \int f(\bm{z}) d \mbP_{\bm{Z}|Y=1}(\bm{z})   \right| \\
     & \leq \mathbb{E}_{S } \left| \int v(\bm{z}) d{\mbP}_{\bm{Z}|S, Y=1}(\bm{z}) - \int v(\bm{z}) d \mbP_{\bm{Z}|Y=1}(\bm{z})   \right| + 2\kappa \\
    & \leq  \mathbb{E}_{S} \sup_{v \in \mathcal{V}} \left|  
    \int v(\bm{z}) d \mbP_{\bm{Z}|S, Y=1}(\bm{z}) - \int v(\bm{z}) d \mbP_{\bm{Z}|Y=1}(z)   \right| + 2\kappa \\
    & =  \textup{EIPM}_\mathcal{V} (\bm{Z} ; S|Y=1) + 2\kappa.
\end{align*}

2) For any $f \in \mathcal{F}$,
    \begin{align*}
     \Delta \textup{\texttt{GEO}}(f \circ h) 
     & =  \mathbb{E}_{S} \ \left| \mathbb{E}_{\bm{Z}}( f(\bm{Z})  | S, Y=1) - \mathbb{E}_{\bm{Z}} ( f(\bm{Z}) | Y=1 ) \right| \\
    & =  \mathbb{E}_{S } \left| \int f(\bm{z}) d{\mbP}_{\bm{Z}|S, Y=1}(\bm{z}) - \int f(\bm{z}) d \mbP_{\bm{Z}|Y=1}(\bm{z})   \right| \\
    & \leq  \mathbb{E}_{S} \sup_{f \in \mathcal{F}} \left|  
    \int f(\bm{z}) d \mbP_{\bm{Z}|S, Y=1}(\bm{z}) - \int f(\bm{z}) d \mbP_{\bm{Z}|Y=1}(z)   \right|  \\
    & =  \mathbb{E}_{S} \textup{IPM}_{\cF}(\mbP_{\bm{Z}|S,Y=1} ,\mbP_{\bm{Z}|Y=1} )  \\
    & \leq  \mathbb{E}_{S} \xi(\textup{IPM}_{\cV}(\mbP_{\bm{Z}|S,Y=1} ,\mbP_{\bm{Z}|Y=1}))  \\ 
    & \leq  \xi(\textup{EIPM}_\mathcal{V} (\bm{Z} ; S | Y=1)),
\end{align*}
where the last inequality holds by the Cauchy inequality.
\end{proof}

\subsection{Estimation of EIPM for equal opportunity and its convergence analysis}
We denote $n_1 := |\{i : Y_i = 1 \}|$.
We estimate $\textup{EIPM}_\mathcal{V} (\bm{Z}; S| Y=1)$ by
\begin{align*}
    \hat{\textup{EIPM}}^{\gamma}_{\mathcal{V}} (\bm{Z}; S | Y=1) := 
    \frac{1}{n_1} \sum_{i : Y_i=1}  \textup{IPM}_\mathcal{V} \left(\hat{\mbP}^{(-i), \gamma}_{\bm{Z}|S=S_i, Y=1} ,
    \hat{\mbP}^{(-i)}_{\bm{Z}|Y=1}\right),
\end{align*}
where $\hat{\mbP}^{(-i), \gamma}_{\bm{Z}|S=S_i, Y=1}$ for $i \in \{i : Y_i = 1\}$ is defined as
$$\hat{\mbP}^{(-i), \gamma}_{\bm{Z}|S=S_i, Y=1}:= \sum_{j \neq i} \frac{K_{\gamma}(S_i, S_j) \mathbb{I}(Y_i = 1)}{\sum_{j \neq i} K_{\gamma}(S_i, S_j) \mathbb{I}(Y_i = 1)} \delta(Z_i)$$
and $\hat{\mbP}^{(-i)}_{\bm{Z}|Y=1} := \frac{1}{n_1} \sum_{j \neq i} \mathbb{I}(Y_j = 1)\delta(\bm{Z}_j)$.

We derive the convergence rate of $\hat{\textup{EIPM}}^{\gamma}_{\mathcal{V}}(\bm{Z}; S | Y=1)$ as the sample size increases under the following very mild regularity conditions.
\begin{assumption} \label{assum_ps_EO}
    $\mbP_{S|Y=1}$ admits a density $p_1(s)$ with respect to Lebesgue measure $\mu_{S}$ on $\mathcal{S}$. Also, there exist $0<L_p<U_p<\infty$ such that $L_p < p_1(s) < U_p$ on $s \in \mathcal{S}$.    
\end{assumption}

\begin{assumption} \label{assum_pzs_EO}
Suppose that there exists a $\sigma$-finite measure $\mu_{\bm{X}}$ on $\mathcal{X}$
such that $\mbP_{\bm{X},S|Y=1} << \mu_{\bm{X}} \otimes \mu_{S},$ where $\mu_{S}$ is the Lebesgue measure on $\mathcal{S}.$
    We denote $p_1(\bm{x},s) := \frac{d\mbP_{\bm{X},S|Y=1}}{d(\mu_{\bm{X}} \otimes  \mu_{S})}(\bm{x},s)$ as the Radon-Nikodym derivative of $\mbP_{\bm{X},S|Y=1}$ with respect to $\mu_{\bm{X}} \otimes  \mu_{S}$.     
    For every $\bm{x} \in \mathcal{X}$, $p_1(\bm{x},s)$ is twice differentiable with respect to $s$ and has bounded second derivative.
\end{assumption}

\begin{theorem} \label{EIPM_conv_EO}
Let $h$ be a bounded measurable encoder.
Suppose that Assumption \ref{assum_kernel}, \ref{assum_gamma}, \ref{assum_ps_EO} and \ref{assum_pzs_EO} hold.
Then, for 
$$
\epsilon_n = \gamma_n^2 + \frac{\log n_1}{\sqrt{n_1 \gamma_n}} \left( 1 + \log \mathcal{N} \left(\sqrt{\frac{\gamma_n}{n_1}}, \mathcal{V}, ||\cdot||_{\infty}\right) \right)^{\frac{1}{2}},
$$
we have
\begin{align*}
    \left|\hat{\textup{EIPM}}^{\gamma_n}_{\mathcal{V}} (\bm{Z} ; S| Y=1) - \textup{EIPM}_{\mathcal{V}} (\bm{Z} ; S| Y=1) \right| 
    < c \epsilon_n
\end{align*}
for sufficiently large $n$ with probability at least $1-\frac{4}{n_1}$, where $c$ is the constant not depending on $n$ and $m$. 
\end{theorem}
\begin{proof}
    In the proof of Theorem \ref{EIPM_conv}, if we replace $n$, $p(s)$, $p(\bm{x},s)$, $\mathbb{P}_{S}$, $\mathbb{E}_{S}$, $\mathbb{V}_{S}$, $\mathbb{P}_{\bm{X}}$, $\mathbb{E}_{\bm{X}}$, $\mathbb{P}_{\bm{Z}}$, $\mathbb{E}_{\bm{Z}}$ $\mathbb{P}_{\bm{Z},S}$, $\mathbb{E}_{\bm{Z},S}$, $\mathbb{V}_{\bm{Z},S}$, $\mathbb{P}^{(n)}$ and $\mathbb{P}^{(n)}_{-i}$  
    with
    $n_1$, $p_1(s)$, $p_1(\bm{x},s)$, $\mathbb{P}_{S|Y=1}$, $\mathbb{E}_{S|Y=1}$, $\mathbb{V}_{S|Y=1}$, $\mathbb{P}_{\bm{X}|Y=1}$, $\mathbb{E}_{\bm{X}|Y=1}$, $\mathbb{P}_{\bm{Z}|Y=1}$, $\mathbb{E}_{\bm{Z}|Y=1}$ $\mathbb{P}_{\bm{Z},S|Y=1}$, $\mathbb{E}_{\bm{Z},S|Y=1}$, $\mathbb{V}_{\bm{Z},S|Y=1}$, $\mathbb{P}^{(n_1)}$ and $\mathbb{P}^{(n_1)}_{-i}$ respectively, the proof can be done similarly as that of Theorem \ref{EIPM_conv}.
\end{proof}

\subsection{FREM for equal opportunity}

We aim to find $(h, f)$ such that
\begin{align}
    \argmin_{h \in \mathcal{H}, f \in \mathcal{F}} & \mathcal{L}_{\textup{sup}}(f \circ h) && \textup{ s.t. } \textup{EIPM}_{\mathcal{V}} (h(\bm{X}); S | Y=1) \leq \delta, \label{argmin_1_EO} 
\end{align}
where $\mathcal{L}_{\textup{sup}}$ is a certain supervised risk such as the cross-entropy loss or MSE loss.
That is, the algorithm obtains a reasonable encoder from the fairness constraint set that minimizes the task-specific risk jointly with the prediction head.
In practice, the value of $\textup{EIPM}_{\mathcal{V}}$ in (\ref{argmin_1_EO}) is not available and we replace it by its estimator 
$\hat{\textup{EIPM}}^{\gamma}_{\mathcal{V}}$ to find the solution of  
\begin{align}
    \argmin_{h \in \mathcal{H}, f \in \mathcal{F}} & \mathcal{L}_{\textup{sup}}(f \circ h) && \textup{ s.t. } \hat{\textup{EIPM}}^{\gamma}_{\mathcal{V}} (h(\bm{X}), S | Y=1) \leq \delta \label{argmin_2_EO} 
\end{align}
with properly chosen bandwidth $\gamma$.
A statistical question is whether the two constraints in (\ref{argmin_1_EO})
and (\ref{argmin_2_EO}) becomes similar as the sample size increases.
The following theorem ensures that the two constraints are asymptotically equivalent.

\begin{theorem} \label{thm_encoder_set_EO}
    Consider a set of encoders $\mathcal{H}$ and the set of discriminators $\mathcal{V}$, where the elements of $\mathcal{V}$ are Lipschitz with the Lipschitz constant $L>0$.
    For $\delta>0$, we define 
    $$\mathcal{H}_{\mathcal{V}}(\delta)^{\operatorname{EO}} := \{ h \in \mathcal{H} : \textup{EIPM}_{\mathcal{V}} (h(\bm{X}); S| Y=1) \leq \delta\}$$
    and 
    $$\hat{\mathcal{H}}^{\gamma}_{\mathcal{V}}(\delta)^{\operatorname{EO}} := \{ h \in \mathcal{H} : \hat{\textup{EIPM}}^{\gamma}_{\mathcal{V}} (h(\bm{X}); S| Y=1) \leq \delta\}$$
    as the set of encoders whose representation spaces satisfy the fairness constraints defined by $\textup{EIPM}_{\mathcal{V}}$ and $\hat{\textup{EIPM}}^{\gamma}_{\mathcal{V}}$, respectively.
    Suppose that Assumption \ref{assum_kernel}, \ref{assum_gamma}, \ref{assum_ps_EO} and \ref{assum_pzs_EO} hold.
    Then, for every $\delta>0$ and 
    \begin{align*}
        \epsilon_n = \gamma_n^2 + \frac{\log n}{\sqrt{n_1 \gamma_n}} \Bigg( 1 + \log \mathcal{N} \left(\frac{1}{2}\sqrt{\frac{\gamma_n}{n_1}}, \mathcal{V}, ||\cdot||_{\infty}\right) 
        + \log \mathcal{N} \left(\frac{1}{2L}\sqrt{\frac{\gamma_n}{n_1}}, \mathcal{H}, ||\cdot||_{\infty}\right)
        \Bigg)^{\frac{1}{2}},
    \end{align*}
    we have
    \begin{align*}        
    \hat{\mathcal{H}}^{\gamma_n}_{\mathcal{V}}(\delta- c\epsilon_n)^{\operatorname{EO}} 
    \subseteq \mathcal{H}_{\mathcal{V}}(\delta)^{\operatorname{EO}} 
    \subseteq  \hat{\mathcal{H}}^{\gamma_n}_{\mathcal{V}}(\delta + c\epsilon_n)^{\operatorname{EO}}
    \end{align*}
    for sufficiently large $n$ with probability at least $1-\frac{4}{n_1}$, where $c$ is a constant not depending on $n$ and $m$.
\end{theorem}

\begin{proof}
    In the proof of Theorem \ref{thm_encoder_set}, if we replace $n$, $p(s)$, $p(\bm{x},s)$, $\mathbb{P}_{S}$, $\mathbb{E}_{S}$, $\mathbb{V}_{S}$, $\mathbb{P}_{\bm{X}}$, $\mathbb{E}_{\bm{X}}$, $\mathbb{P}_{\bm{Z}}$, $\mathbb{E}_{\bm{Z}}$ $\mathbb{P}_{\bm{Z},S}$, $\mathbb{E}_{\bm{Z},S}$, $\mathbb{V}_{\bm{Z},S}$, $\mathbb{P}^{(n)}$ and $\mathbb{P}^{(n)}_{-i}$  
    with
    $n_1$, $p_1(s)$, $p_1(\bm{x},s)$, $\mathbb{P}_{S|Y=1}$, $\mathbb{E}_{S|Y=1}$, $\mathbb{V}_{S|Y=1}$, $\mathbb{P}_{\bm{X}|Y=1}$, $\mathbb{E}_{\bm{X}|Y=1}$, $\mathbb{P}_{\bm{Z}|Y=1}$, $\mathbb{E}_{\bm{Z}|Y=1}$ $\mathbb{P}_{\bm{Z},S|Y=1}$, $\mathbb{E}_{\bm{Z},S|Y=1}$, $\mathbb{V}_{\bm{Z},S|Y=1}$, $\mathbb{P}^{(n_1)}$ and $\mathbb{P}^{(n_1)}_{-i}$ respectively, the proof can be done similarly to that of Theorem \ref{thm_encoder_set}.
\end{proof}

There are several choices for the set of discriminators in EIPM.
For the Reproducing Kernel Hilbert Space (RKHS) $(\mathcal{V}_{\kappa}(\mathcal{Z}), ||\cdot||_{\mathcal{V}_{\kappa}(\mathcal{Z})})$ corresponding to a positive definite kernel function $\kappa$, we use the unit ball in the RKHS 
$$\mathcal{V}_{\kappa, 1}=\{v \in \mathcal{V}_{\kappa}(\mathcal{Z}): ||v||_{\mathcal{V}_{\kappa}(\mathcal{Z})} \le 1\}$$ 
for the set of discriminators used for EIPM.
The closed-form formula of $\hat{\textup{EIPM}}^{\gamma}_{\mathcal{V}_{\kappa, 1}} (\bm{Z}; S|Y=1)$ is given in the following proposition. 
\begin{proposition} \label{pro_estimator_derivation_EO}
    For given $\gamma>0$, $h \in \mathcal{H}$, $\{\bm{X}_i , S_i\}_{i=1}^n$ and $\bm{Z}_i = h(\bm{X}_i)$,  $\hat{\textup{EIPM}}^{\gamma}_{\mathcal{V}_{\kappa, 1}} (\bm{Z}; S| Y=1)$ is derived as
    \begin{align*}
    \hat{\textup{EIPM}}^{\gamma}_{\mathcal{V}_{\kappa, 1}} (\bm{Z}; S| Y=1) = \frac{1}{n_1} \sum_{i : Y_i=1} \left[ \sum_{j,k \neq i} [\tilde{A}_{\gamma}]_{i,j} [\tilde{A}_{\gamma}]_{i,k}  \kappa(\bm{Z}_j, \bm{Z}_k)
    \right]^{\frac{1}{2}},
    \end{align*}
    where $\tilde{A}_{\gamma}$ is the $n \times n$ matrix defined by
    $$[\tilde{A}_{\gamma}]_{i,j} = 
    \left( \frac{K_{\gamma}(S_i , S_j)}{\sum_{j \neq i} K_{\gamma}(S_i , S_j) \mathbb{I}(Y_j = 1) } - \frac{1}{n_1 - 1} \right) \cdot \mathbb{I}(Y_j = 1) .$$
\end{proposition}

\begin{proof}
    For simplicity, we denote 
    $$\tilde{w}_{\gamma}(j ; i) := \frac{K_{\gamma}(S_j, S_i) \mathbb{I}(Y_j = 1)}{\sum_{j \neq i} K_{\gamma}(S_j, S_i) \mathbb{I}(Y_j = 1)}.$$
    
    Lemma 6 of \cite{gretton2012kernel} states that for independent random variables $U, U^{\prime} \sim \mbP^U$
    and for independent random variables $T, T^{\prime} \sim \mbP^T$, 
$$\textup{IPM}_{\mathcal{V}_{\kappa,1}}(\mbP^{U},\mbP^{T}) = \sqrt{\mathbb{E} \Big[ \kappa (U, U^{\prime}) + \kappa (T, T^{\prime}) - 2 \kappa (U, T)\Big]},$$
where the expectation is respect to $U$, $U^{\prime}$, $T$ and $T^{\prime}$.
Hence, for $\gamma>0$, we obtain
\begin{align*}
    &\hat{\textup{EIPM}}^{\gamma}_{\mathcal{V}_{\kappa, 1}} (\bm{Z}; S| Y=1) \\
    &=     \frac{1}{n_1} \sum_{i : Y_i=1}  \textup{IPM}_{\mathcal{V}_{\kappa, 1}} \left(\hat{\mbP}^{(-i), \gamma}_{\bm{Z}|S=S_i, Y=1} ,    \hat{\mbP}^{(-i)}_{\bm{Z}|Y=1}\right) \\
    &= \frac{1}{n_1} \sum_{i : Y_i=1} \left[ 
    \sum_{j,k \neq i} \tilde{w}_\gamma (j;i) \tilde{w}_\gamma (k;i) \kappa(\bm{Z}_j, \bm{Z}_k) 
    + \sum_{j,k \neq i} \frac{1}{(n_1-1)^2} \kappa(\bm{Z}_j, \bm{Z}_k) 
    - 2 \sum_{j,k \neq i} \frac{1}{n_1-1} \tilde{w}_\gamma (j;i) \kappa(\bm{Z}_j, \bm{Z}_k) \right]^{\frac{1}{2}} \\
    &= \frac{1}{n_1} \sum_{i : Y_i=1} \left[ 
    \sum_{j,k \neq i} \left( \tilde{w}_\gamma (j;i) \tilde{w}_\gamma (k;i) 
    +\frac{1}{(n_1-1)^2}  
    - \frac{2}{n_1-1} \tilde{w}_\gamma (j;i) \right)  \kappa(\bm{Z}_j, \bm{Z}_k) \right]^{\frac{1}{2}} \\
    &= \frac{1}{n_1} \sum_{i : Y_i=1} \left[ 
    \sum_{j,k \neq i} \left( \tilde{w}_\gamma (j;i) \tilde{w}_\gamma (k;i) 
    +\frac{1}{(n_1-1)^2}  
    - \frac{1}{n_1 -1} \tilde{w}_\gamma (j;i) - \frac{1}{n_1 -1} \tilde{w}_\gamma (k;i) \right)  \kappa(\bm{Z}_j, \bm{Z}_k) \right]^{\frac{1}{2}} \\
    &= \frac{1}{n_1} \sum_{i : Y_i=1} \left[ 
    \sum_{j,k \neq i} \left( \tilde{w}_\gamma (j;i) - \frac{1}{n_1 -1} \right) 
    \left( \tilde{w}_\gamma (k;i) - \frac{1}{n_1 -1} \right)  
    \kappa(\bm{Z}_j, \bm{Z}_k) \right]^{\frac{1}{2}} \\
    &= \frac{1}{n_1} \sum_{i : Y_i=1} \left[ \sum_{j,k \neq i} [\tilde{A}_{\gamma}]_{i,j} [\tilde{A}_{\gamma}]_{i,k}  \kappa(\bm{Z}_j, \bm{Z}_k) 
    \right]^{\frac{1}{2}},
\end{align*}
where we use the fact that $\kappa(\bm{Z}_j, \bm{Z}_k) = \kappa(\bm{Z}_k, \bm{Z}_j)$ for the fourth equality.
\end{proof}

Based on the foregoing discussions, the algorithm of  FREM for EO is summarized in Algorithm \ref{alg:eipm_EO}.
\begin{algorithm}[t]
\caption{FREM for EO}
\label{alg:eipm_EO}
    \begin{algorithmic}[1]
        \REQUIRE 1. Network parameters.
        \\
        $\theta$: Parameter of the representation encoder $h.$
        \\
        $\phi$: Parameter of prediction head $f.$
        \REQUIRE 2. Hyper-parameters.
        \\
        $\lambda:$ Regularization parameter.
        \\
        $\textup{lr}:$ Learning rate.
        \\
        $T$: Training epochs.
        \\
        $n_{\textup{mb}}:$ Mini-batch size.
        \\
        $\gamma$ : Radius of kernel for EIPM estimation.
        \\
        \FOR{$t = 1, \cdots, T$}
        \STATE Randomly sample a mini-batch $(\bm{x}_{i}, y_{i}, s_{i})_{i=1}^{n_{\textup{mb}}}$
        \newline
        \STATE \textbf{(Compute task loss)}
        \\
        $\mathcal{L}_{\textup{sup}}(\theta, \phi) = \frac{1}{n_{\textup{mb}}} \sum_{i=1}^{n_{\textup{mb}}} l(y_{i}, f_{\phi}(h_{\theta}(\bm{x}_{i})) $ 
        \newline
        \STATE Compute
        $$n_{\textup{mb},1} = |\{i \in [n_{\textup{mb}}] : Y_i = 1\}|$$
        \STATE \textbf{(Compute EIPM)}
        \\
        Compute $n_{\textup{mb}} \times n_{\textup{mb}}$ matrix $\tilde{A}_{\gamma}$ by
        $$[\tilde{A}_{\gamma}]_{i,j} = \left( \frac{K_{\gamma}(s_i , s_j)}{\sum_{j \neq i} K_{\gamma}(s_i , s_j) } - \frac{1}{n_{\textup{mb},1}-1} \right) \cdot \mathbb{I}(Y_j = 1)$$
        \STATE Compute
        $$
        \mathcal{L}_{\textup{fair}}(\theta) = \sum_{i : Y_i = 1} \left( \sum_{j,k \neq i} [A_\gamma]_{i,j}   [A_\gamma]_{i,k} \kappa(j, k) \right)^{\frac{1}{2}}
        $$
        where $\kappa(j, k) := \kappa(h_{\theta}(\bm{x}_{j}), h_{\theta}(\bm{x}_{k})).$
        \STATE $ \mathcal{L}_{\textup{fair}}(\theta) \leftarrow \frac{1}{ n_{\textup{mb},1} } \mathcal{L}_{\textup{fair}}(\theta) $
        \STATE \textbf{(Total loss)}
        \\
        $\mathcal{L}(\theta, \phi) = \mathcal{L}_{\textup{sup}}(\theta, \phi) + \lambda \mathcal{L}_{\textup{fair}}(\theta)$
        \STATE \textbf{(Parameter updates)}
        \\
        $\theta \leftarrow \theta - \textup{lr} \cdot\nabla_{\theta} \mathcal{L} (\theta, \phi)
        \newline
        \phi \leftarrow \phi - \textup{lr}\cdot\nabla_{\phi} \mathcal{L} (\theta, \phi)$ 
        \ENDFOR
        \newline
        \textbf{Return} $\theta$ and $\phi$
    \end{algorithmic}
\end{algorithm}

\clearpage

\section{Experiments for synthetic dataset}\label{appen:syn}
\renewcommand{\theequation}{D.\arabic{equation}}

\subsection{Calculation of the true EIPM values in Section \ref{sec:simul}} \label{appen:syn_true_1}

Since
$$
\begin{pmatrix}
    S \\
    X^{(1)} \\
    X^{(2)} \\
\end{pmatrix}
\sim 
N \left( 
\begin{bmatrix}
    0   \\
    0 \\
    0 \\
\end{bmatrix},
\begin{bmatrix}
    1 & \rho & 0  \\
    \rho & 1 & 0  \\
    0 & 0 & 1  \\
\end{bmatrix}
\right)
$$
and $\bm{Z} = w_1 X^{(1)} + w_2 X^{(2)}$, we obtain
$$ \begin{pmatrix}
    S \\
    Z \\
\end{pmatrix}
\sim 
N \left( 
\begin{bmatrix}
    0   \\
    0 \\
\end{bmatrix},
\begin{bmatrix}
    1 & w_1 \rho \\
    w_1 \rho & w_1^2 + w_2^2   \\
\end{bmatrix}
\right)
$$
and hence
$
\bm{Z} \sim \mathcal{N} \left( 0, 1 \right)
$
and
$
\bm{Z} | S = s \sim \mathcal{N} \left( w_1 \rho s, 1 - w_1^2 \rho^2 \right),
$ where we use $w_1^2 + w_2^2 = 1$.

To obtain the true EIPM value, we use the fact that
for two given Gaussian distributions $\mbP^{1} = \mathcal{N}(\mu_{1}, \sigma_{1}^{2})$ and
$\mbP^{2} = \mathcal{N}(\mu_{2}, \sigma_{2}^{2}),$
the expected kernel
$
\kappa ( \mbP^{1}, \mbP^{2} ) = 
\mathbb{E}_{X_{1} \sim \mbP^{1}, X_{2} \sim \mbP^{2}} \kappa (X_{1}, X_{2})
$
is calculated as
$$
\kappa ( \mbP^{1}, \mbP^{2} ) =
\frac{ 1 }{\sqrt{1 + \sigma_1^2 + \sigma_2^2}}
e^{ -\frac{(\mu_1 - \mu_2)^2}{2 (1 + \sigma_1^2 + \sigma_2^2) } },
$$
as shown in \cite{muandet2012learning}.
Since both $\mbP_{\bm{Z}}$ and $\mbP_{\bm{Z} | S=s}$ are both Gaussian distributions, we obtain
\begin{align*}
    \kappa ( \mbP_{\bm{Z}}, \mbP_{\bm{Z}} ) &=
\frac{ 1 }{\sqrt{3}}, \\
    \kappa ( \mbP_{\bm{Z} | S=s}, \mbP_{\bm{Z} | S=s} ) &=
\frac{ 1 }{\sqrt{3 - 2 w_1^2 \rho^2}} \\
    \kappa ( \mbP_{\bm{Z}}, \mbP_{\bm{Z} | S=s} ) &=
\frac{ 1 }{\sqrt{3 - w_1^2 \rho^2}} 
e^{ -\frac{w_1^2 \rho^2 s^2}{2 (3 - w_1^2 \rho^2) } }.
\end{align*}
Then, we can directly derive
\begin{align*}
    \textup{IPM}_{\mathcal{V}_{\kappa, 1}} \left( \mbP_{\bm{Z} | S=s}, \mbP_{\bm{Z}} \right) 
    = \sqrt{\kappa (\mbP_{\bm{Z}}, \mbP_{\bm{Z}}) + \kappa (\mbP_{\bm{Z}|S=s}, \mbP_{\bm{Z}|S=s}) + \kappa_{\bm{Z}_{s} \bm{Z}} = \kappa (\mbP_{\bm{Z}|S=s}, \mbP_{\bm{Z}})},
\end{align*}
using \cite{gretton2012kernel}.
Finally, we obtain the true EIPM value by Monte-Carlo simulation with respect to $S$. 
That is,
\begin{align}
    \mathbb{E}_{S} \left[ \textup{IPM}_{\mathcal{V}_{\kappa,1}} (\mbP_{\bm{Z}|S} ,\mbP_{\bm{Z}}) \right] \approx
\frac{1}{N} \sum_{i=1}^N \textup{IPM}_{\mathcal{V}_{\kappa,1}} \left( \mbP_{\bm{Z} | S=S_i}, \mbP_{\bm{Z}} \right) \label{monte_simul}
\end{align}
where $S_1, \dots, S_N \overset{i.i.d.}{\sim} \mathcal{N}(0,1)$ for $N = 100,000.$

\subsection{Additional results for Fig. \ref{fig:synthetic1}}
\label{Appen:syn_result_1}

In Section \ref{sec:simul}, we randomly generate synthetic datasets 100 times and obtain 100 EIPM estimates.
Table \ref{tab:synthetic1} shows the resulting biases, MAEs and RMSEs of the estimates with various $n_{\textup{bins}}$ or $\gamma$s.
This simulation results confirm that proposed estimator is more accurate and more stable than the binning estimator, consistently on various encoder functions.

\begin{table*}[h]
    \centering
    \caption{\textbf{Additional results for Fig. \ref{fig:synthetic1}.} Comparison between the proposed estimator and the binning estimator for three cases of $h:$ $(w_{1}, w_{2}) \in \{ (\sqrt{0.2}, \sqrt{0.8}), (\sqrt{0.5}, \sqrt{0.5}), (\sqrt{0.8}, \sqrt{0.2}) \}.$
    For each case, the best results (the lowest values of Bias, MAE, and RMSE) of the binning estimator (w.r.t. the number of bins $n_{\textup{bins}}$) and the proposed estimator (w.r.t. $\gamma$) are highlighted by \underline{underlining} and \textbf{bold} face, respectively.}
    \begin{tabular}{|c|c||c|c|c|c|c|c|}
        \toprule
        \multirow{2}{*}{$(w_{1}, w_{2})$} & \multirow{2}{*}{Error measure $(\times 10^{-2})$} & \multicolumn{3}{c|}{Binning} & \multicolumn{3}{c|}{Proposed $\checkmark$}
        \\
        & & $n_{\textup{bins}}=2$ & $n_{\textup{bins}}=3$ & $n_{\textup{bins}}=4$ & $\gamma = 0.3$ & $\gamma = 0.5$ & $\gamma = 0.7$
        \\
        \midrule
        \multirow{3}{*}{$(\sqrt{0.2}, \sqrt{0.8})$} & Bias & \underline{1.16} & 3.66 & 5.62 & 5.42 & 2.02 & \textbf{-0.02}
        \\
        & MAE & \underline{2.75} & 3.74 & 5.62 & 5.42 & 2.27 & \textbf{1.43}
        \\
        & RMSE & \underline{3.42} & 4.46 & 6.15 & 5.81 & 2.84 & \textbf{1.80}
        \\
        \midrule
        \multirow{3}{*}{$(\sqrt{0.5}, \sqrt{0.5})$} & Bias & \underline{0.44} & 2.70 & 3.85 & 3.59 & \textbf{0.18} & -2.07
        \\
        & MAE & \underline{3.27} & 3.63 & 4.06 & 3.72 & \textbf{2.27} & 2.67
        \\
        & RMSE & \underline{3.98} & 4.50 & 4.91 & 4.55 & \textbf{2.74} & 3.19
        \\
        \midrule
        \multirow{3}{*}{$(\sqrt{0.8}, \sqrt{0.2})$} & Bias & \underline{0.59} & 2.49 & 3.41 & 2.86 & \textbf{-0.43} & -2.98
        \\
        & MAE & \underline{3.02} & 3.28 & 3.74 & 3.29 & \textbf{2.26} & 3.22
        \\
        & RMSE & \underline{3.66} & 4.05 & 4.51 & 4.07 & \textbf{2.69} & 3.75
        \\
        \bottomrule
    \end{tabular}
    \label{tab:synthetic1}
\end{table*}

\subsection{Additional experiments for multi-dimensional representations} \label{Appen:syn_multi}

One may consider a plug-in estimator for EIPM, where the Nadaraya–Watson (NW) density estimators of $\bm{Z}$ and $\bm{Z}|S=S_i$
are plugged into EIPM.
To be more specific, for given $\{(\bm{Z}_i, S_i)\}_{i=1}^n$ and $\bm{z} \in \mathcal{Z}$, Nadaraya–Watson density estimators of $q(\bm{z})$ and $q(\bm{z}|S=s)$ are given as
$$\hat{q}^{(-i)}(\bm{z}) = \frac{ \sum_{j \neq i} K_\gamma(\bm{Z}_j , \bm{z})}{n-1}$$
and
$$\hat{q}^{(-i)}(\bm{z}|S=s) = \frac{ \sum_{j \neq i} K_\gamma((\bm{Z}_j^{\top}, S_j)^{\top} , (\bm{z}^{\top}, s)^{\top})}{\sum_{j \neq i} K_\gamma(S_j , s)},$$
respectively \cite{de2003conditional}.

A problem of this plug-in estimator is the integration in the EIPM, where the curse of dimensionality emerges.
High-dimensional numerical integration is known to be a very difficult problem, and 
naive approaches such as approximating the integration by the sum at prespecified grid points
usually fail. In this subsection, we investigate how well the EIPM estimator obtained by the NW density estimator and naive integration. 
First, for each $i \in [n]$, we can estimate $\textup{IPM}_{\mathcal{V}_{\kappa, 1}} \left(\mathbb{P}_{\bm{Z}|S=S_i} ,\mathbb{P}_{\bm{Z}} \right)$ by
\begin{align*}
    \int_{\bm{z}} \int_{\bm{z}^{\prime}} \kappa(\bm{z}, \bm{z}^{\prime}) \left(\hat{q}^{(-i)}(\bm{z}) \hat{q}^{(-i)}(\bm{z}^{\prime}) +
\hat{q}^{(-i)}(\bm{z}|S=s) \hat{q}^{(-i)}(\bm{z}^{\prime}|S=s)
- 2 \hat{q}^{(-i)}(\bm{z}) \hat{q}^{(-i)}(\bm{z}^{\prime}|S=s) \right)
d \bm{z} d \bm{z}^{\prime}.
\end{align*}
Using the importance sampling, this estimator can be approximated by
\begin{align*}
    \widehat{\textup{IPM}}^{\text{NW}}_{\mathcal{V}_{\kappa, 1}} \left(\mathbb{P}_{\bm{Z}|S=S_i} ,\mathbb{P}_{\bm{Z}} \right) = \frac{1}{R^2} \sum_{r_1 = 1}^{R} \sum_{r_2 = 1}^{R} \frac{\kappa(\bm{z}_{r_1}, \bm{z}_{r_2}^{\prime})}{\mathfrak{p}(\bm{z}_{r_1})\mathfrak{p}( \bm{z}_{r_2}^{\prime})} \Bigg(\hat{q}^{(-i)}(\bm{z}_{r_1}) \hat{q}^{(-i)}(\bm{z}_{r_2}^{\prime}) +
 \hat{q}^{(-i)}(\bm{z}_{r_1}|S=S_i) \hat{q}^{(-i)}(\bm{z}_{r_2}^{\prime}|S=S_i)& \\
 - 2 \hat{q}^{(-i)}(\bm{z}_{r_1}) \hat{q}^{(-i)}(\bm{z}_{r_2}^{\prime}|S=S_i)& \Bigg),
\end{align*}
where
$\bm{z}_1, \dots, \bm{z}_R, \bm{z}^{\prime}_1, \dots, \bm{z}^{\prime}_R$ are sampled from some proposal distribution whose density function is $\mathfrak{p}$.
Then, the plug-in estimator for EIPM is given as
\begin{align*}
    \hat{\textup{EIPM}}^{\text{NW}}_{\mathcal{V}} (\bm{Z}; S) := \frac{1}{n} \sum_{i=1}^n \widehat{\textup{IPM}}^{\text{NW}}_{\mathcal{V}_{\kappa, 1}} \left(\mathbb{P}_{\bm{Z}|S=S_i} ,\mathbb{P}_{\bm{Z}} \right).
\end{align*}

However, as is well-known, Nadaraya–Watson density estimators are known to not perform well in high dimensions \cite{bengio2005curse}.
Since we use Nadaraya–Watson density estimator for $\bm{Z}$, 
it can be expected that corresponding EIPM estimator will not perform well when the representation is high-dimensional.
To verify this claim experimentally, we consider another simulation design for multi-dimensional $\bm{Z}$.
For given correlation $\rho$, consider
$$
\begin{pmatrix}
    S \\
    Z^{(1)} \\
    ... \\
    Z^{(m)}
\end{pmatrix}
\sim 
N \left( 
\begin{bmatrix}
    0   \\
    0 \\
    ... \\
    0
\end{bmatrix},
\begin{bmatrix}
    1 & \rho/\sqrt{m} & ...  & \rho/\sqrt{m} \\
    \rho/\sqrt{m} & 1/m & ...  & 0 \\
    ... & ... & ... & ... \\
    \rho/\sqrt{m} & 0 & ... & 1/m
\end{bmatrix}
\right).
$$
Note that the covariance matrix is positive definite if and only if $0 \leq \rho < 1 / \sqrt{m}$.
Then, we obtain
$$S \sim N\left(0, 1 \right),\qquad \bm{Z} \sim N\left(\bm{0}_m,\frac{1}{m} I_m \right)$$
and
$$\bm{Z}|S=s \sim N\left( \frac{1}{\sqrt{m}} \rho s \cdot \bm{1}_m, \frac{1}{m}I_m - \frac{1}{m} \rho^2 \bm{1}_m \bm{1}_m^{\top}\right).$$
By \cite{muandet2012learning} and Sherman-Morrison formula, we get
\begin{align*}
    \kappa ( \mbP_{\bm{Z}}, \mbP_{\bm{Z}} ) = &
    \frac{1}{\left| \frac{1}{m} I_m + \frac{1}{m} I_m + I_m \right|^{\frac{1}{2}}}
    = \frac{1}{\sqrt{\left(1+ \frac{2}{m}\right)^m}},
\end{align*}
\begin{align*}
    \kappa ( \mbP_{\bm{Z}|S=s}, \mbP_{\bm{Z}|S=s} ) = &
    \frac{1}{\left| \frac{1}{m}I_m - \frac{1}{m} \rho^2 \bm{1}_m \bm{1}_m^{\top} + \frac{1}{m}I_m - \frac{1}{m} \rho^2 \bm{1}_m \bm{1}_m^{\top} + I_m \right|^{\frac{1}{2}}} \\
    = & \frac{1}{\sqrt{\left(1+\frac{2}{m}\right)^{m-1} ( 1+\frac{2}{m}- 2 \rho^2 )}}
\end{align*}
and
\begin{align*}
    \kappa ( \mbP_{\bm{Z}}, \mbP_{\bm{Z}|S=s} ) = &
    \frac{1}{\left| \frac{1}{m} I_m + \frac{1}{m}I_m - \frac{1}{m} \rho^2 \bm{1}_m \bm{1}_m^{\top} + I_m \right|^{\frac{1}{2}}} \exp 
    \left( -\frac{1}{2m} \rho^2 s^2 \bm{1}_m^{\top} \left(\frac{1}{m} I_m + \frac{1}{m}I_m - \frac{1}{m} \rho^2 \bm{1}_m \bm{1}_m^{\top} + I_m\right)^{-1} \bm{1}_m \right)\\
    = & \frac{1}{\sqrt{\left(1+\frac{2}{m}\right)^{m-1} ( 1+\frac{2}{m}-  \rho^2 )}} 
    \exp 
    \left( -\frac{1}{2m} \rho^2 s^2 \bm{1}_m^{\top} 
    \left(\frac{m}{m+2} I_m + \frac{\rho^2}{(m+2)\left(\frac{m+2}{m}- \rho^2 \right)}\bm{1}_m \bm{1}_m^{\top} \right)\bm{1}_m \right) \\
    = & \frac{1}{\sqrt{\left(1+\frac{2}{m}\right)^{m-1} ( 1+\frac{2}{m}- \rho^2 )}} 
    \exp 
    \left( -\frac{1}{2m} \rho^2 s^2 m^2  
    \left( \frac{1}{m+2} + \frac{\rho^2}{(m+2)\left(\frac{m+2}{m}- \rho^2 \right)} \right)\right) \\
    = & \frac{1}{\sqrt{\left(1+\frac{2}{m}\right)^{m-1} ( 1+\frac{2}{m}- \rho^2 )}} 
    \exp 
    \left( -\frac{1}{2} \cdot  
    \frac{ \rho^2 s^2 m }{m+2-m \rho^2} \right).     
\end{align*}
Then, we can obtain the true EIPM value by Monte-Carlo simulation similar with (\ref{monte_simul}).


\begin{table*}[h]
    \centering
    \caption{ {\textbf{Simulation results}: Comparison between the proposed estimator, binning estimator and Nadaraya–Watson estimator.
    For each case, the best results (the lowest values of RMSE) are highlighted by \textbf{bold} face.}}
    \begin{tabular}{|c|c||c|c|c|c|c|c|c|c|c|}
        \toprule
        \multirow{2}{*}{$m$} & \multirow{2}{*}{$n$} & \multicolumn{3}{c|}{Binning} & \multicolumn{3}{c|}{Nadaraya–Watson} & \multicolumn{3}{c|}{Proposed $\checkmark$}
        \\
        & & $n_{\textup{bins}}=2$ & $n_{\textup{bins}}=3$ & $n_{\textup{bins}}=4$ & $\gamma = 0.3$ & $\gamma = 0.5$ & $\gamma = 0.7$ & $\gamma = 0.3$ & $\gamma = 0.5$ & $\gamma = 0.7$
        \\
        \midrule
        \multirow{6}{*}{$1$} & 100 & 0.047& 0.037& 0.044& 0.038& 0.054& 0.065& 0.039& \textbf{0.034}& 0.044
        \\
        & 140 & 0.038& 0.034& 0.037& 0.039& 0.054 & 0.066& 0.033& \textbf{0.032}& 0.042
        \\
        & 180 & 0.033& 0.031& 0.031& 0.049& 0.061& 0.072 & 0.028& \textbf{0.027}& 0.041
        \\
        & 220 & 0.029& 0.029& 0.035& 0.055& 0.065& 0.075& 0.027& \textbf{0.023}& 0.035
        \\
        & 260 & 0.03 & 0.025& 0.029& 0.043& 0.055& 0.067 & \textbf{0.024}& 0.027& 0.041
        \\
        & 300 & 0.026& 0.022& 0.024& 0.047 & 0.058& 0.069 & \textbf{0.02} & 0.026& 0.04
        \\
        \midrule
        \multirow{6}{*}{$3$} & 100 & 0.033& 0.049& 0.068& 0.037& 0.049& 0.057& 0.066& 0.03 & \textbf{0.017}
        \\
        & 140 & 0.031& 0.046& 0.059& 0.041& 0.05& 0.057& 0.057& 0.026& \textbf{0.016}
        \\
        & 180 & 0.023& 0.034& 0.046& 0.046& 0.055& 0.061 & 0.043& 0.017& \textbf{0.016}
        \\
        & 220 & 0.021& 0.03 & 0.039& 0.045& 0.053& 0.06& 0.035& \textbf{0.014}& 0.018
        \\
        & 260 & 0.02 & 0.026& 0.036& 0.048& 0.056& 0.062& 0.031& \textbf{0.013}& 0.019
        \\
        & 300 & 0.019& 0.022& 0.029& 0.047& 0.05& 0.06& 0.025& \textbf{0.014}& 0.023
        \\
        \midrule
        \multirow{6}{*}{$10$} & 100 & 0.044& 0.07 & 0.095& 0.044& 0.046& 0.047& 0.098& 0.053& \textbf{0.028}
        \\
        & 140 & 0.035& 0.056& 0.075& 0.044& 0.046& 0.046& 0.078& 0.041& \textbf{0.019}
        \\
        & 180 & 0.028& 0.049& 0.064& 0.045& 0.046& 0.047& 0.065& 0.033& \textbf{0.014}
        \\
        & 220 & 0.023& 0.041& 0.055& 0.046& 0.047& 0.047& 0.055& 0.026& \textbf{0.009}
        \\
        & 260 & 0.02 & 0.038& 0.049& 0.046& 0.046& 0.047& 0.05 & 0.023& \textbf{0.008}
        \\
        & 300 & 0.019& 0.034& 0.045& 0.045& 0.047& 0.048& 0.045& 0.02 & \textbf{0.006}
        \\        
        \bottomrule
    \end{tabular}
    \label{tab:synthetic_multi}
\end{table*}

For $m \in \{1,3,10\}$, we consider $\rho = \frac{1}{3\sqrt{m}}$ for the covariance matrix to ensure it is positive definite.
$n \in \{100, 120, 140, 180, 220, 260, 300\}$ samples are generated from this probabilistic model and the proposed estimator $\hat{\textup{EIPM}}_{\mathcal{V}_{\kappa,1}}^{\gamma} (h(\bm{X}); S)$ is computed using Proposition \ref{pro_estimator_derivation}.
We also consider the binning estimator and Nadaraya–Watson estimator.
We vary the number of bins $n_{bins}$ over $\{ 2, 3, 4 \}$ and consider the bandwidth $\gamma$ in $\{ 0.3, 0.5, 0.7 \}.$
For Nadaraya–Watson estimator, we use $N(\bm{0}_m , \frac{2}{m}I_m)$ for the proposal distribution and $R=1000$.

We randomly generate synthetic datasets 100 times and obtain 100 EIPM estimates.
Resulting RMSEs of the estimates with various $n$ and $m$ are provided in Table \ref{tab:synthetic_multi}.
As a result, we observe that the proposed estimator performed reliably well in all settings. 
Particularly, for high-dimensional $\bm{Z}$, 
we observe that increasing the sample size cannot improve the performance of the Nadaraya-Watson estimator.

\clearpage

\section{Experiments for Real dataset}\label{appen:exp}

\subsection{Datasets}\label{appen:data}
       
The information of the two tabular datasets and two graph datasets used in the numerical studies are summarized in Table \ref{tab:datasets}.
We standardize each input feature in $\bm{X}$ and the sensitive attribute $S$ into $[0, 1]$ by the min-max scaling.

\begin{table}[h]
    \centering
    \caption{
    {\textbf{Dataset information:} Real datasets and corresponding tasks with pre-defined continuous sensitive attributes.}
    }
    \begin{tabular}{|c|c|c|c|c|c|}
        \toprule
        Dataset & Task & Input dimension $d$ & Sample size (Train / Test) & $Y$ & $S$ \\
        \midrule
        \textsc{Adult} & Classification & 101 & 32,561 / 12,661 & Income $> 50\$$k & Age \\
        \textsc{Crime} & Regression & 121 & 1,794 / 200 & Crime ratio & Black group ratio \\
        \textsc{Pokec-z} & Classification & 276 & 33,898 / 16,949 & Working field & Age \\
        \textsc{Pokec-n} & Classification & 265 & 33,284 / 16,643 & Working field & Age \\
        \bottomrule
    \end{tabular}
    \label{tab:datasets}
\end{table}

\textbf{The meaning of unfairness in the datasets}
In \textsc{Adult} dataset, the target variable is `Income  $> 50\$$k' and the sensitive attribute is `Age', and hence a large value  $\Delta \textup{\texttt{GDP}}$ or $\Delta \textup{\texttt{GEO}}$ implies that the prediction model   
decides elderly people as a rich more frequently, which would be unfair in certain situations. For example,
elderly people may be left out of government support for low-income people. 
In \textsc{Crime} dataset, the target variable is `Crime ratio' and
the sensitive attribute is `Black group ratio' of a given community, and hence a large value of $\Delta \textup{\texttt{GDP}}$ implies 
that the prediction model could simply decide a community having a higher ratio of black people as a community of high crime ratio, which would undesirable in various situations.

\textbf{Example of dataset bias}
Fig. \ref{fig:conditional} below shows the observation of the dataset bias on \textsc{Crime} dataset, in which we provide the conditional distributions of three representative features (PctKids2Par = \% of kids in family with two parents, PctPopUnderPov = \% of people under the poverty level and PctHousOccup = \% of housing occupied) with respect to the sensitive attribute (= black group ratio of a given community).
We divide the features by five groups to draw box plots using the 20, 40, 60, 80, 100\% quantiles of the sensitive attribute.
The results clearly show that the features have biases with respect to the sensitive attribute.
In contrast, Fig. \ref{fig:conditional2} below shows the mitigation of bias in learned representation by FREM.
We provide the conditional distributions of three randomly selected features of the learned fair representation by FREM.


\begin{figure*}[ht]
    \centering
    \includegraphics[width=0.6\textwidth]{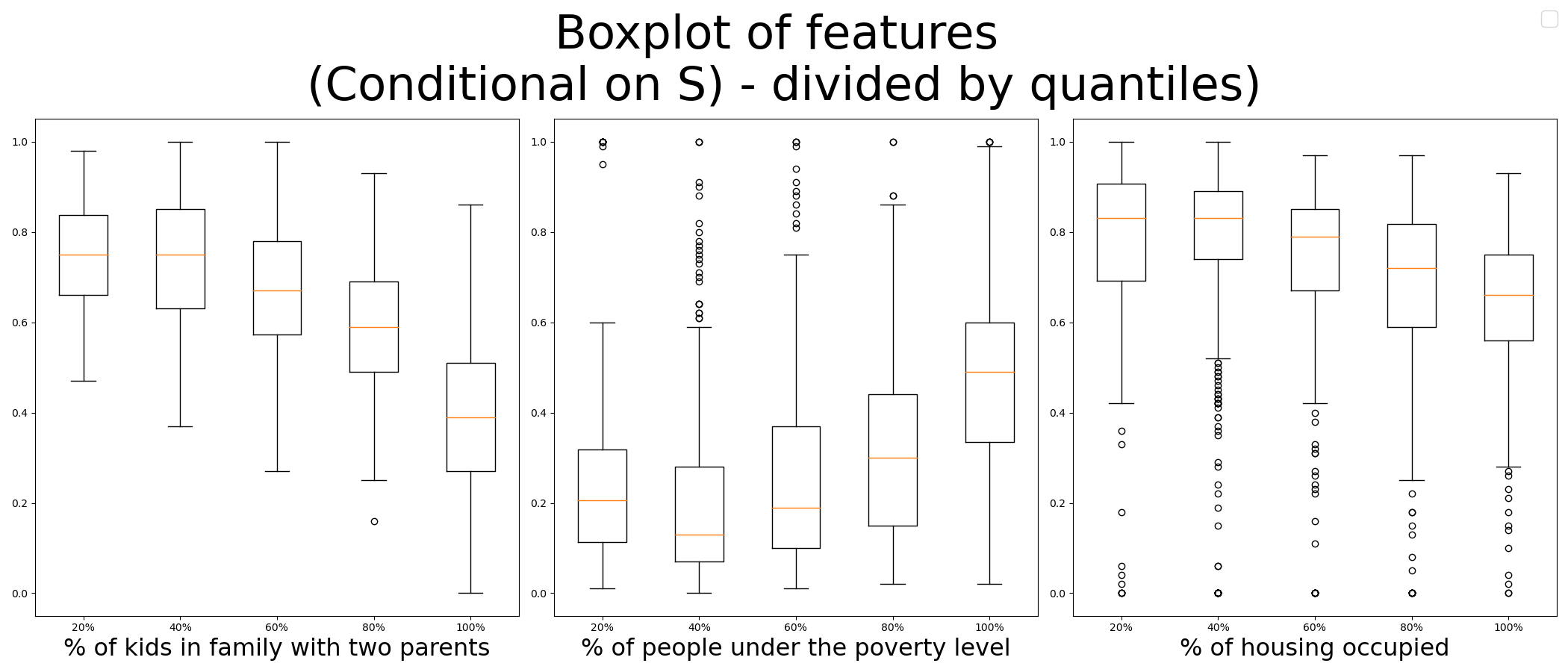}
    \caption{    \textbf{Bias on feature space:} Conditional distributions of three features (PctKids2Par = \% of kids in family with two parents, PctPopUnderPov = \% of people under the poverty level and PctHousOccup = \% of housing occupied) on \textsc{Crime} dataset.
    }
    \label{fig:conditional}
\end{figure*}

\begin{figure*}[ht]
    \centering
    \includegraphics[width=0.6\textwidth]{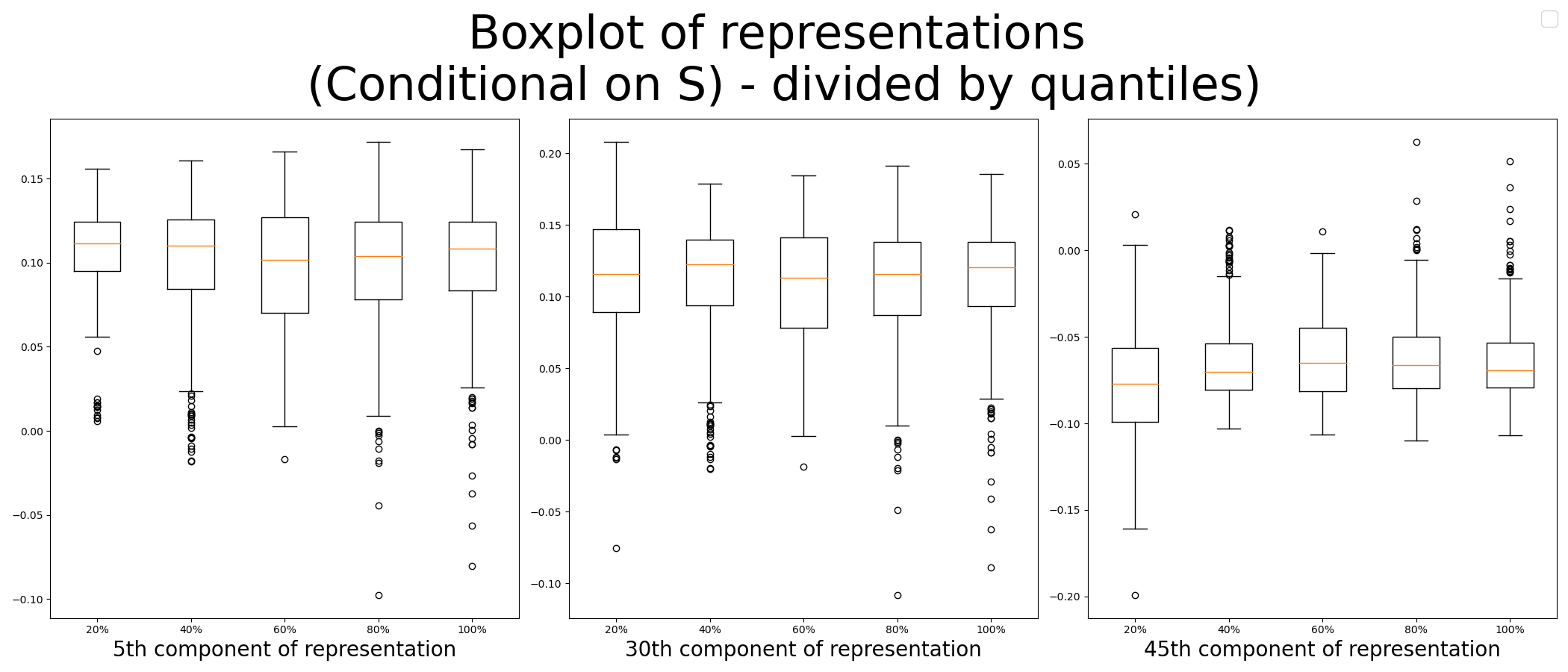}
    \caption{    \textbf{Bias on representation space:}
    Conditional distributions of three randomly selected features of the learned fair representation on \textsc{Crime} dataset.
    }
    \label{fig:conditional2}
\end{figure*}

\clearpage
\subsection{Experimental details}\label{appen:impl}

\subsubsection{Model network}
For tabular datasets, we use a two-layer neural network with the selu activation \cite{NIPS2017_5d44ee6f} and hidden node size $50.$
This network architecture is used in prior works \cite{mary2019fairness, jiang2021generalized}.
The representation of each input is the $50$-dimensional hidden vector extracted from the last layer before the linear prediction head.
For graph datasets, we use the GCN \cite{kipf2017semisupervised} and SGC \cite{pmlr-v97-wu19e} considered in \cite{jiang2021generalized}.

\subsubsection{Training hyperparameters}
For FREM and all baseline methods, the total training epochs is set to be $200.$
The Adam optimizer is used with the learning rate  $0.001$ and weight decay hyperparameter of $0.01.$
We set the batch size $1024$ for \textsc{Adult}, and $200$ for \textsc{Crime} dataset.
For the graph datasets \textsc{Pokec-n} and \textsc{Pokec-z}, we set the batch size $512.$

\subsubsection{Brief introduction about the baselines}\label{appen:ADV}

\begin{itemize}
    \item 
{
Reg-GDP \cite{jiang2021generalized} and Reg-HGR \cite{mary2019fairness} are regularization methods for continuous sensitive attributes.
These methods find a prediction model $g : \cX \to \mathcal{Y}$ that minimizes $$\mathcal{L}_{\textup{sup}}(g)+\lambda \Delta_{n}(g),$$ 
where $\mathcal{L}_{\textup{sup}}(g) = \frac{1}{n} \sum_{i=1}^{n} l(y_{i}, g(\bm{x}_{i}))$ is the task loss, $\Delta_{n}(g)$ is a fairness related regularization term and $\lambda>0$ is the Lagrangian multiplier.
For $\Delta_{n}(g)$, Reg-GDP uses a kernel estimator of $\Delta \texttt{GDP}$ while Reg-HGR uses an estimator of $\Delta \texttt{HGR}$ via a density estimation on a regular square grid.}

    \item 
{
ADV, an adversarial learning approach for fair representation learning for continuous sensitive attributes, 
is a new algorithm developed by modifying the algorithm in \cite{10.1145/3278721.3278779}.
In the method of \cite{10.1145/3278721.3278779}, under the setting where $S$ is binary, the discriminator is trained to predict $S$ using $\hat{Y}(=g(\bm{X}))$.
We apply this method to continuous $S$ as well, where the discriminator tries to predict the real-valued $S$ using the representation $\bm{Z}(=h(\bm{X}))$, and the encoder is trained to make it difficult for $\bm{Z}$ to predict $S.$ 
}

    \item 

{LAFTR \cite{Madras2018LearningAF}, MMD \cite{deka2023mmd} and sIPM-LFR \cite{kim2022learning} are fair representation learning methods for binary sensitive attributes. 
Their aim is to find an encoder $h : \mathcal{X} \to \mathcal{Z}$ such that $h(\bm{X})$ contains most of the information about $\bm{X}$ while ensuring that
$d(\mbP_{h(\bm{X})|S=0}, \mbP_{h(\bm{X})|S=1})$ is small for some deviance measure $d$ between two distributions.
For the deviance measure $d$, LAFTR uses Jensen-Shannon Divergence, while MMD and sIPM-LFR use the  IPM with the RKHS unit ball and $\{ \sigma(\theta^\top \bm{z}+\mu): \theta\in \mathbb{R}^m, \mu\in \mathbb{R}\}$ for the set of discriminator, respectively.
}    
    
\end{itemize}

\subsubsection{Calculation of mutual information (MI) as a fairness measure}\label{sec:appen-fairness_measures}

The MI results (i.e., $\texttt{MI}(\bm{Z}, S)$ and $\texttt{MI}(\hat{Y}, S)$) are calculated by `mutual\_info\_regression' function from the scikit-learn library, which implements the practical approaches from \cite{ross2014mutual}, based on the Kozachenko-Leonenko estimator \cite{kozachenko1987sample}.
Note that the estimator in \cite{ross2014mutual} has been widely used in diverse tasks \cite{9150630, franzmeyer2022learnmatterscrossdomainimitation, Liu_2024}.
    
    To verify the stability of the Kozachenko-Leonenko estimator we use, we evaluate the variation of the estimated values using the bootstrap method, based on the values of $\hat{Y}$ and $S$ from models learned by several algorithms.
    We randomly resample the test data with replacement 1,000 times, then report the (i) average, (ii) standard deviation, and (iii) coefficient of variation (= standard deviation $\div$ average) of the results.
    The results, presented in Table \ref{tab:mi_stable}, show that this estimator is stable, supporting its practical validity as a fairness measure.

    \begin{table*}[h]
        \centering
        \caption{\textbf{Stability of the used MI measure}: Averages, standard deviations, and coefficient of variations of $\texttt{MI} (\hat{Y}, S)$ using 1,000 bootstrap samples.
        Avg = average.
        Std = standard deviation.
        CV = coefficient of variation.
        (Left) \textsc{Adult} dataset. (Right) \textsc{Crime} dataset.
        }
        \begin{tabular}{|c||c|c|c|c|c|c|}
            \toprule
            Results of \texttt{MI}$(\hat{Y}, S)$ & \multicolumn{3}{c|}{\textsc{Adult}} & \multicolumn{3}{c|}{\textsc{Crime}}
            \\
            \midrule
            Algorithm & Avg & Std & CV & Avg & Std & CV
            \\
            \midrule
            MMD & 0.237 & 0.018 & 0.076 & 0.321 & 0.009 & 0.028
            \\
            Reg-GDP & 0.252 & 0.011 & 0.044 & 0.152 & 0.021 & 0.138
            \\
            FREM & 0.212 & 0.020 & 0.094 & 0.076 & 0.004 & 0.053
            \\
            \bottomrule
        \end{tabular}
        \label{tab:mi_stable}
    \end{table*}


\color{black} 

\clearpage
\subsection{Omitted experimental results}\label{appen:exp_results}

\subsubsection{Main results: Fairness-prediction trade-off}

\begin{figure*}[ht]
    \centering
    \fbox{
    \includegraphics[width=0.3\textwidth]{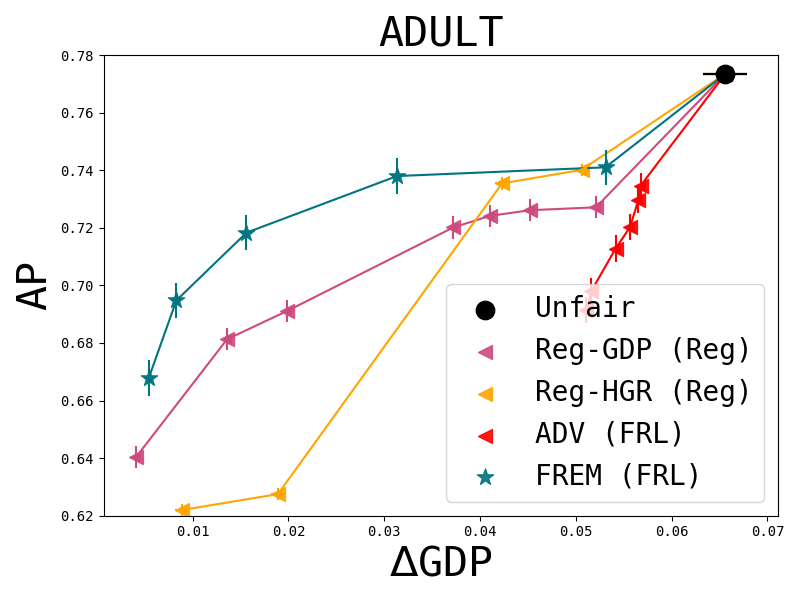}
    }
    \fbox{
    \includegraphics[width=0.3\textwidth]{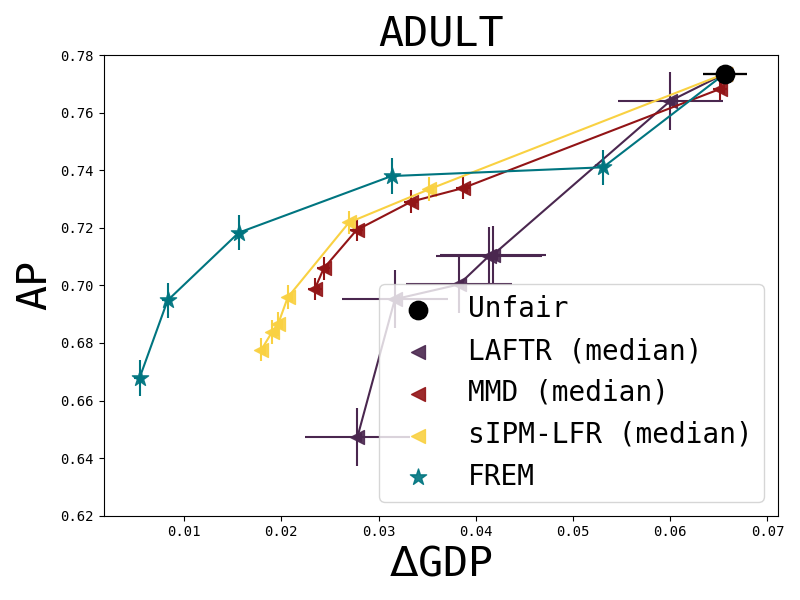}
    \includegraphics[width=0.3\textwidth]{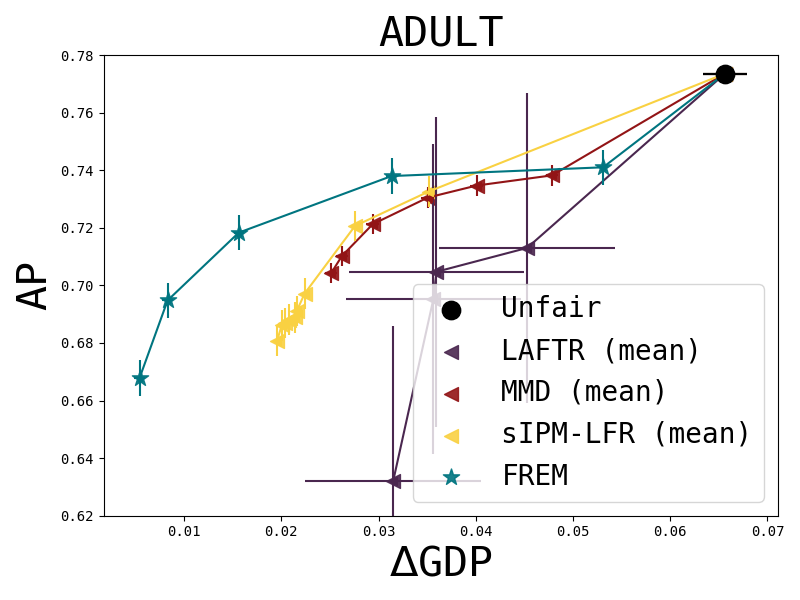}
    }
    \\~\\
    \fbox{
    \includegraphics[width=0.3\textwidth]{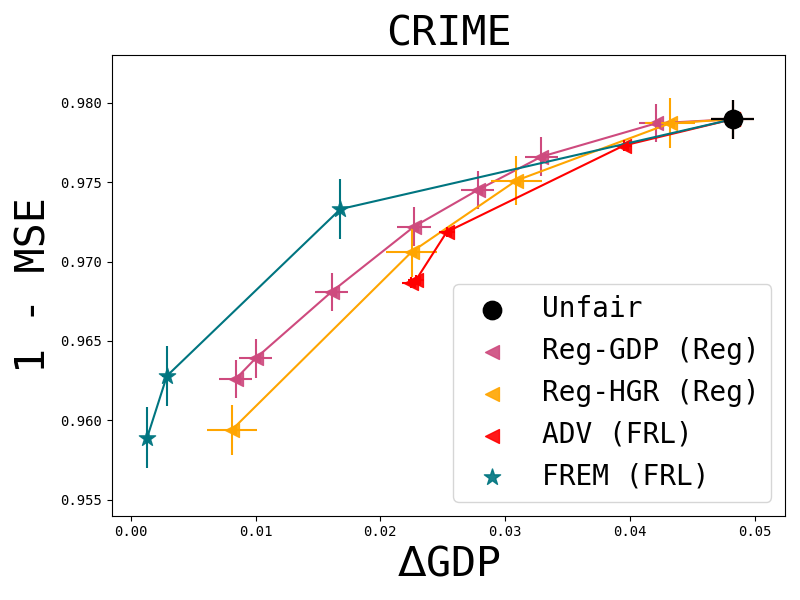}
    }
    \fbox{
    \includegraphics[width=0.3\textwidth]{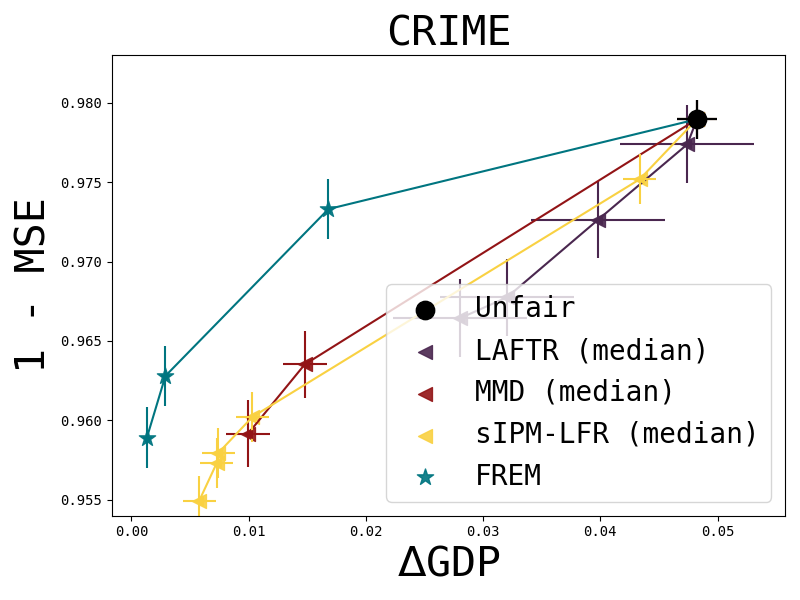}
    \includegraphics[width=0.3\textwidth]{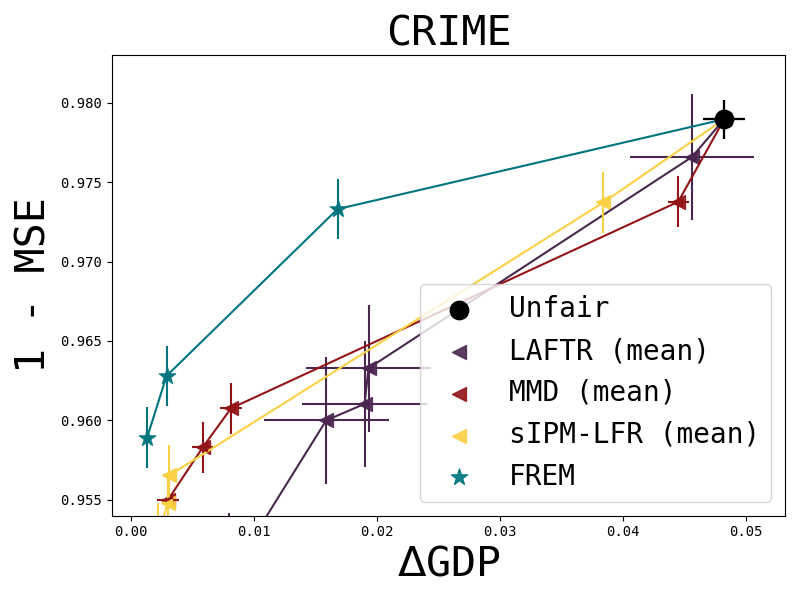}
    }
    \caption{
    Similar to Fig. \ref{fig:tradeoff-dp}.
    \textbf{Demographic Parity}: Pareto-front lines for fairness-prediction trade-off.
    (Top)
    \textsc{Adult} dataset:
    $\Delta \texttt{GDP}$ vs. \texttt{AP}.
    (Bottom)
    \textsc{Crime} dataset:
    $\Delta \texttt{GDP}$ vs. 1 - \texttt{MSE}.
    $\bullet$: Unfair,
    \textcolor{Regkernelcolor}{--$\blacktriangleleft$--}: Reg-GDP,
    \textcolor{orange}{--$\blacktriangleleft$--}: Reg-HGR,
    \textcolor{red}{--$\blacktriangleleft$--}: ADV,
    \textcolor{sIPMcolor}{--$\blacktriangleleft$--}: sIPM-LFR,
    \textcolor{MMDcolor}{--$\blacktriangleleft$--}: MMD,
    \textcolor{LAFTRcolor}{--$\blacktriangleleft$--}: LAFTR,
    \textcolor{Ourcolor}{--$\mathbf{\filledstar}$--}: FREM.
    }
    \label{fig:tradeoff-dp-appen}
\end{figure*}


\begin{table*}[h]
    \centering
    \caption{\textbf{Other fairness measures}: Comparison of algorithms in terms of two additional fairness measures, $\Delta \texttt{HGR-DP}$ and $\texttt{MI}(\hat{Y}, S)$.
    For each fairness measure, the best results are marked by bold face, which are obtained by FREM.
    (Top) \textsc{Adult} dataset. (Bottom) \textsc{Crime} dataset.
    }
    \scalebox{0.95}{
    \begin{tabular}{|c||c|c|c|c|c|c|c|c|c|c|c|}
        \multicolumn{12}{c}{\textsc{Adult}} \\ \midrule \midrule
        Algorithm & Unfair & \multicolumn{2}{c|}{LAFTR \cite{Madras2018LearningAF}} & \multicolumn{2}{c|}{MMD \cite{deka2023mmd}} & \multicolumn{2}{c|}{sIPM-LFR \cite{kim2022learning}} & {Reg-GDP} \cite{jiang2021generalized} & {Reg-HGR} \cite{mary2019fairness} & ADV \cite{10.1145/3278721.3278779} & FREM $\checkmark$ \\
        Binning & - & median & mean & median & mean & median & mean & - & - & - & - \\
        \midrule \midrule
        $\texttt{Acc} \ (\uparrow) $ & 0.847 & 0.814 & 0.822 & 0.818 & 0.824 & 0.825 & 0.825 & 0.826 & 0.757 & 0.827 & 0.827 \\
        \midrule
        $\Delta \texttt{HGR-DP} \ (\downarrow)$ & 0.474 & 0.336 & 0.381 & 0.291 & 0.301 & 0.282 & 0.282 & 0.300 & 0.306 & 0.434 & \textbf{0.172} \\
        $\texttt{MI}(\hat{Y}, S) \ (\downarrow)$ & 0.396 & 0.238 & 0.281 & 0.232 & 0.230 & 0.271 & 0.268 & 0.194 & 0.192 & 0.388 & \textbf{0.175} \\
        \bottomrule
        \multicolumn{12}{c}{}
        \\
        \multicolumn{12}{c}{}
        \\
        \multicolumn{12}{c}{\textsc{Crime}} \\ \midrule \midrule
        Algorithm & Unfair & \multicolumn{2}{c|}{LAFTR \cite{Madras2018LearningAF}} & \multicolumn{2}{c|}{MMD \cite{deka2023mmd}} & \multicolumn{2}{c|}{sIPM-LFR \cite{kim2022learning}} & {Reg-GDP} \cite{jiang2021generalized} & {Reg-HGR} \cite{mary2019fairness} & ADV \cite{10.1145/3278721.3278779} & FREM $\checkmark$ \\
        Binning & - & median & mean & median & mean & median & mean & - & - & - & - \\
        \midrule
        \midrule
        $1 - \texttt{MAE} \ (\uparrow)$ & 0.904 & 0.855 & 0.846 & 0.846 & 0.851 & 0.848 & 0.832 & 0.836 & 0.847 & 0.851 & 0.851 \\
        \midrule
        $\Delta \texttt{HGR-DP} \ (\downarrow)$ & 0.569 & 0.385 & 0.375 & 0.143 & 0.128 & 0.183 & 0.131 & 0.128 & 0.115 & 0.384 & \textbf{0.104} \\
        $\texttt{MI}(\hat{Y}, S) \ (\downarrow)$ & 0.346 & 0.170 & 0.169 & 0.064 & 0.053 & 0.068 & 0.053 & 0.078 & 0.050 & 0.226 & \textbf{0.022} \\
        \bottomrule
    \end{tabular}}
    \label{tab:other}
\end{table*}

\begin{figure*}[ht]
    \centering
    \fbox{
    \includegraphics[width=0.3\textwidth]{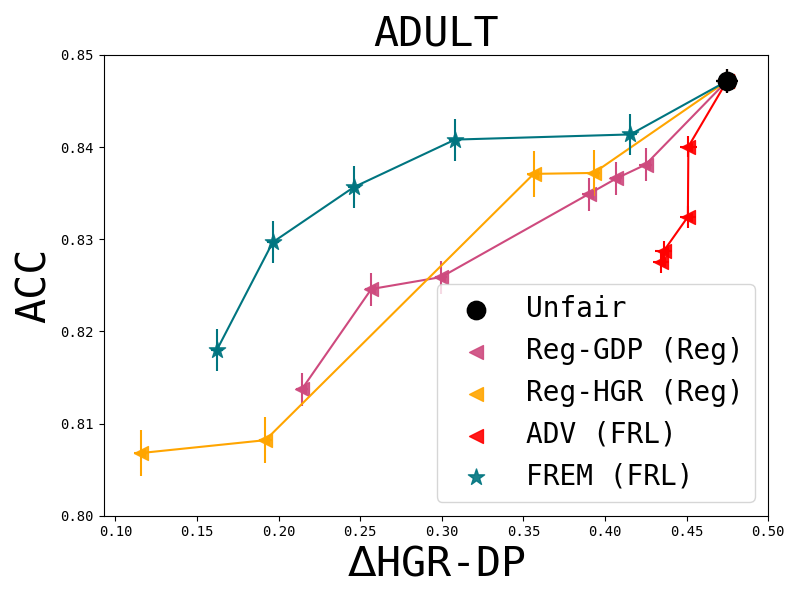}
    }
    \fbox{
    \includegraphics[width=0.3\textwidth]{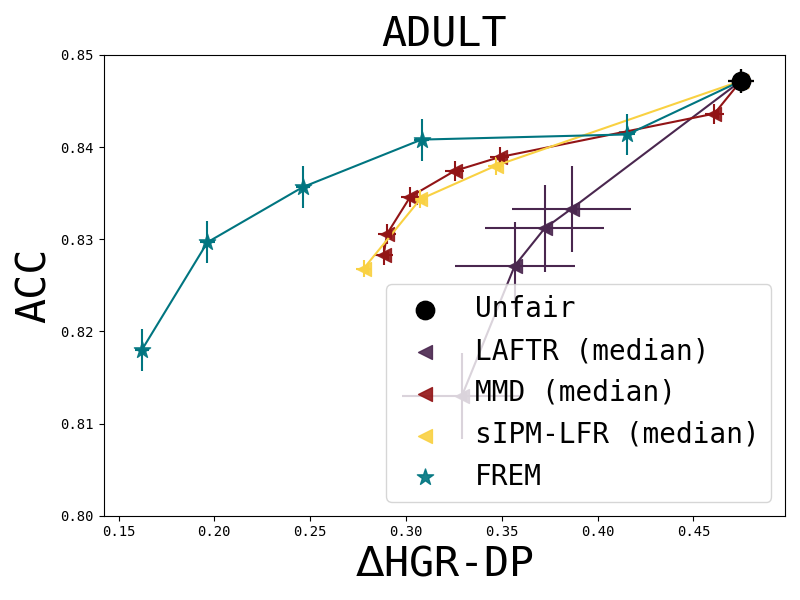}
    \includegraphics[width=0.3\textwidth]{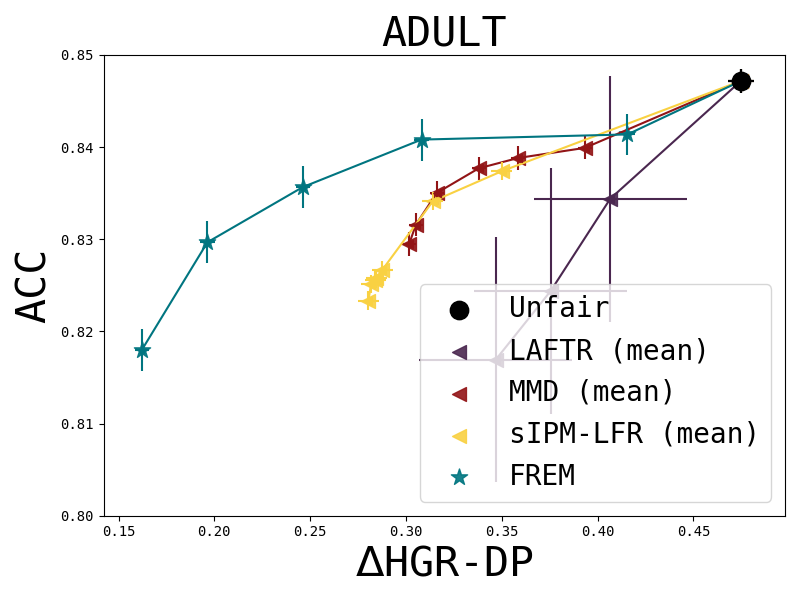}
    }
    \\~\\
    \fbox{
    \includegraphics[width=0.3\textwidth]{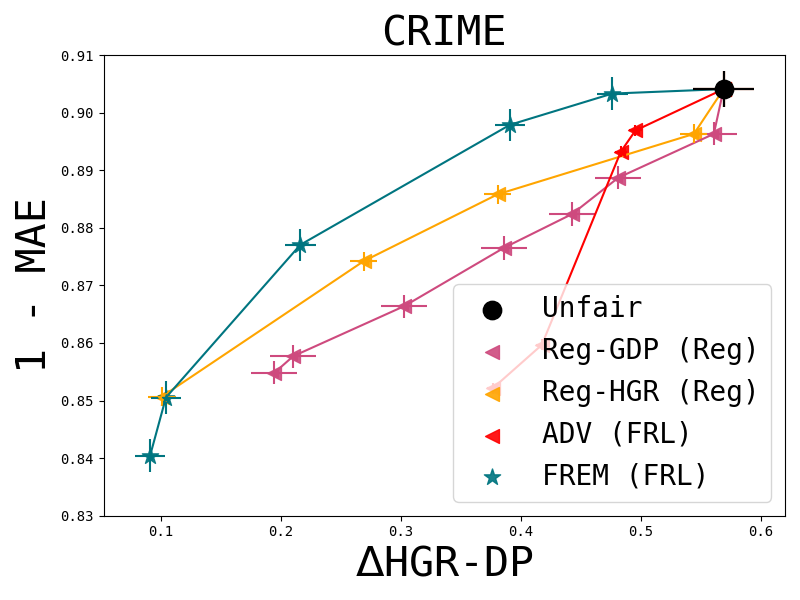}
    }
    \fbox{
    \includegraphics[width=0.3\textwidth]{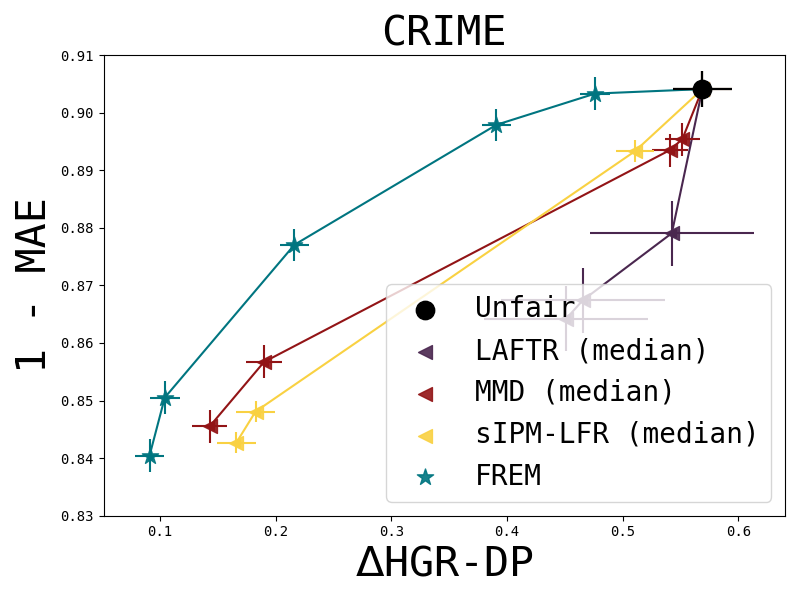}
    \includegraphics[width=0.3\textwidth]{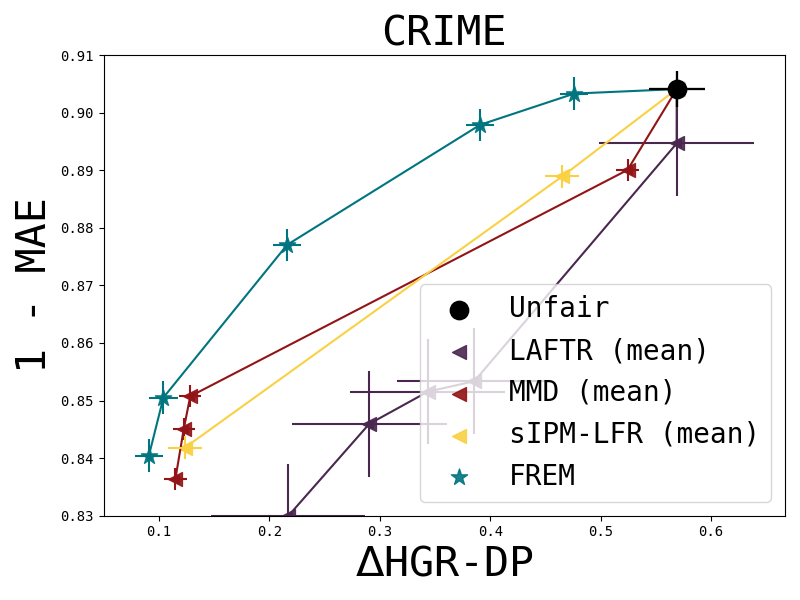}
    }
    \caption{
    \textbf{HGR measure}: Pareto-front lines for fairness-prediction trade-off.
    (Top)
    \textsc{Adult} dataset, $\Delta \texttt{HGR-DP}$ vs. \texttt{ACC}.
    (Bottom)
    \textsc{Crime} dataset, $\Delta \texttt{HGR-DP}$ vs. 1 - \texttt{MAE}.
    $\bullet$: Unfair,
    \textcolor{Regkernelcolor}{--$\blacktriangleleft$--}: Reg-GDP,
    \textcolor{orange}{--$\blacktriangleleft$--}: Reg-HGR,
    \textcolor{red}{--$\blacktriangleleft$--}: ADV,
    \textcolor{sIPMcolor}{--$\blacktriangleleft$--}: sIPM-LFR,
    \textcolor{MMDcolor}{--$\blacktriangleleft$--}: MMD,
    \textcolor{LAFTRcolor}{--$\blacktriangleleft$--}: LAFTR,
    \textcolor{Ourcolor}{--$\mathbf{\filledstar}$--}: FREM.
    }
    \label{fig:tradeoff-hgr}
\end{figure*}

\begin{figure*}[ht]
    \centering
    \fbox{
    \includegraphics[width=0.3\textwidth]{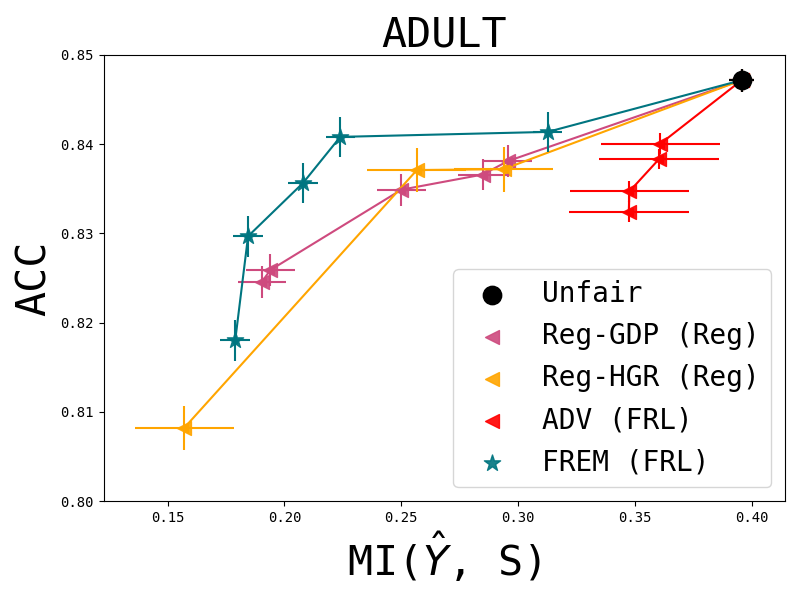}
    }
    \fbox{
    \includegraphics[width=0.3\textwidth]{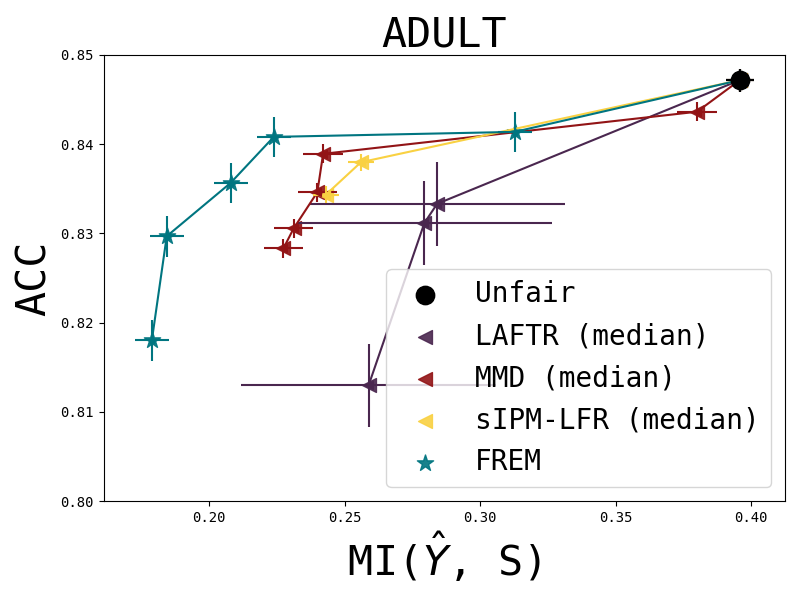}
    \includegraphics[width=0.3\textwidth]{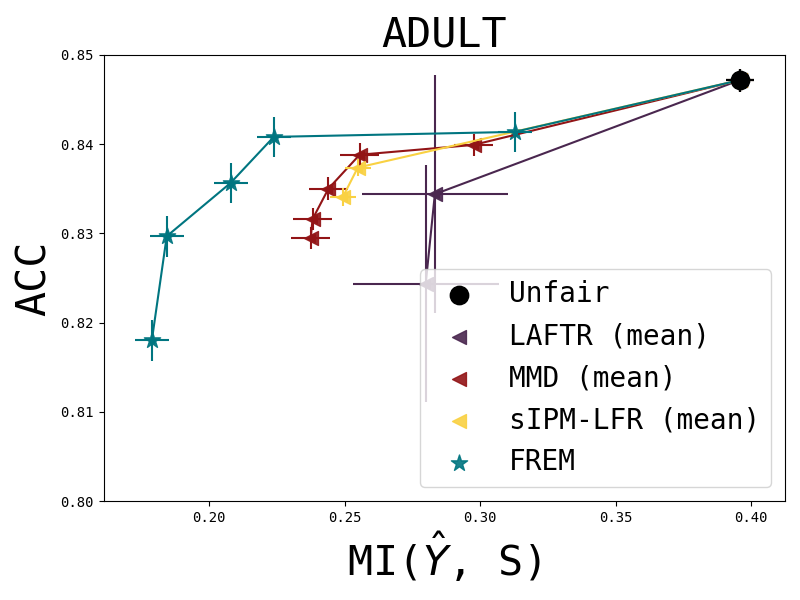}
    }
    \\~\\
    \fbox{
    \includegraphics[width=0.3\textwidth]{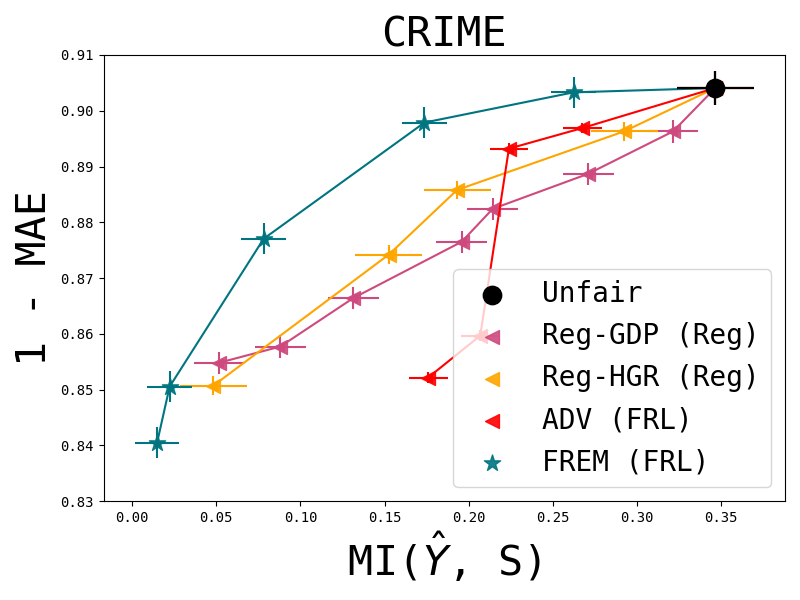}
    }
    \fbox{
    \includegraphics[width=0.3\textwidth]{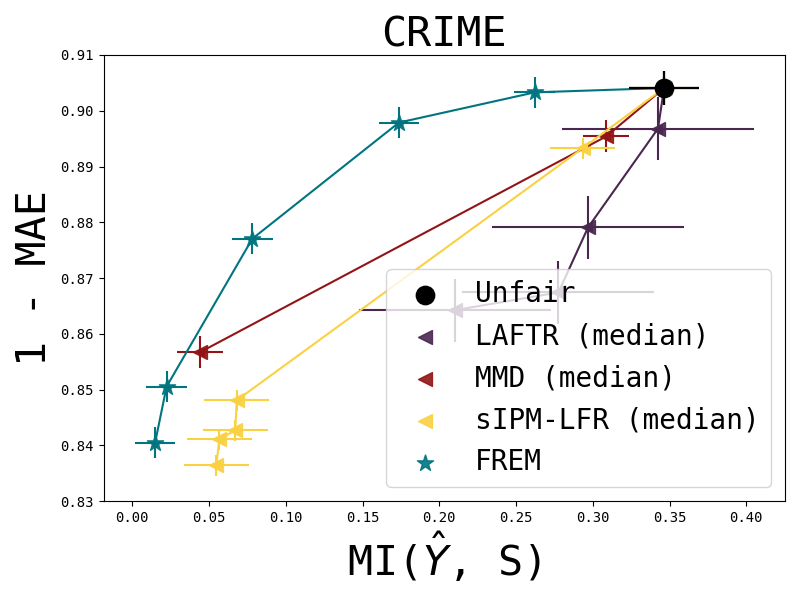}
    \includegraphics[width=0.3\textwidth]{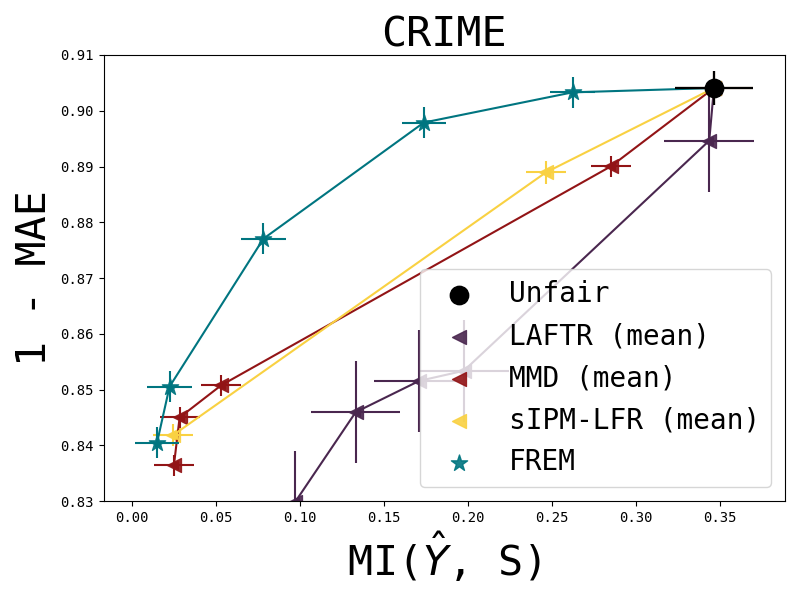}
    }
    \caption{
    \textbf{MI measure}: Pareto-front lines for fairness-prediction trade-off.
    (Top)
    \textsc{Adult} dataset, $\texttt{MI}(\hat{Y}, S)$ vs. \texttt{ACC}.
    (Bottom)
    \textsc{Crime} dataset, $\texttt{MI}(\hat{Y}, S)$ vs. 1 - \texttt{MAE}.
    $\bullet$: Unfair,
    \textcolor{Regkernelcolor}{--$\blacktriangleleft$--}: Reg-GDP,
    \textcolor{orange}{--$\blacktriangleleft$--}: Reg-HGR,
    \textcolor{red}{--$\blacktriangleleft$--}: ADV,
    \textcolor{sIPMcolor}{--$\blacktriangleleft$--}: sIPM-LFR,
    \textcolor{MMDcolor}{--$\blacktriangleleft$--}: MMD,
    \textcolor{LAFTRcolor}{--$\blacktriangleleft$--}: LAFTR,
    \textcolor{Ourcolor}{--$\mathbf{\filledstar}$--}: FREM.
    }
    \label{fig:tradeoff-mi}
\end{figure*}

\begin{figure*}[ht]
    \centering
    \fbox{
    \includegraphics[width=0.3\textwidth]{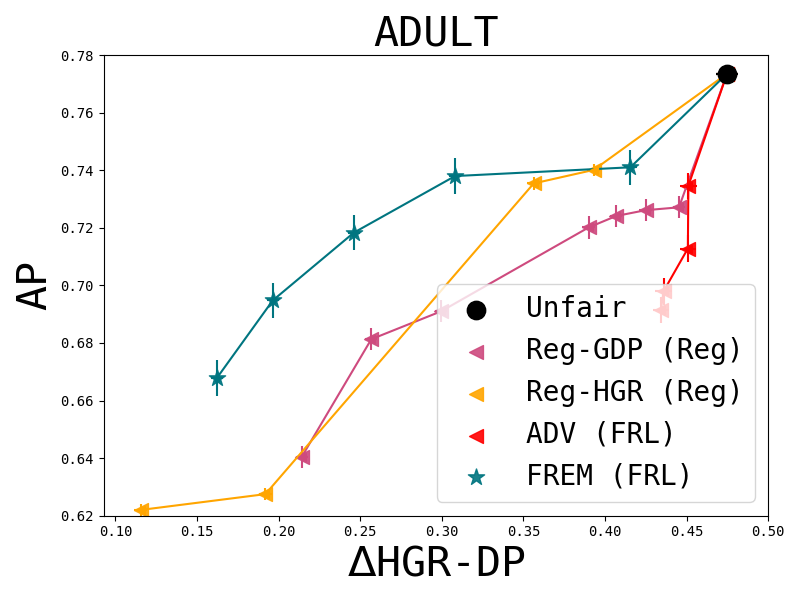}
    }
    \fbox{
    \includegraphics[width=0.3\textwidth]{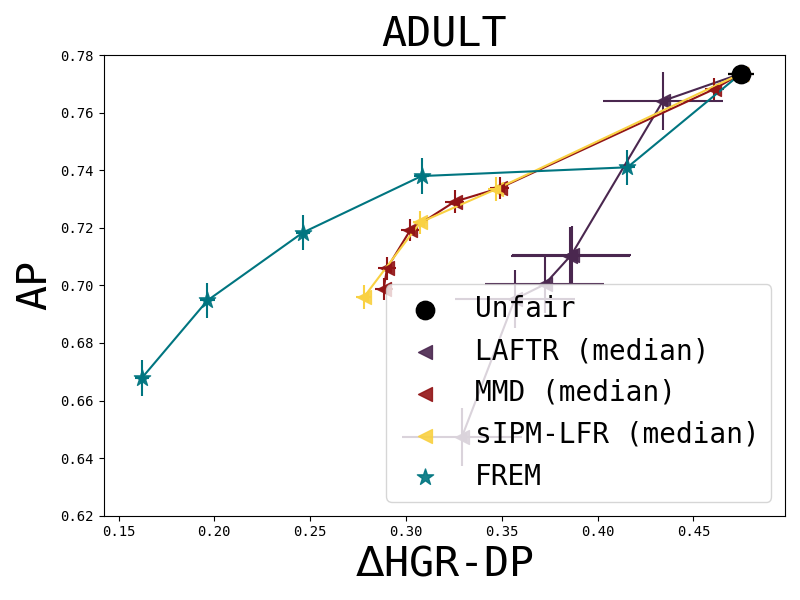}
    \includegraphics[width=0.3\textwidth]{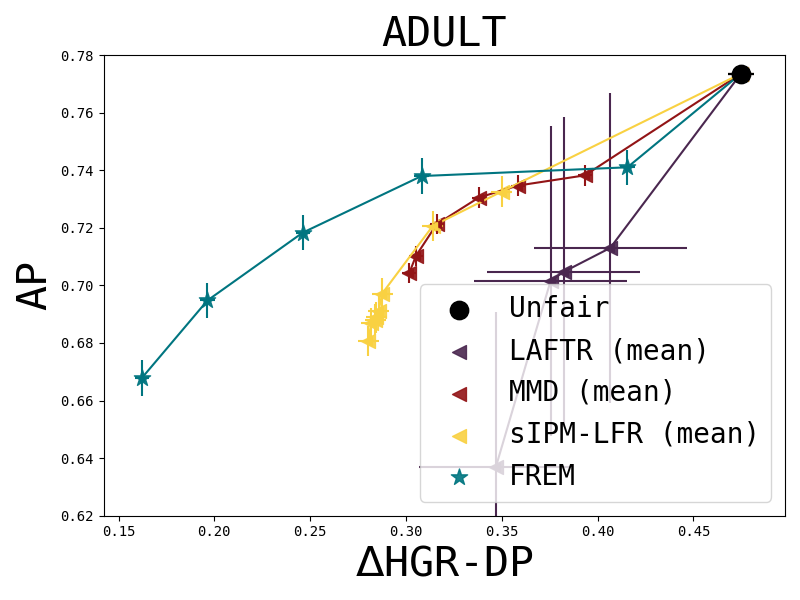}
    }
    \\~\\
    \fbox{
    \includegraphics[width=0.3\textwidth]{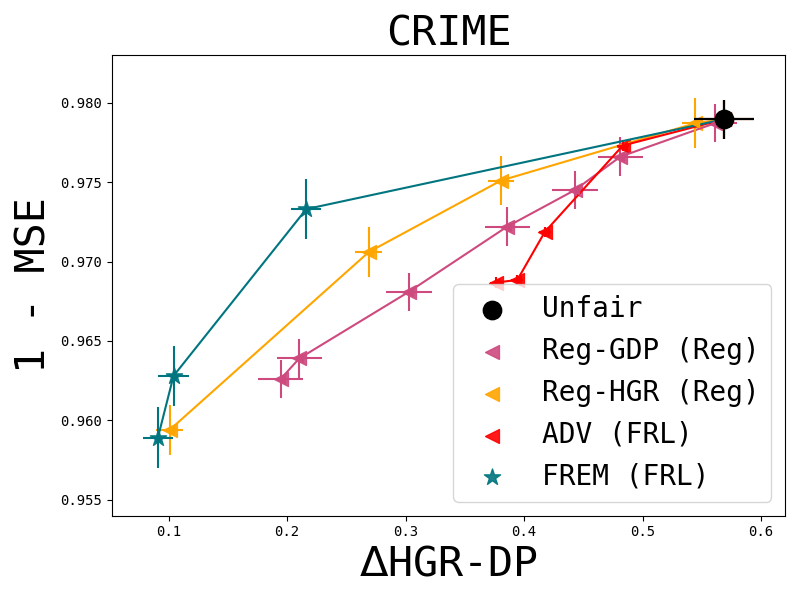}
    }
    \fbox{
    \includegraphics[width=0.3\textwidth]{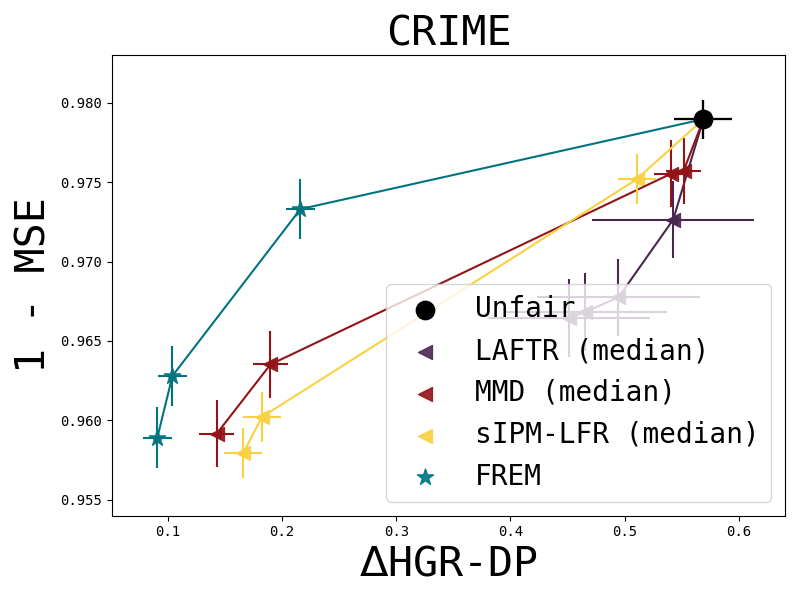}
    \includegraphics[width=0.3\textwidth]{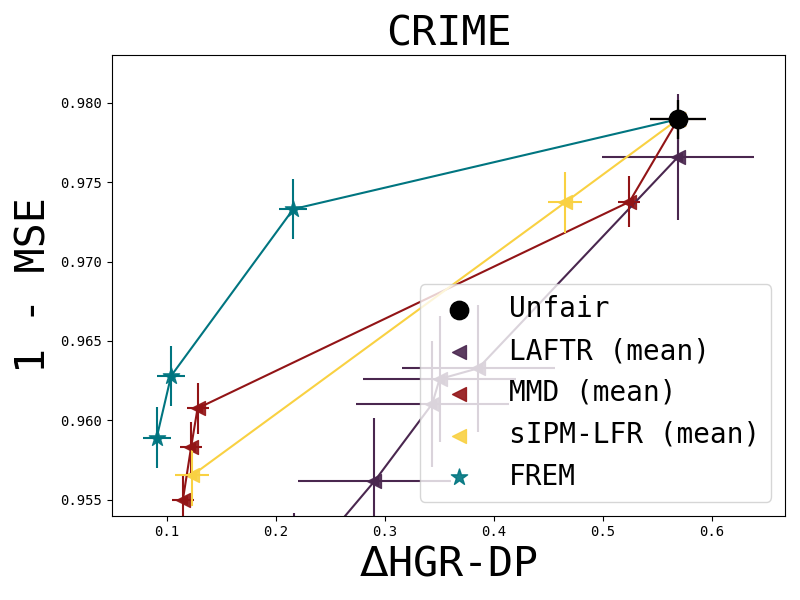}
    }
    \caption{
    \textbf{HGR measure}: Pareto-front lines for fairness-prediction trade-off.
    (Top)
    \textsc{Adult} dataset,
    $\Delta \texttt{HGR-DP}$ vs. \texttt{AP}.
    (Bottom)
    \textsc{Crime} dataset,
    $\Delta \texttt{HGR-DP}$ vs. 1 - \texttt{MSE}.
    $\bullet$: Unfair,
    \textcolor{Regkernelcolor}{--$\blacktriangleleft$--}: Reg-GDP,
    \textcolor{orange}{--$\blacktriangleleft$--}: Reg-HGR,
    \textcolor{red}{--$\blacktriangleleft$--}: ADV,
    \textcolor{sIPMcolor}{--$\blacktriangleleft$--}: sIPM-LFR,
    \textcolor{MMDcolor}{--$\blacktriangleleft$--}: MMD,
    \textcolor{LAFTRcolor}{--$\blacktriangleleft$--}: LAFTR,
    \textcolor{Ourcolor}{--$\mathbf{\filledstar}$--}: FREM.
    }
    \label{fig:tradeoff-hgr2}
\end{figure*}

\begin{figure*}[ht]
    \centering
    \fbox{
    \includegraphics[width=0.3\textwidth]{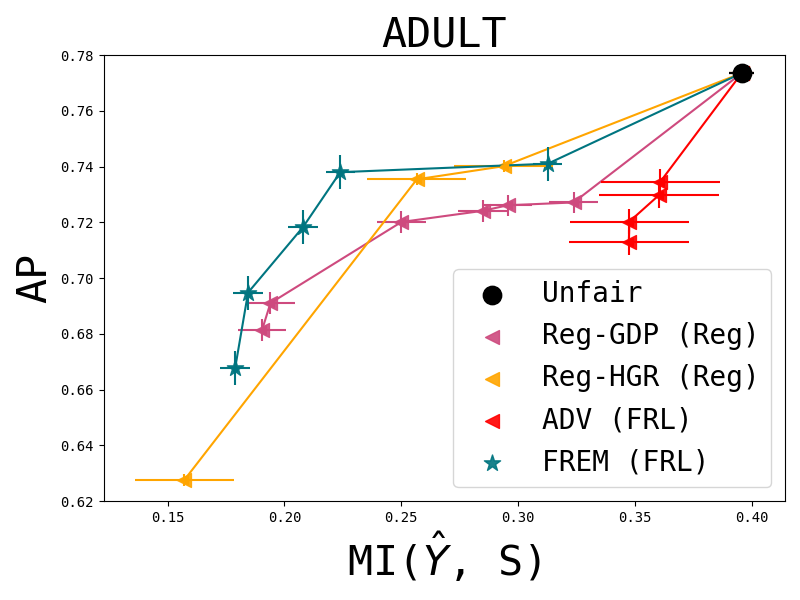}
    }
    \fbox{
    \includegraphics[width=0.3\textwidth]{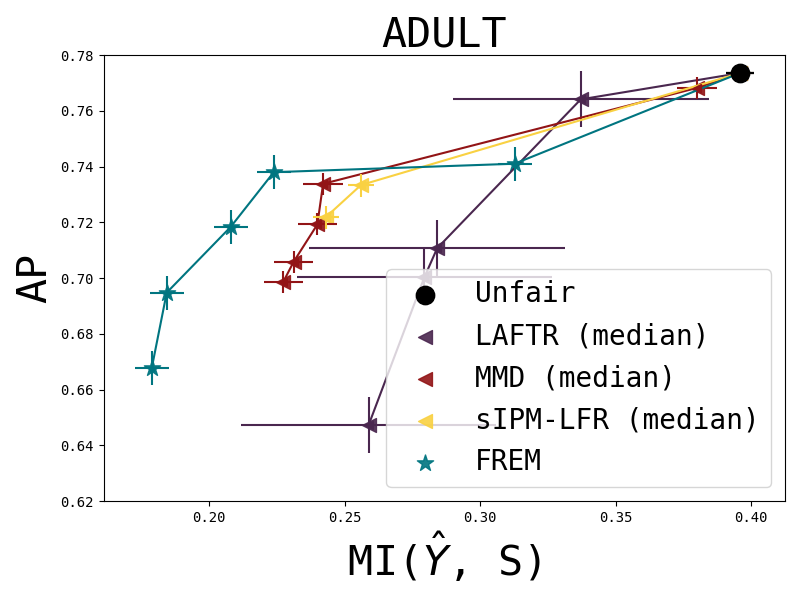}
    \includegraphics[width=0.3\textwidth]{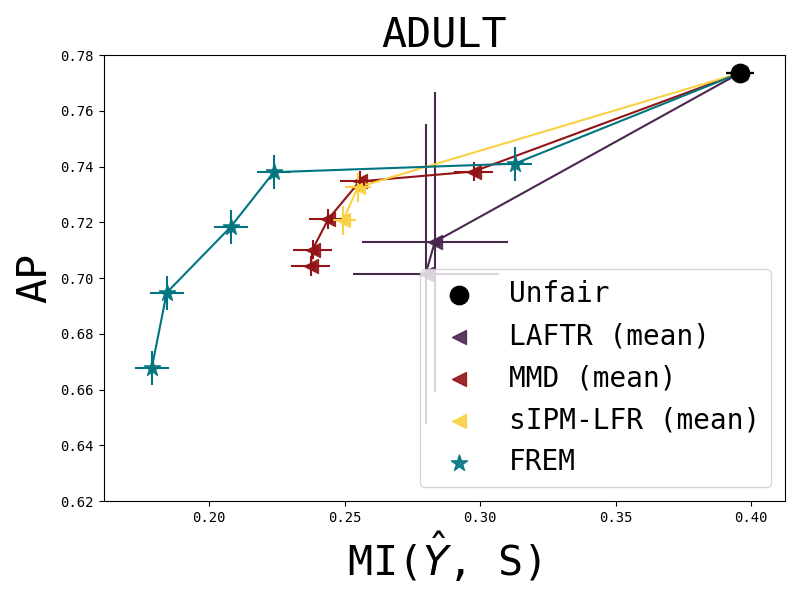}
    }
    \\~\\
    \fbox{
    \includegraphics[width=0.3\textwidth]{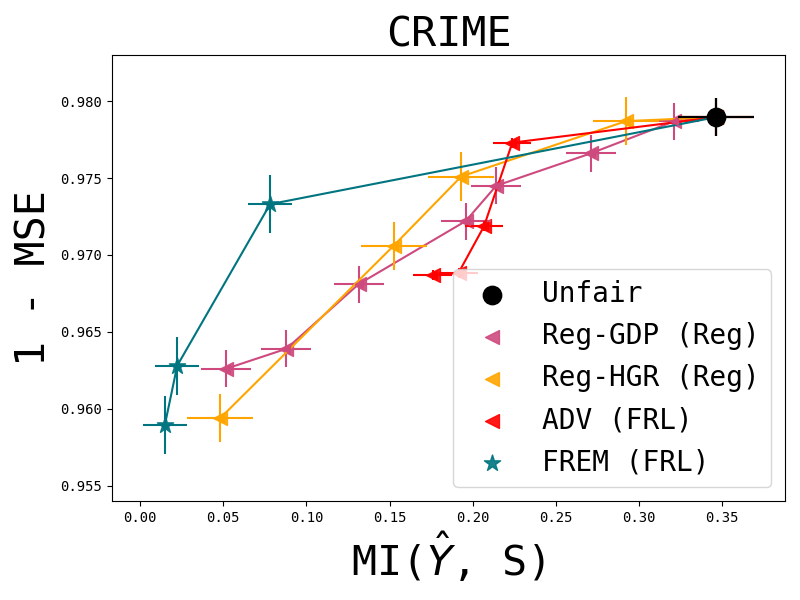}
    }
    \fbox{
    \includegraphics[width=0.3\textwidth]{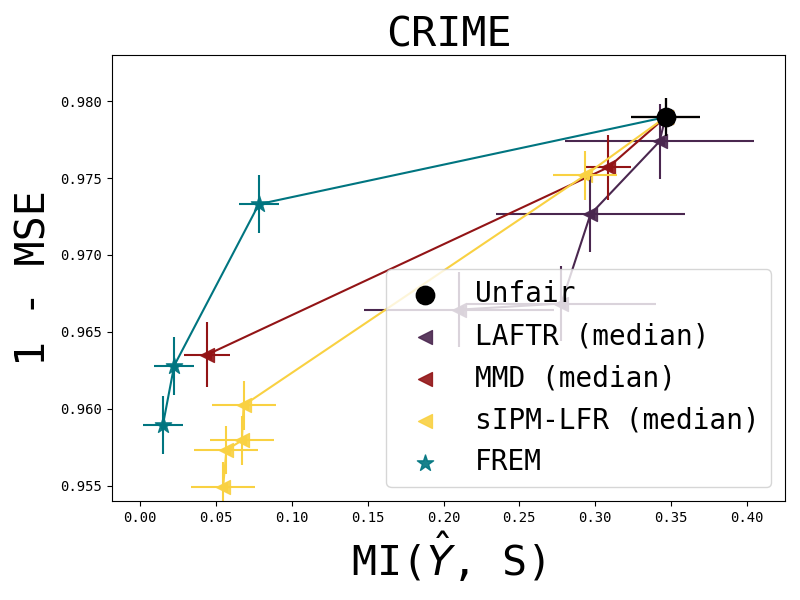}
    \includegraphics[width=0.3\textwidth]{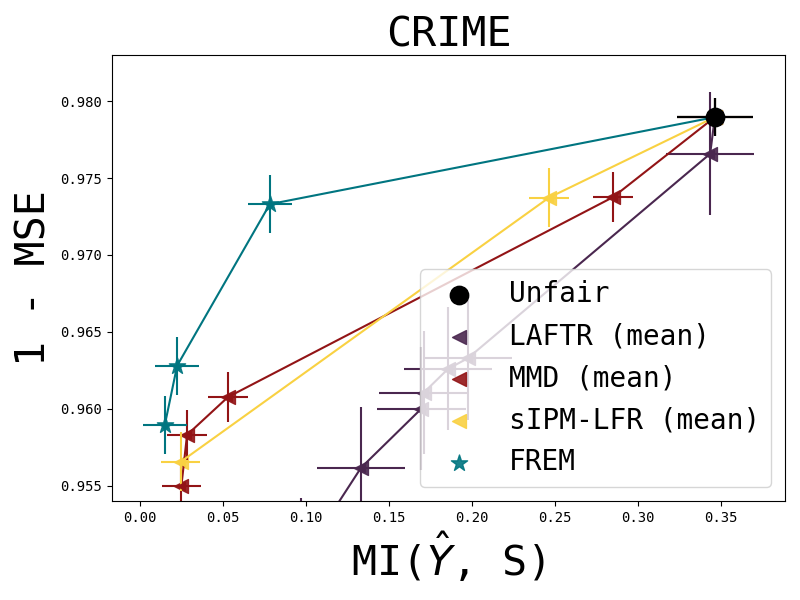}
    }
    \caption{
    \textbf{MI measure}: Pareto-front lines for fairness-prediction trade-off.
    (Top)
    \textsc{Adult} dataset,
    $\Delta \texttt{MI}(\hat{Y}, S)$ vs. \texttt{AP}.
    (Bottom)
    \textsc{Crime} dataset,
    $\Delta \texttt{MI}(\hat{Y}, S)$ vs. 1 - \texttt{MSE}.
    $\bullet$: Unfair,
    \textcolor{Regkernelcolor}{--$\blacktriangleleft$--}: Reg-GDP,
    \textcolor{orange}{--$\blacktriangleleft$--}: Reg-HGR,
    \textcolor{red}{--$\blacktriangleleft$--}: ADV,
    \textcolor{sIPMcolor}{--$\blacktriangleleft$--}: sIPM-LFR,
    \textcolor{MMDcolor}{--$\blacktriangleleft$--}: MMD,
    \textcolor{LAFTRcolor}{--$\blacktriangleleft$--}: LAFTR,
    \textcolor{Ourcolor}{--$\mathbf{\filledstar}$--}: FREM.
    }
    \label{fig:tradeoff-mi2}
\end{figure*}

\clearpage

\begin{figure}[ht]
    \centering
    \includegraphics[width=0.4\textwidth]{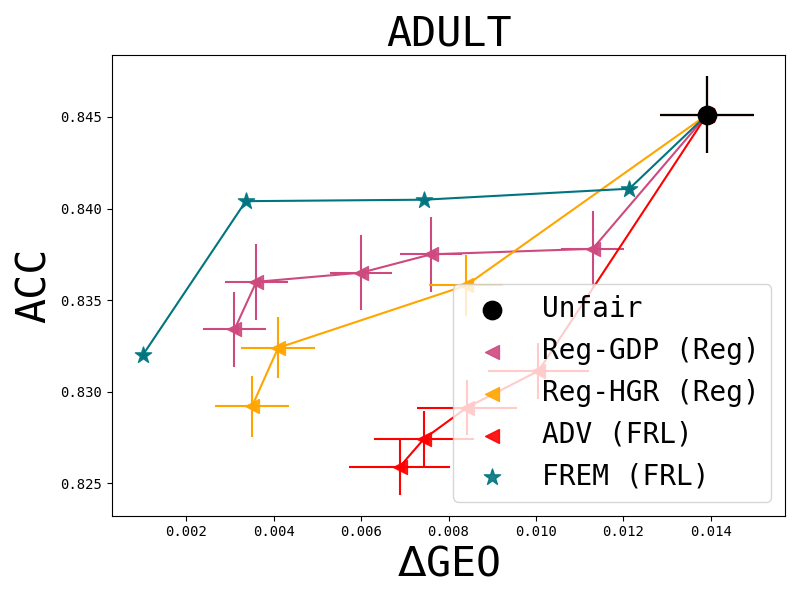}
    \includegraphics[width=0.4\textwidth]{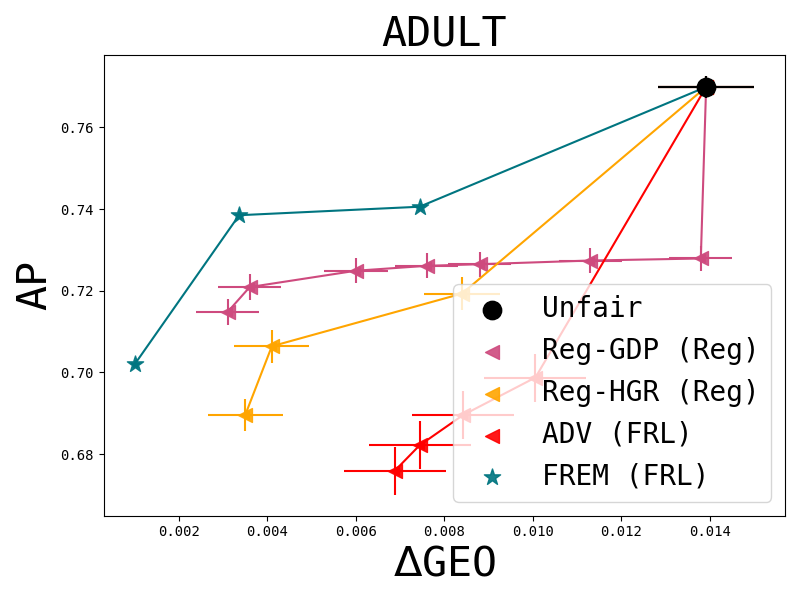}
    \caption{\textbf{Equal Opportunity}: Pareto-front lines for fairness-prediction trade-off on
    \textsc{Adult} dataset:
    (Left) $\Delta \texttt{GEO}$ vs. \texttt{ACC}.
    (Right) $\Delta \texttt{GEO}$ vs. \texttt{AP}.
    $\bullet$: Unfair,
    \textcolor{Regkernelcolor}{--$\blacktriangleleft$--}: {Reg-GDP},
    \textcolor{orange}{--$\blacktriangleleft$--}: {Reg-HGR},
    \textcolor{red}{--$\blacktriangleleft$--}: ADV,
    \textcolor{Ourcolor}{--$\mathbf{\filledstar}$--}: FREM.
    }
    \label{fig:tradeoff-eqopp}
\end{figure}

As described in Appendix \ref{App_EO}, FREM algorithm (for DP) can be modified easily for EO,
which measures the degree of dependency between $\hat{Y}|Y=1$ and $S|Y=1.$
Note that equal opportunity is only defined on a classification task because it requires fairness only for those with $Y=1.$
We present the Pareto-front lines of $\Delta \texttt{GEO}$ vs. \texttt{Acc} and 
$\Delta \texttt{GEO}$ vs. \texttt{AP} in Fig. \ref{fig:tradeoff-eqopp}.
It is clear that FREM outperforms all baseline methods consistently.

\clearpage


\subsubsection{Graph data analysis}\label{sec:appen-graph_data}

We also evaluate the performances of FREM and compare with those of baselines (i.e., Reg-GDP, Reg-HGR, ADV, and MMD with binarized sensitive attributes) on graph datasets.
Fig. \ref{fig:graph_trade_off}, \ref{fig:graph_trade_off2} and \ref{fig:graph_trade_off3} draw the Pareto-front lines for the two graph datasets with various fairness measures.
It depends on the fairness measure being considered, but overall, these results imply that FREM is comparable to the regularization methods (i.e., Reg-GDP and Reg-HGR).
Notably, FREM outperforms the existing FRL methods (i.e., ADV and MMD with binarized $S$), demonstrating that FREM is the most favorable among the FRL methods.

\begin{figure}[ht]
    \centering
    \includegraphics[scale=0.20]{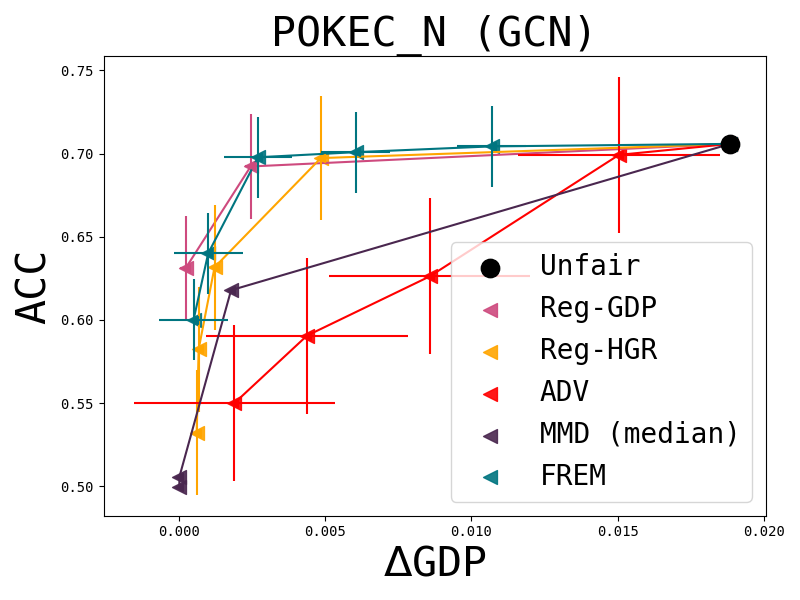}
    \includegraphics[scale=0.20]{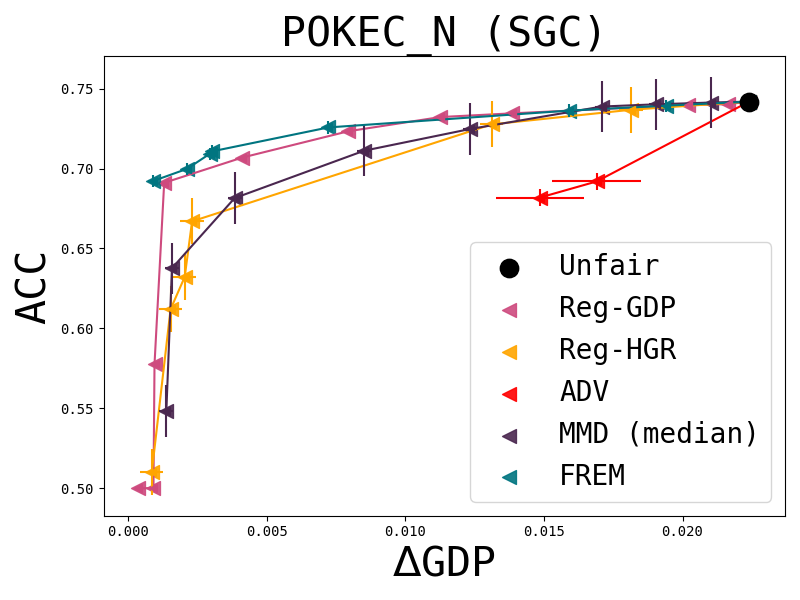}
    \includegraphics[scale=0.20]{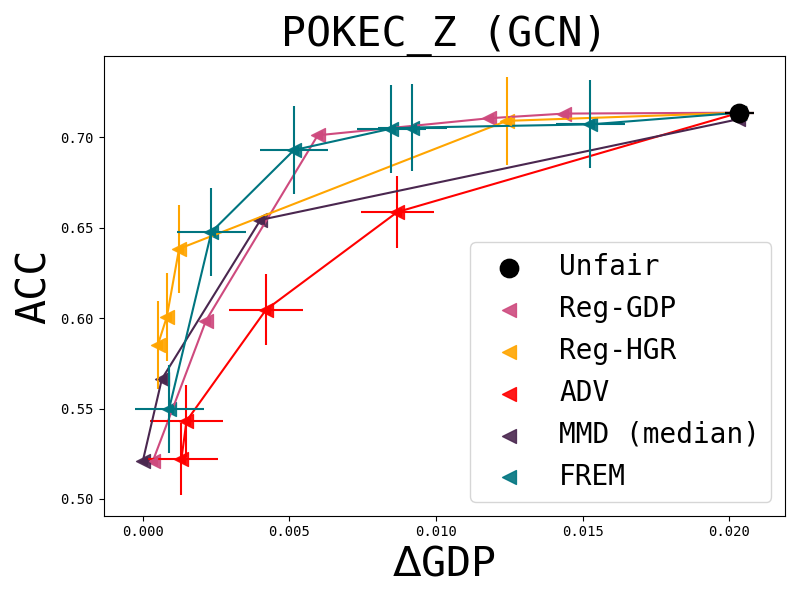}    
    \includegraphics[scale=0.20]{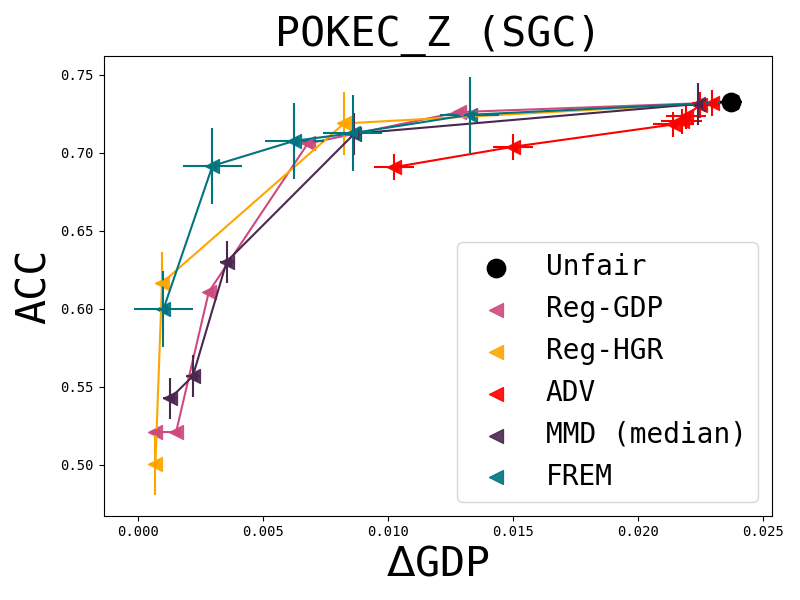}
    \caption{ {\textbf{$\Delta \texttt{GDP}$ vs. \texttt{Acc} trade-offs on graph datasets:}
    (Left two) \textsc{Pokec-n}, (Right two) \textsc{Pokec-z}.
    $\bullet$: Unfair,
    \textcolor{Regkernelcolor}{--$\blacktriangleleft$--}: {Reg-GDP},
    \textcolor{orange}{--$\blacktriangleleft$--}: {Reg-HGR},
    \textcolor{red}{--$\blacktriangleleft$--}: ADV,
    \textcolor{MMDcolor}{--$\blacktriangleleft$--}: MMD,
    \textcolor{Ourcolor}{--$\mathbf{\filledstar}$--}: FREM.
    }}
    \label{fig:graph_trade_off}
\end{figure}

\begin{figure}[ht]
    \centering
    \includegraphics[scale=0.2]{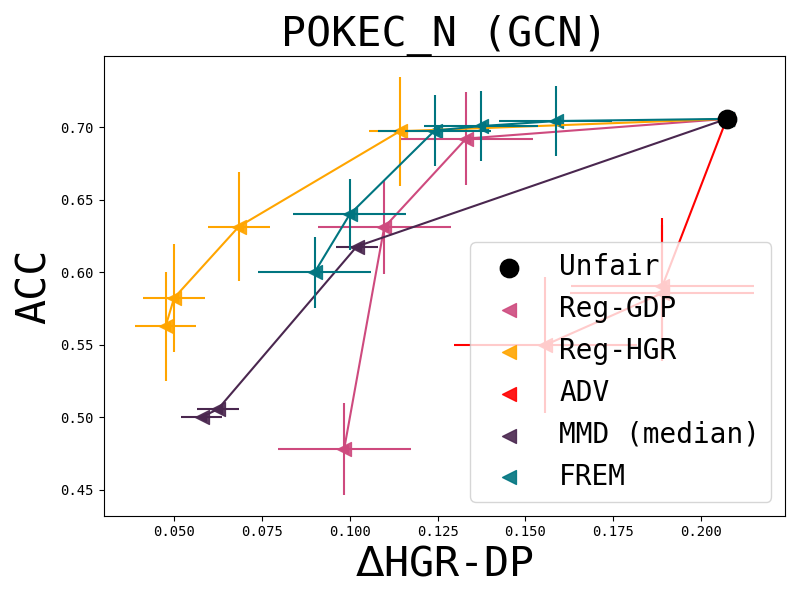}    
    \includegraphics[scale=0.2]{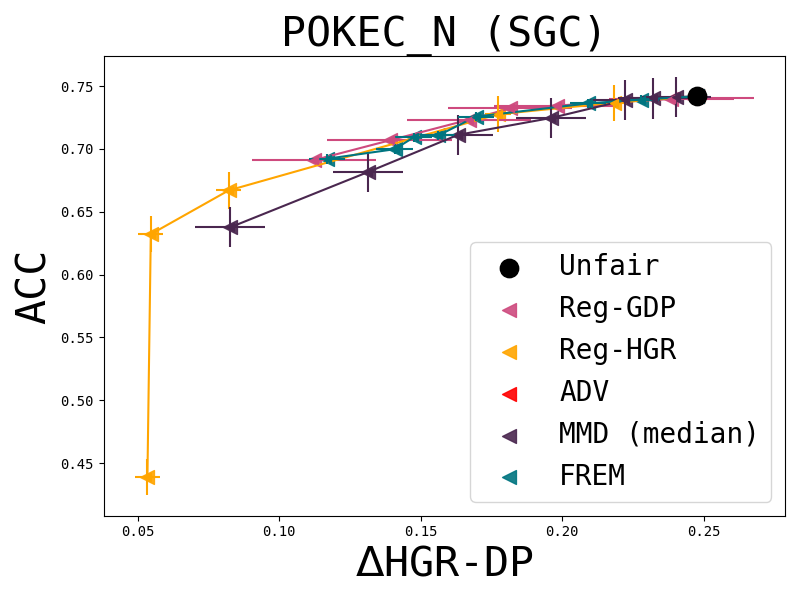}
    \includegraphics[scale=0.2]{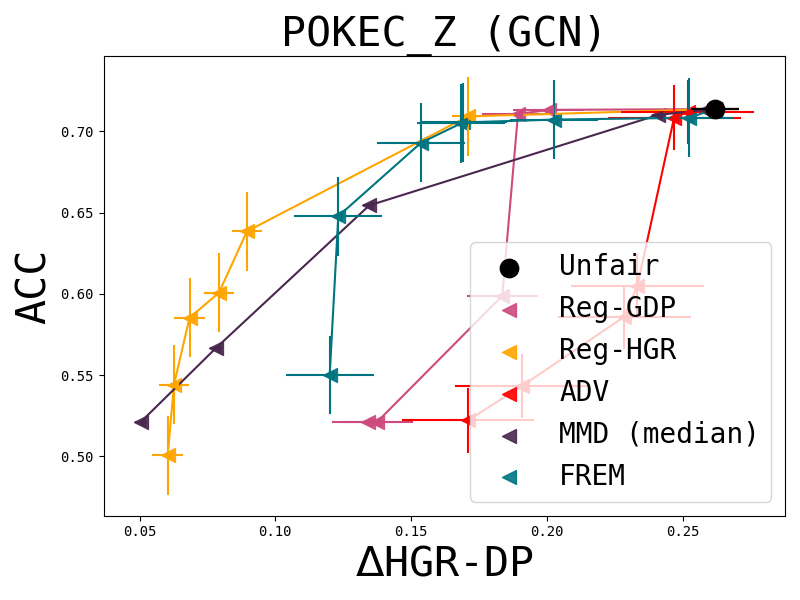}
    \includegraphics[scale=0.2]{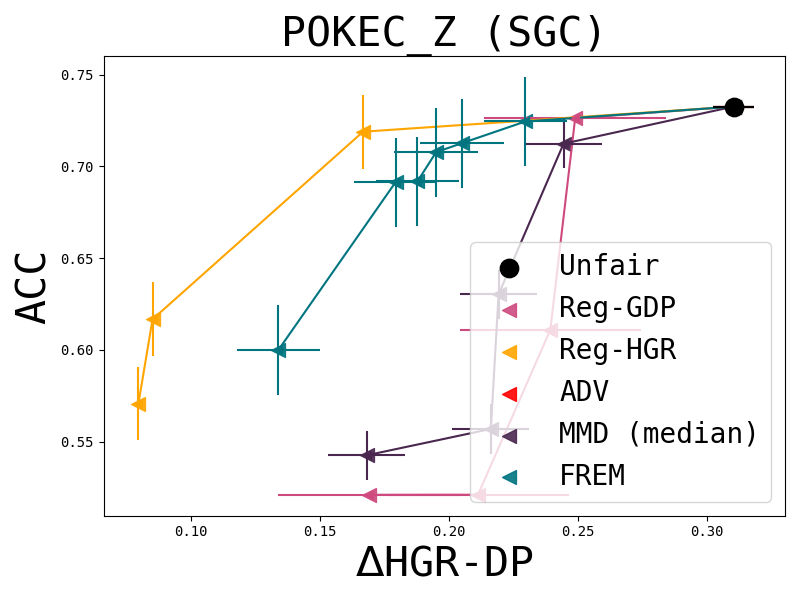}
    \caption{{\textbf{$\Delta \texttt{HGR-DP}$ vs. \texttt{Acc} trade-offs on graph datasets:}
    (Left two) \textsc{Pokec-n}, (Right two) \textsc{Pokec-z}.
    $\bullet$: Unfair,
    \textcolor{Regkernelcolor}{--$\blacktriangleleft$--}: {Reg-GDP},
    \textcolor{orange}{--$\blacktriangleleft$--}: {Reg-HGR},
    \textcolor{red}{--$\blacktriangleleft$--}: ADV,
    \textcolor{MMDcolor}{--$\blacktriangleleft$--}: MMD,
    \textcolor{Ourcolor}{--$\blacktriangleleft$--}: FREM.
    }}
    \label{fig:graph_trade_off2}
\end{figure}

\begin{figure}[ht]
    \centering
    \includegraphics[scale=0.2]{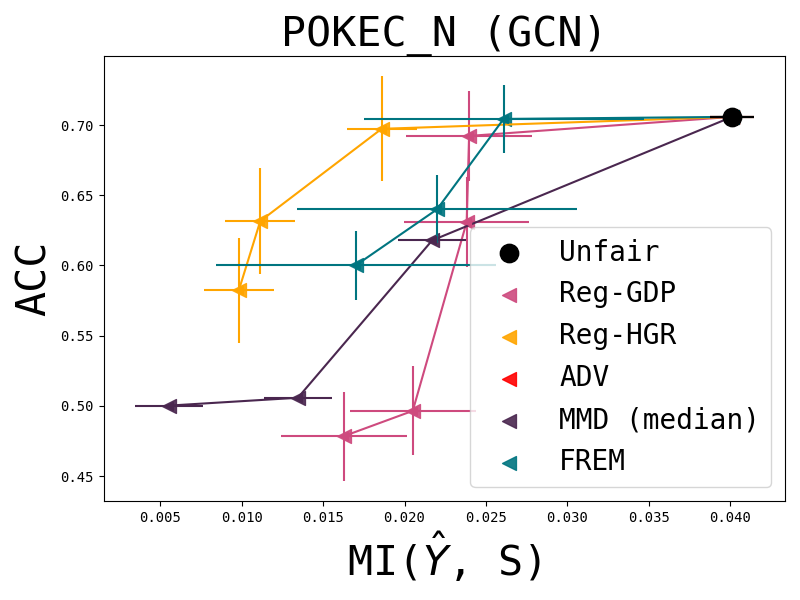}
    \includegraphics[scale=0.2]{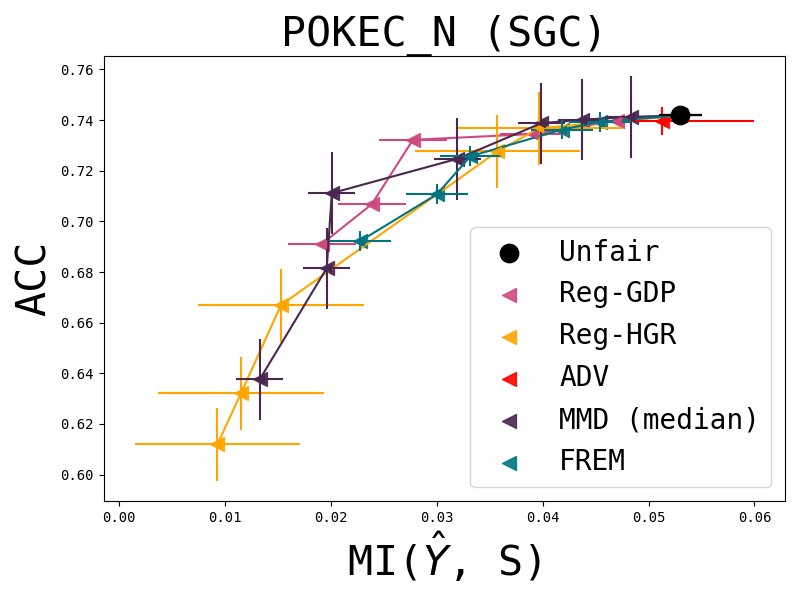}
    \includegraphics[scale=0.2]{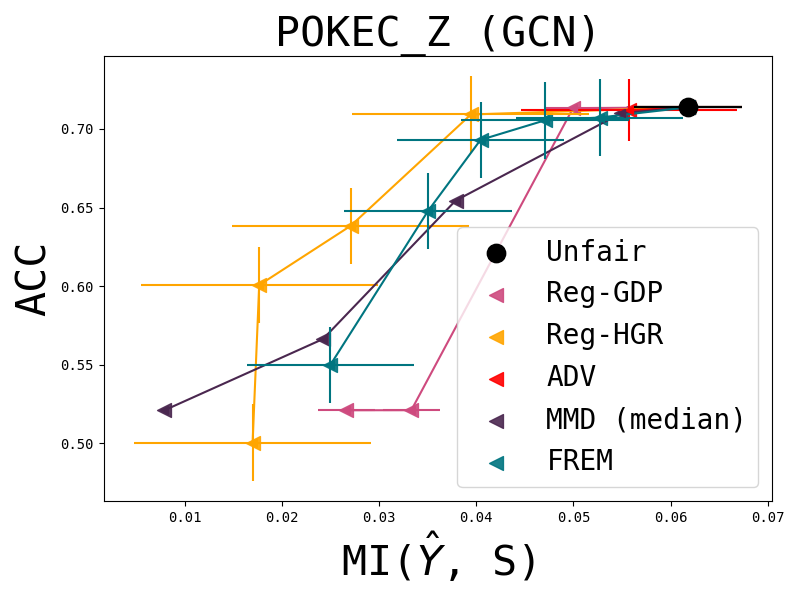}
    \includegraphics[scale=0.2]{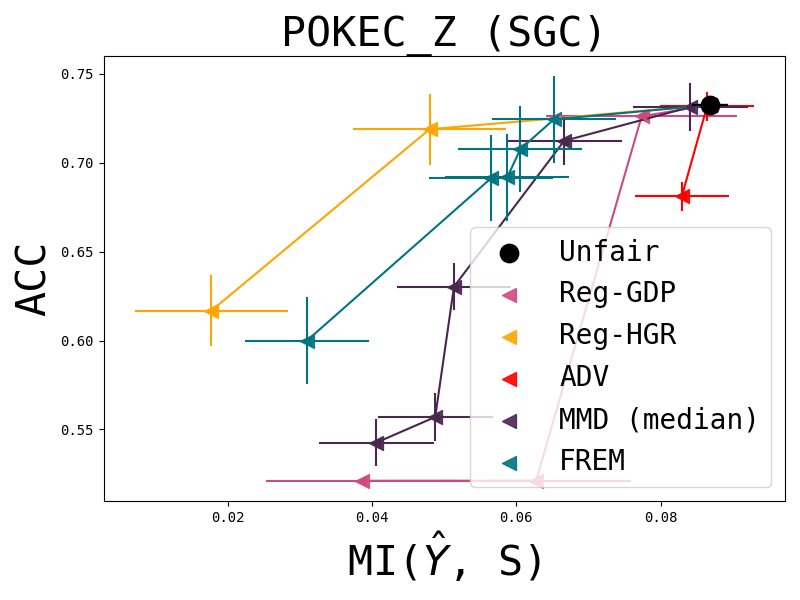}
    \caption{{\textbf{$\texttt{MI}(\hat{Y}, S)$ vs. \texttt{Acc} trade-offs on graph datasets:}
    (Left two) \textsc{Pokec-n}, (Right two) \textsc{Pokec-z}.
    $\bullet$: Unfair,
    \textcolor{Regkernelcolor}{--$\blacktriangleleft$--}: {Reg-GDP},
    \textcolor{orange}{--$\blacktriangleleft$--}: {Reg-HGR},
    \textcolor{red}{--$\blacktriangleleft$--}: ADV,
    \textcolor{MMDcolor}{--$\blacktriangleleft$--}: MMD,
    \textcolor{Ourcolor}{--$\blacktriangleleft$--}: FREM.
    }}
    \label{fig:graph_trade_off3}
\end{figure}

Similar to Fig. \ref{fig:compare_mi_z_s}, we also investigate how fair the learned representation $\bm{Z}$ is, by evaluting the Mutual Information (MI) \cite{MacKay2003} as a measure of fairness of the learned representation $\bm{Z}.$
We compare FREM with the two FRL baselines (i.e., MMD and ADV) in terms of the trade-off between \texttt{Acc} and fairness of the representation (i.e., $\texttt{MI}(\bm{Z}, S)$), whose results are given in Fig. \ref{fig:graph_trade_off4}.
The results clearly show that FREM is generally good at controlling fairness of representations on graph datasets.

\begin{figure}[ht]
    \centering
    \includegraphics[scale=0.25]{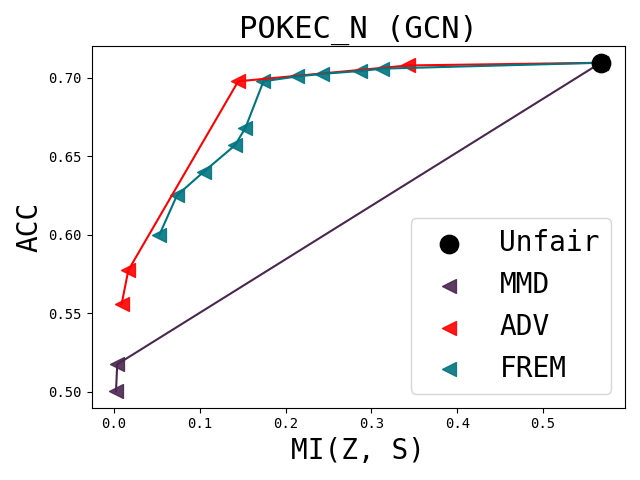}
    \includegraphics[scale=0.25]{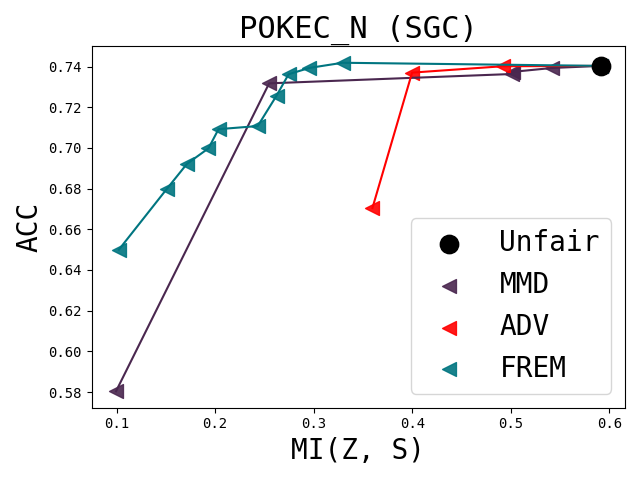}
    \includegraphics[scale=0.25]{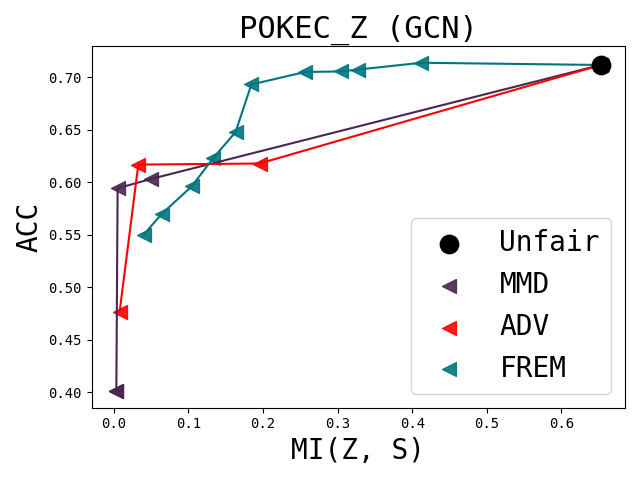}
    \includegraphics[scale=0.25]{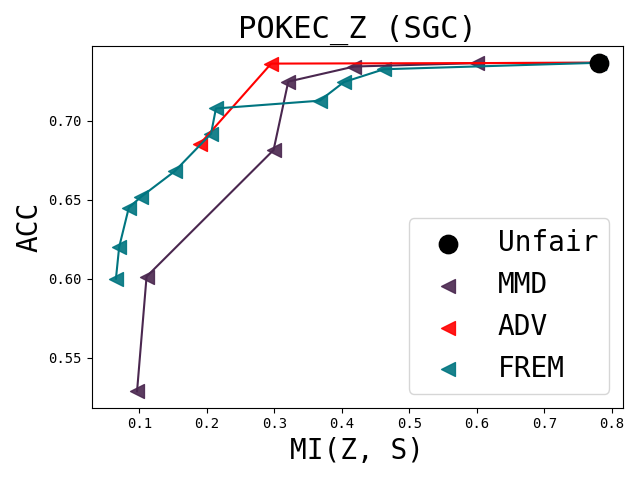}
    \caption{{\textbf{$\texttt{MI}(\bm{Z}, S)$ vs. \texttt{Acc} trade-offs of FREM on graph datasets:}
    (Left two) \textsc{Pokec-n}, (Right two) \textsc{Pokec-z}.
    $\bullet$: Unfair,
    \textcolor{Regkernelcolor}{--$\blacktriangleleft$--}: {Reg-GDP},
    \textcolor{red}{--$\blacktriangleleft$--}: ADV,
    \textcolor{MMDcolor}{--$\blacktriangleleft$--}: MMD,
    \textcolor{Ourcolor}{--$\blacktriangleleft$--}: FREM.
    }}
    \label{fig:graph_trade_off4}
\end{figure}
\color{black} 

\clearpage
\subsubsection{Fairness of the representation: Mutual information between $\bm{Z}$ and $S$}

We also investigate the ability of FREM to eliminate the information of $S$ in $\bm{Z}$ by calculating $\texttt{MI}(\bm{Z}, S),$ whose results are presented in Fig. \ref{fig:mi_z_s} and \ref{fig:mi_z_s2}.  
MI between $\bm{Z}$ and $S$ decreases as the regularization parameter $\lambda$ increases,
which amply demonstrates that
the information of $S$ is being successfully removed from $\bm{Z}$ by FREM.

\begin{figure}[ht]
    \centering
    \includegraphics[scale=0.20]{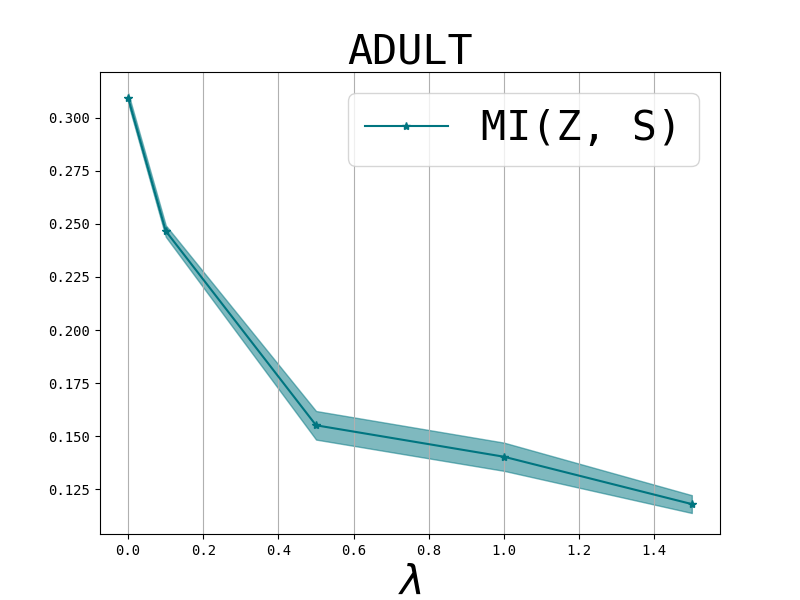}
    \includegraphics[scale=0.20]{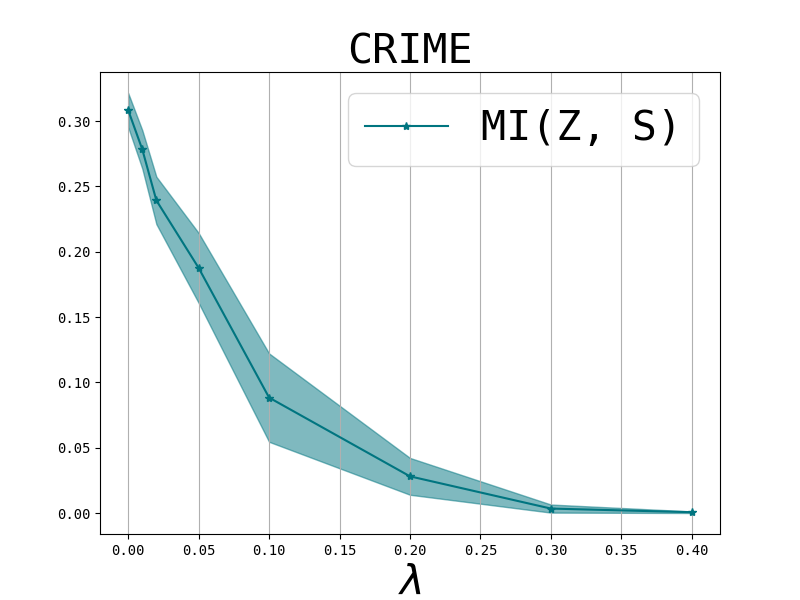}
    \caption{\textbf{Mitigation of bias in representation}: Mutual information between $\bm{Z}$ and $S$ decreases as $\lambda$ increases.
    (Left) \textsc{Adult} (Right) \textsc{Crime}.
    }
    \label{fig:mi_z_s}
\end{figure}

\begin{figure}[ht]
    \centering
    \includegraphics[scale=0.20]{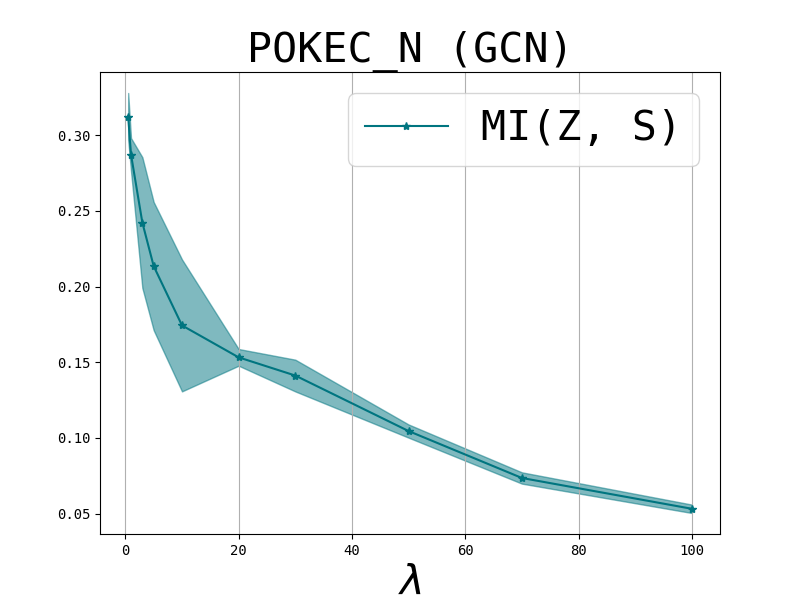}
    \includegraphics[scale=0.20]{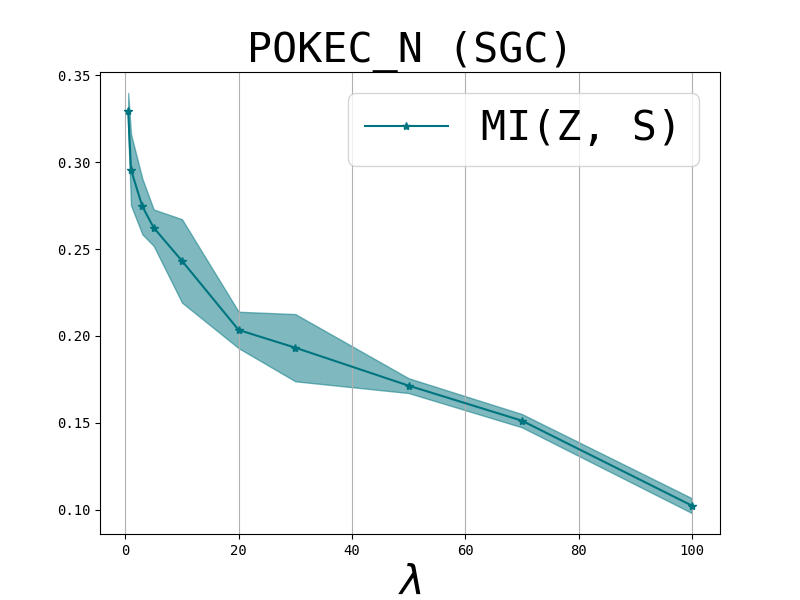}
    \includegraphics[scale=0.20]{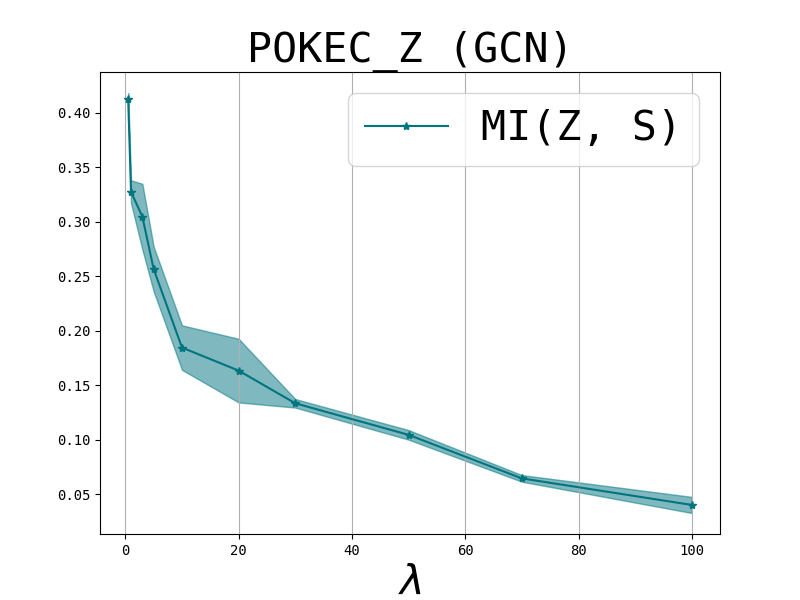}
    \includegraphics[scale=0.20]{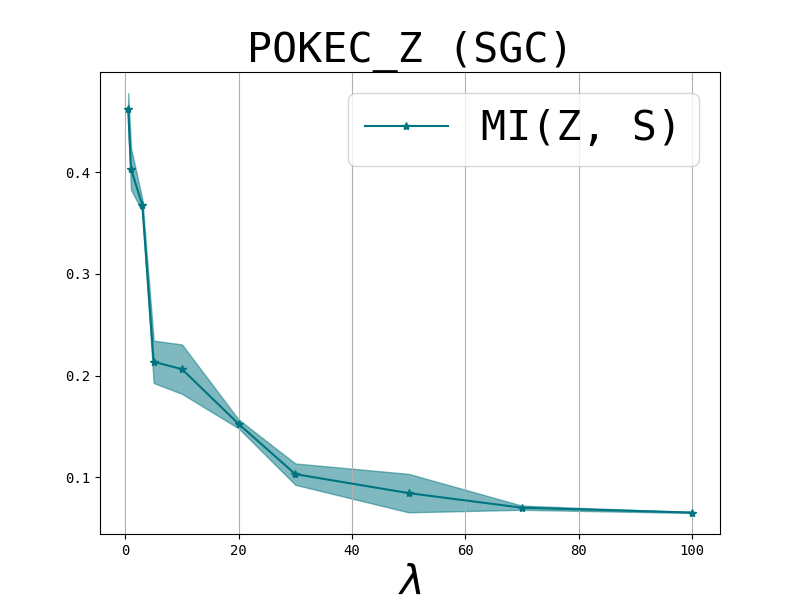}
    \caption{\textbf{Mitigation of bias in representation}: Mutual information between $\bm{Z}$ and $S$ decreases as $\lambda$ increases.
    (Left two) \textsc{Pokec-n} (Right two) \textsc{Pokec-z}.
    }
    \label{fig:mi_z_s2}
\end{figure}

\subsubsection{Additional comparison between FREM and the regularization methods in terms of fairness of representation}\label{frem_vs_reg_mi_z_s}

To empirically validate the generalization ability of FREM to downstream tasks, we compare the fairness of the representations learned by FTM and the two regularization methods (i.e., Reg-GDP and Reg-HGR).
The results presented in Fig. \ref{fig:mi_z_s_reg} (for tabular datasets) and Fig. \ref{fig:mi_z_s_reg2} (for graph datasets) indicate that FREM is the most effective at building fair representations, which can be successfully applied to downstream tasks requiring fairness.
Specifically, FREM achieves better trade-offs between fairness of $\bm{Z}$ and accuracy, when compared to the regularization methods.

\begin{figure}[ht]
    \centering
    \includegraphics[scale=0.27]{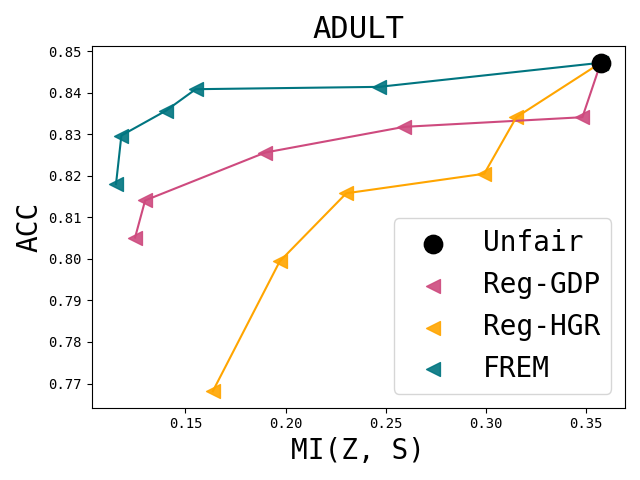}
    \includegraphics[scale=0.27]{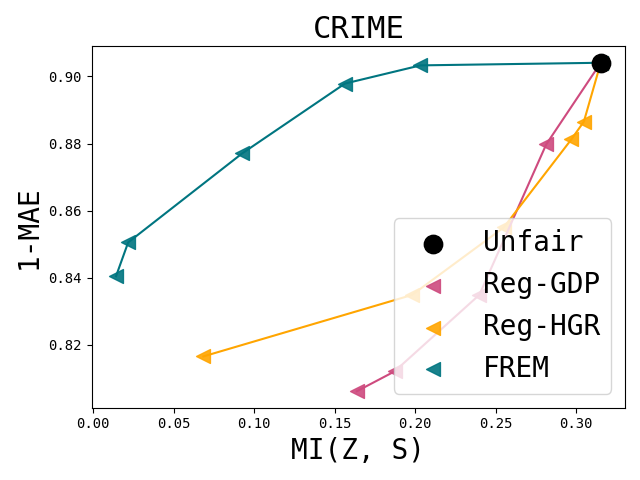}
    \caption{
    \textbf{$\texttt{MI}(\bm{Z}, S)$ vs. \texttt{Acc} trade-offs on tabular datasets:}
    (Left) \textsc{Adult} (Right) \textsc{Crime}.
    $\bullet$: Unfair,
    \textcolor{Regkernelcolor}{--$\blacktriangleleft$--}: {Reg-GDP},
    \textcolor{orange}{--$\blacktriangleleft$--}: {Reg-HGR},
    \textcolor{Ourcolor}{--$\blacktriangleleft$--}: FREM.
    }
    \label{fig:mi_z_s_reg}
\end{figure}

\begin{figure}[ht]
    \centering
    \includegraphics[scale=0.26]{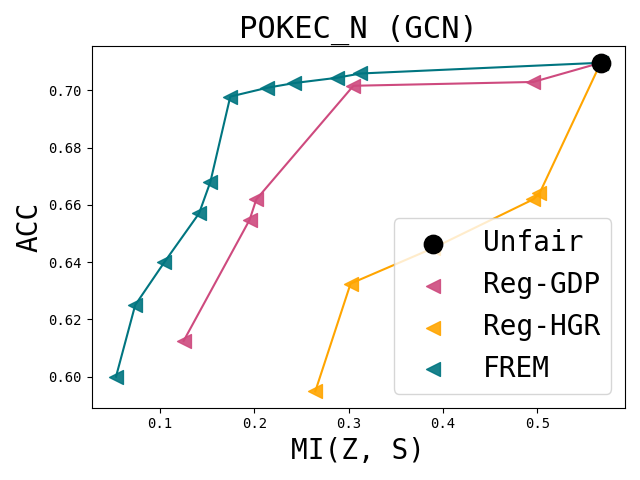}
    \includegraphics[scale=0.26]{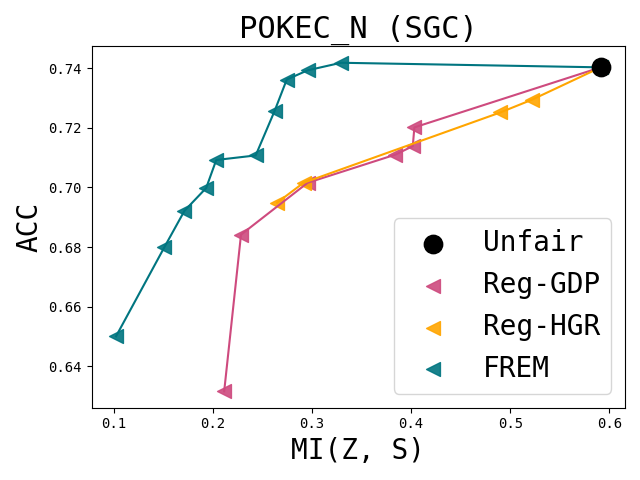}
    \includegraphics[scale=0.26]{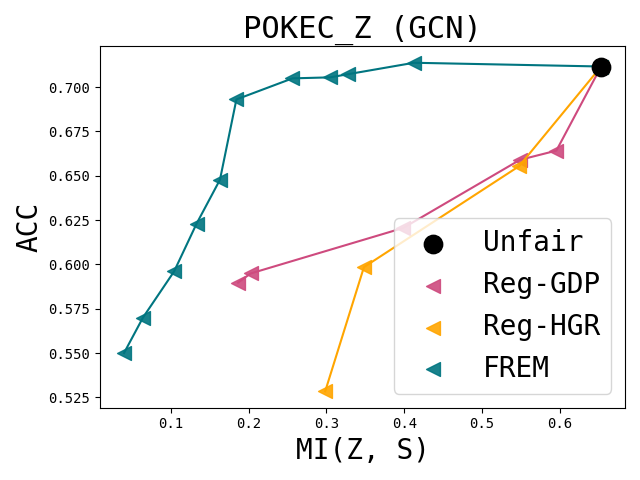}
    \includegraphics[scale=0.26]{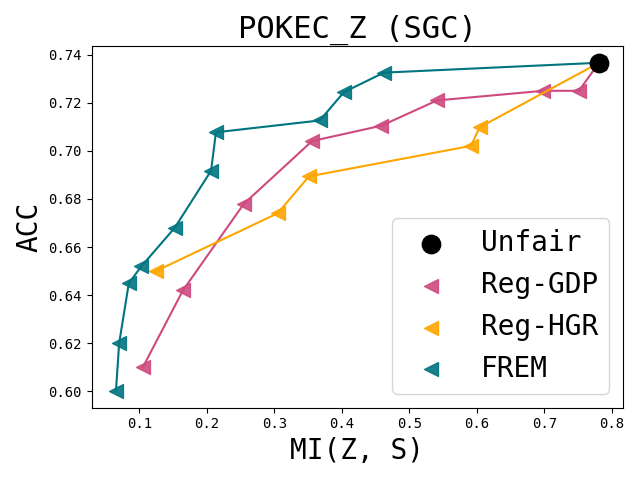}
    \caption{
    \textbf{$\texttt{MI}(\bm{Z}, S)$ vs. \texttt{Acc} trade-offs on graph datasets:}
    (Left two) \textsc{Pokec-n}, (Right two) \textsc{Pokec-z}.
    $\bullet$: Unfair,
    \textcolor{Regkernelcolor}{--$\blacktriangleleft$--}: {Reg-GDP},
    \textcolor{orange}{--$\blacktriangleleft$--}: {Reg-HGR},
    \textcolor{Ourcolor}{--$\blacktriangleleft$--}: FREM.
    }
    \label{fig:mi_z_s_reg2}
\end{figure}

\color{black} 

\clearpage
\subsubsection{Ablation study: choice of kernel functions}
\label{App:abl_ker}

{This section provides empirical results of the ablation studies regarding the choice of kernel functions.}

{
\textbf{(1) Various kernel functions for $K_{\gamma}$:}
As we discuss in Section \ref{subsec_defEIPM_est}, theoretically any kernel function satisfying Assumption \ref{assum_kernel} can be employed for $K_{\gamma}.$
We consider two additional kernels: Triangular and Epanechnikov.
The two kernels are defined by
$K_{\gamma}(s, s') = 1 - \frac{\vert s - s' \vert}{\gamma}$
and
$K_{\gamma}(s, s') = \frac{3}{4} (1 - \frac{(s - s')^{2}}{\gamma}),$
respectively.
For the both kernels, we set the bandwidth as $\gamma = 1.0.$
We compare FREM with (i) RBF, (ii) Triangular, and (iii) Epanechnikov kernels in terms of the fairness-prediction trade-off.
The results are presented in Fig. \ref{fig:vary_kernel}, which show that the three kernels yield similar performances.
Auxillary results with other prediction measures (i.e., \texttt{AP} and \texttt{MSE}) and fairness measures (i.e., $\Delta \texttt{HGR}$ and \texttt{MI}) are given in Fig. \ref{fig:tradeoff-kernel_1} and \ref{fig:tradeoff-kernel_2} below.
}

\begin{figure*}[ht]
    \centering
    \includegraphics[width=0.3\textwidth]{figures/Adult/dp/acc_gdp_w_kernel_kernel.png}
    \includegraphics[width=0.3\textwidth]{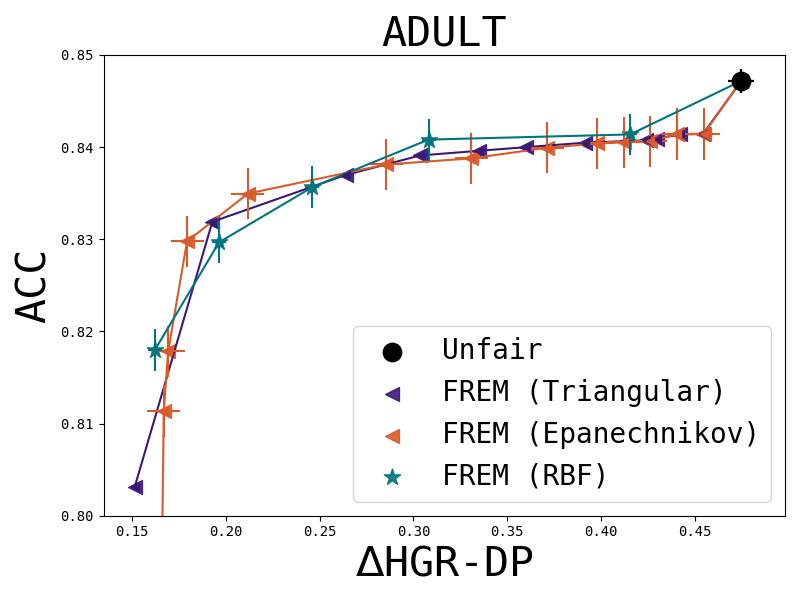}
    \includegraphics[width=0.3\textwidth]{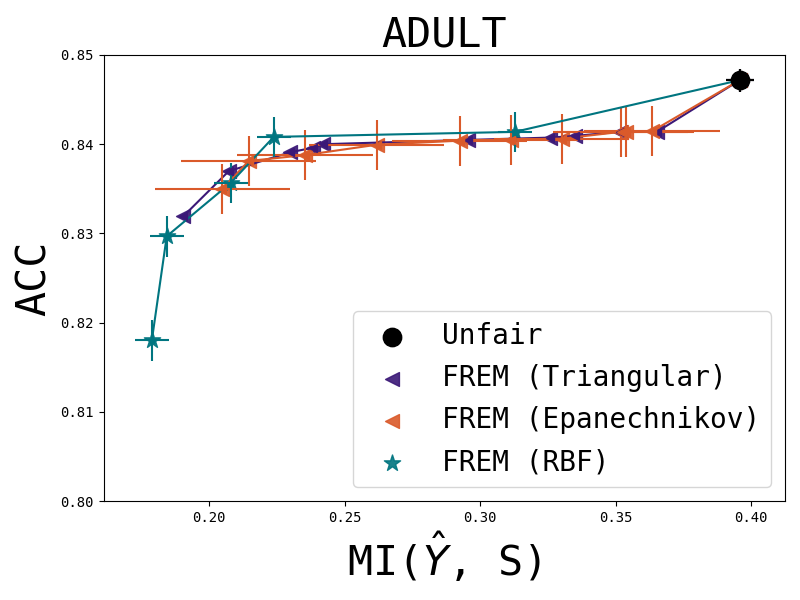}
    \\
    \includegraphics[width=0.3\textwidth]{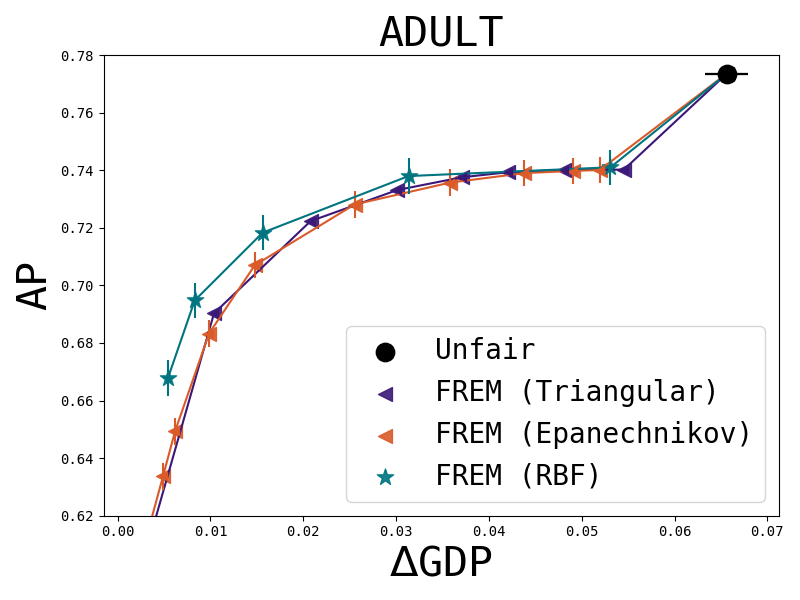}
    \includegraphics[width=0.3\textwidth]{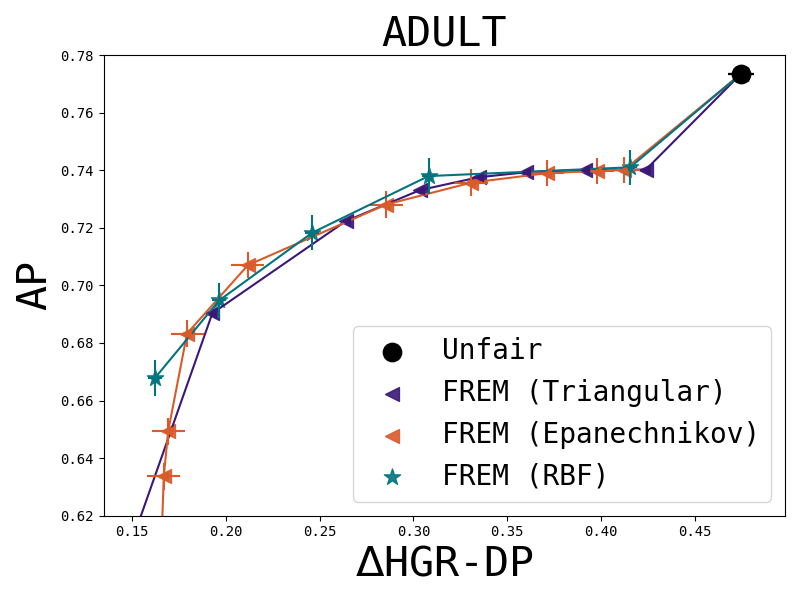}
    \includegraphics[width=0.3\textwidth]{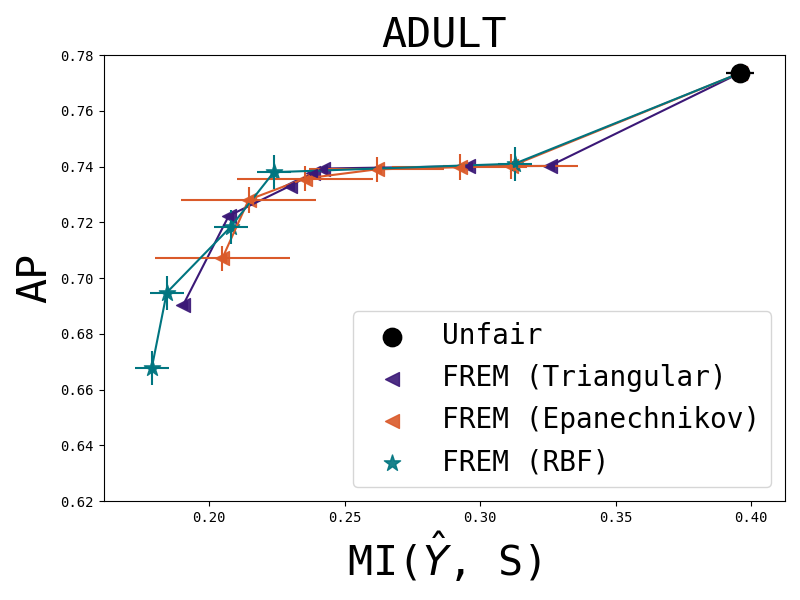}
    \caption{{
    \textbf{FREM with various kernel functions}: 
    Pareto-front lines for fairness-prediction trade-off on \textsc{Adult} dataset with various kernel functions.
    (Top)
    $\Delta \texttt{GDP}$ vs. \texttt{ACC}, $\Delta \texttt{HGR-DP}$ vs. \texttt{ACC} and $\Delta \texttt{MI}(\hat{Y}, S)$ vs. \texttt{ACC}.
    (Bottom)
    $\Delta \texttt{GDP}$ vs. \texttt{AP}, $\Delta \texttt{HGR-DP}$ vs. \texttt{AP} and $\Delta \texttt{MI}(\hat{Y}, S)$ vs. AP.
    $\bullet$: Unfair,
    \textcolor{Trianglecolor}{--$\blacktriangleleft$--}: FREM with Triangular kernel,
    \textcolor{Epanecolor}{--$\blacktriangleleft$--}: FREM with Epanechnikov kernel,
    \textcolor{Ourcolor}{--$\mathbf{\filledstar}$--}: FREM with RBF kernel.
    }}
    \label{fig:tradeoff-kernel_1}
\end{figure*}

\begin{figure*}[ht]
    \centering
    \includegraphics[width=0.3\textwidth]{figures/Crime/dp/mae_gdp_w_kernel_kernel.png}
    \includegraphics[width=0.3\textwidth]{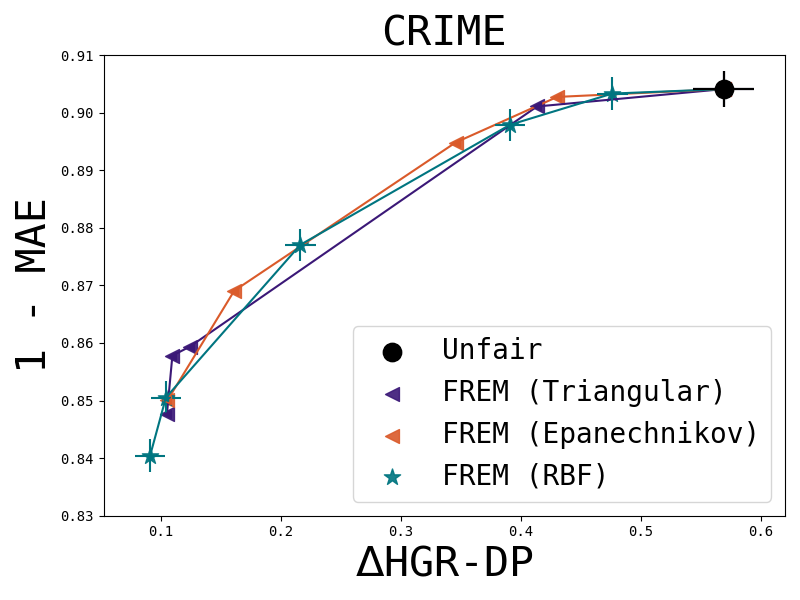}
    \includegraphics[width=0.3\textwidth]{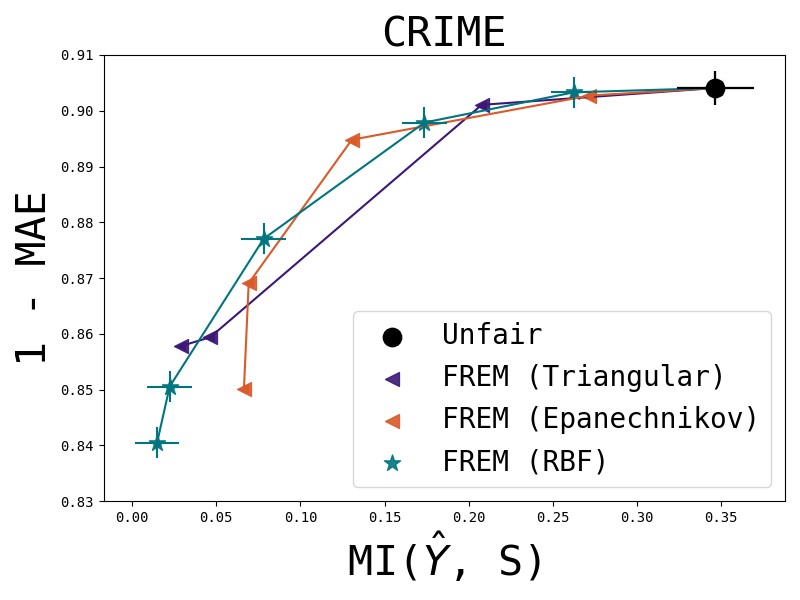}
    \\
    \includegraphics[width=0.3\textwidth]{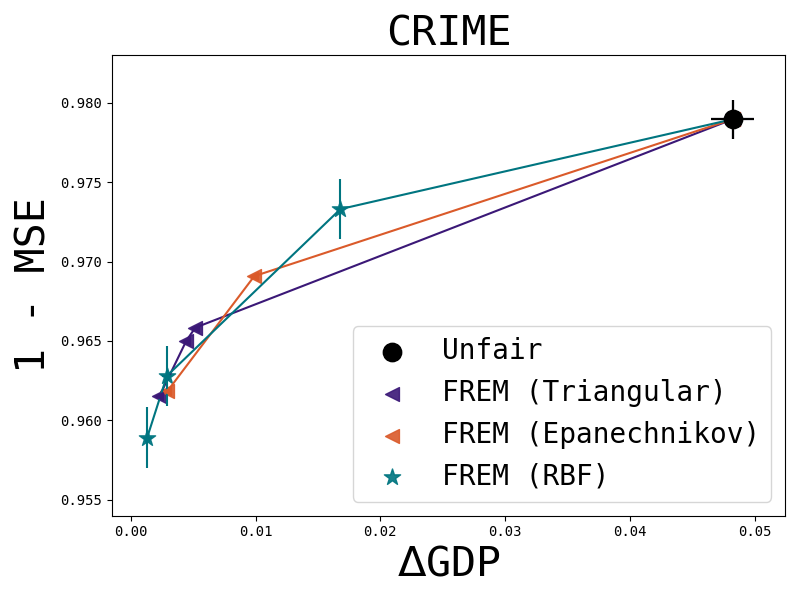}
    \includegraphics[width=0.3\textwidth]{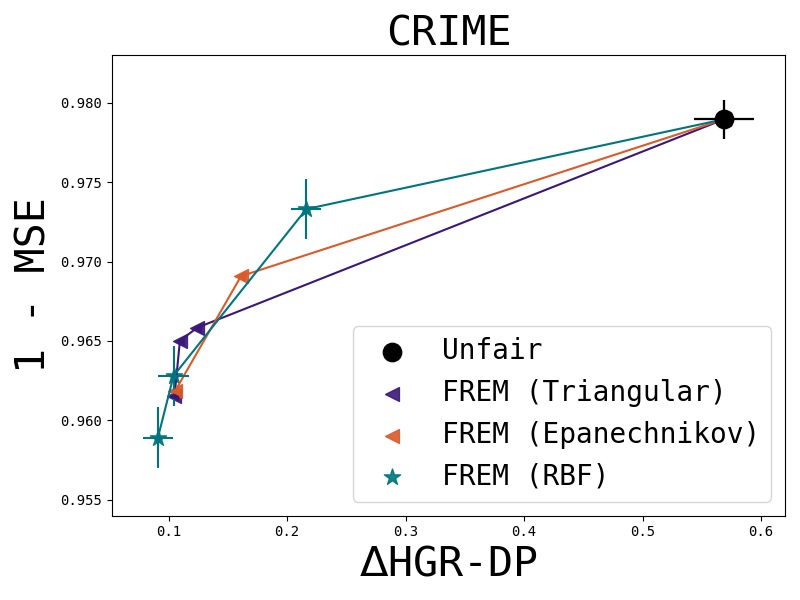}
    \includegraphics[width=0.3\textwidth]{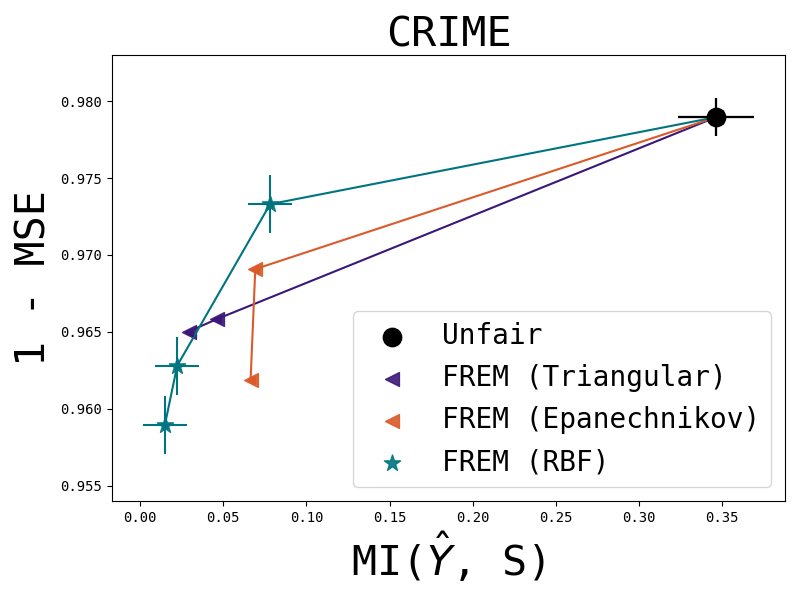}
    \caption{{
    \textbf{FREM with various kernel functions}: 
    Pareto-front lines for fairness-prediction trade-off on \textsc{Crime} dataset with various kernel functions.
    (Top)
    $\Delta \texttt{GDP}$ vs. 1 - \texttt{MAE}, $\Delta \texttt{HGR-DP}$ vs. 1 - \texttt{MAE} and $\Delta \texttt{MI}(\hat{Y}, S)$ vs. 1 - \texttt{MAE}.
    (Bottom)
    $\Delta \texttt{GDP}$ vs. 1 - \texttt{MSE}, $\Delta \texttt{HGR-DP}$ vs. 1 - \texttt{MSE}, and $\Delta \texttt{MI}(\hat{Y}, S)$ vs. 1 - \texttt{MSE}.
    $\bullet$: Unfair,
    \textcolor{Trianglecolor}{--$\blacktriangleleft$--}: FREM with Triangular kernel,
    \textcolor{Epanecolor}{--$\blacktriangleleft$--}: FREM with Epanechnikov kernel,
    \textcolor{Ourcolor}{--$\mathbf{\filledstar}$--}: FREM with RBF kernel.
    }}
    \label{fig:tradeoff-kernel_2}
\end{figure*}

\clearpage
{
\textbf{(2) Kernel function $\kappa$ in MMD:}
In addition, we analyze the impact of kernel used in MMD on the representation space.
Due to computational instability, Epanechnikov kernel for \textsc{Adult} dataset does not provide reasonable performances, and
hence we exclude it from the comparison.
Fig. \ref{fig:tradeoff-kernel_z_1} and \ref{fig:tradeoff-kernel_z_2} below show that the RBF and Triangular
kernels perform similarly, which suggests that FREM is not sensitive to the choice of the kernel in MMD unless
there is any numerical problems.}

\begin{figure*}[ht]
    \centering
    \includegraphics[width=0.3\textwidth]{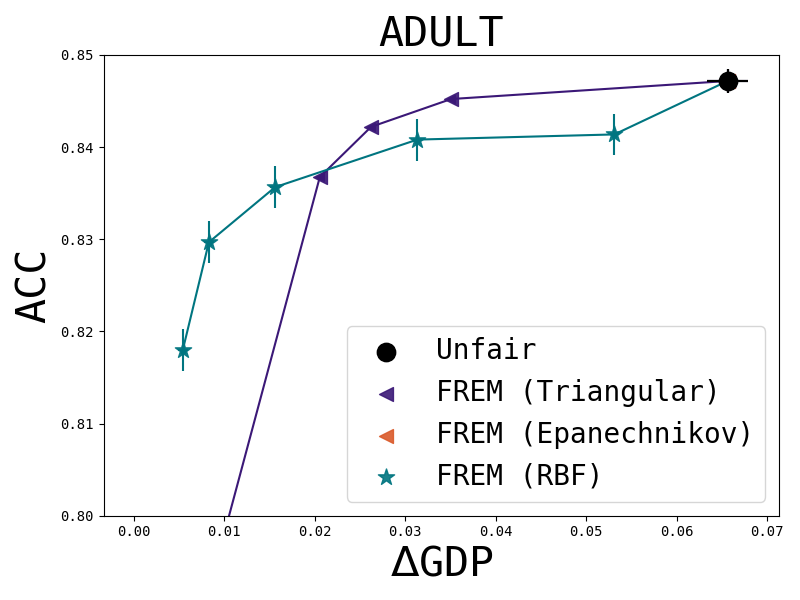}
    \includegraphics[width=0.3\textwidth]{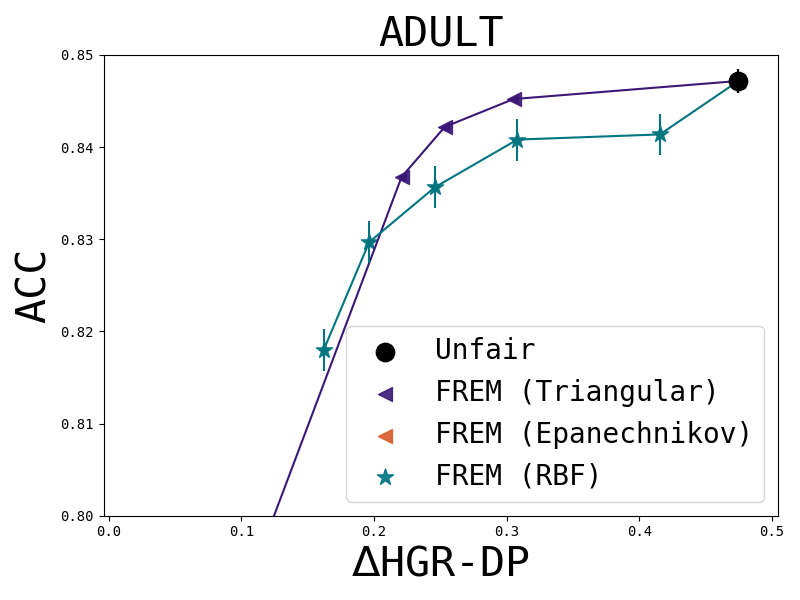}
    \includegraphics[width=0.3\textwidth]{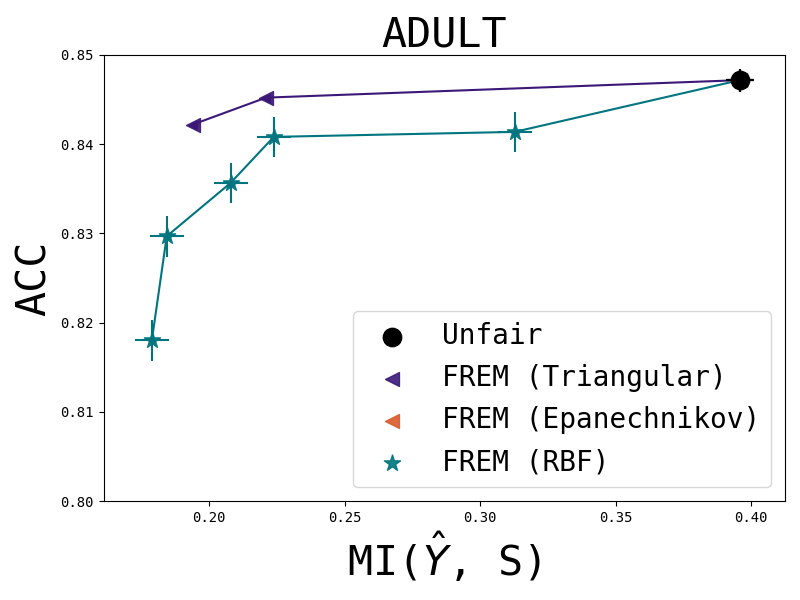}
    \includegraphics[width=0.3\textwidth]{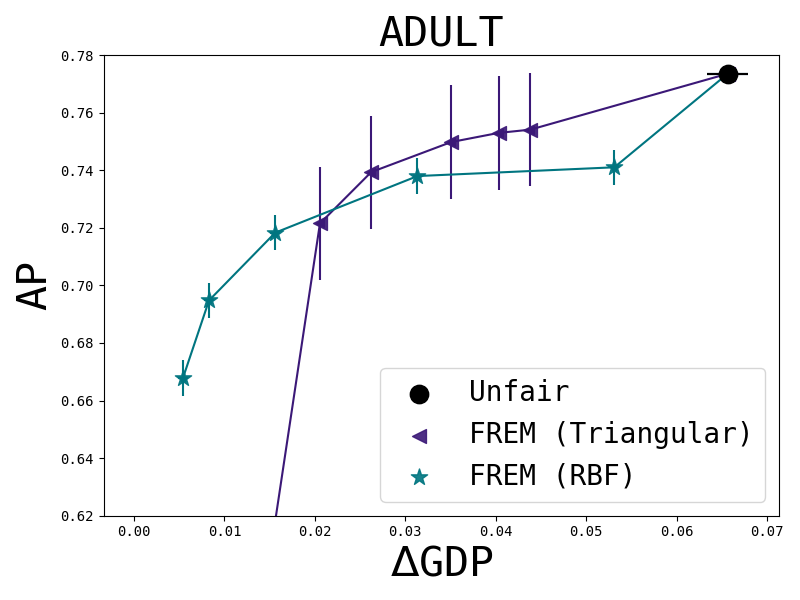}
    \includegraphics[width=0.3\textwidth]{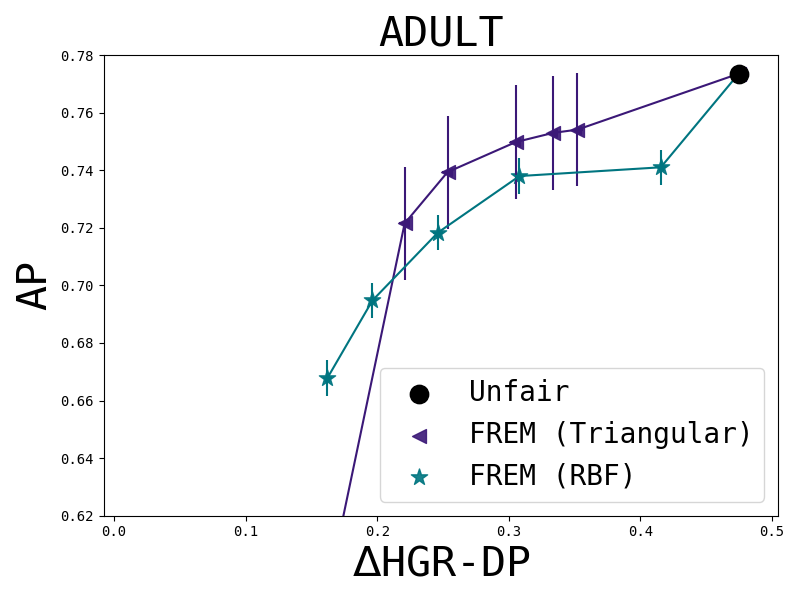}
    \includegraphics[width=0.3\textwidth]{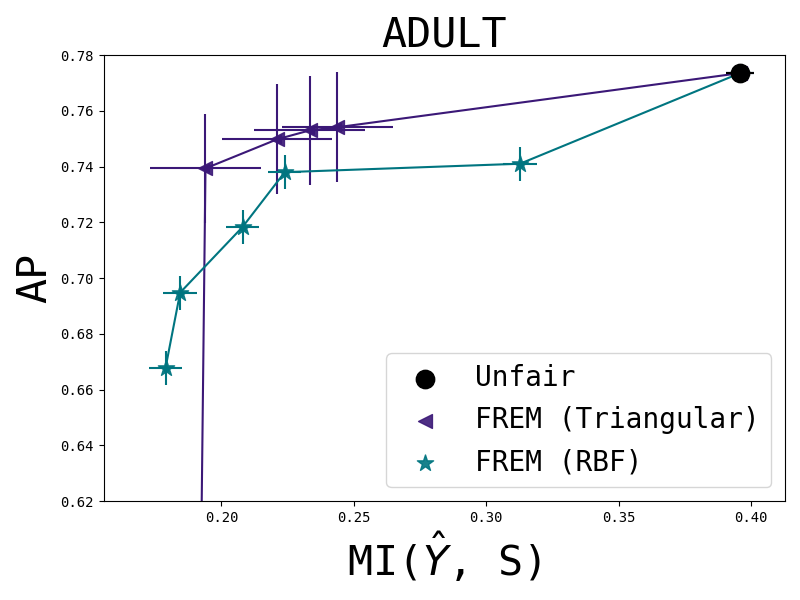}
    \caption{{
    \textbf{Comparison of kernel functions in MMD}: 
    Pareto-front lines for fairness-prediction trade-off on \textsc{Adult} dataset with various kernel functions in MMD.
    $\bullet$: Unfair,
    \textcolor{Trianglecolor}{--$\blacktriangleleft$--}: FREM with Triangular kernel in MMD,
    \textcolor{Ourcolor}{--$\mathbf{\filledstar}$--}: FREM with RBF kernel in MMD.
    }}
    \label{fig:tradeoff-kernel_z_1}
\end{figure*}

\begin{figure*}[ht]
    \centering
    \includegraphics[width=0.3\textwidth]{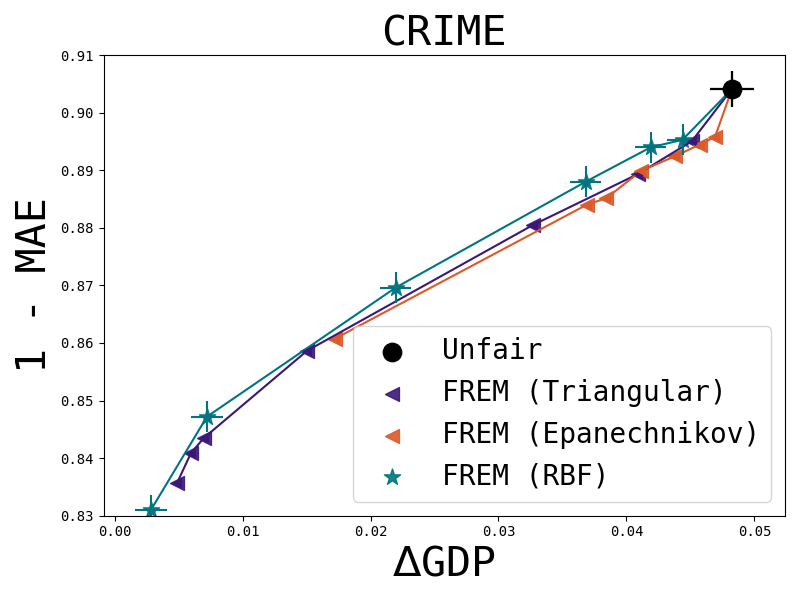}
    \includegraphics[width=0.3\textwidth]{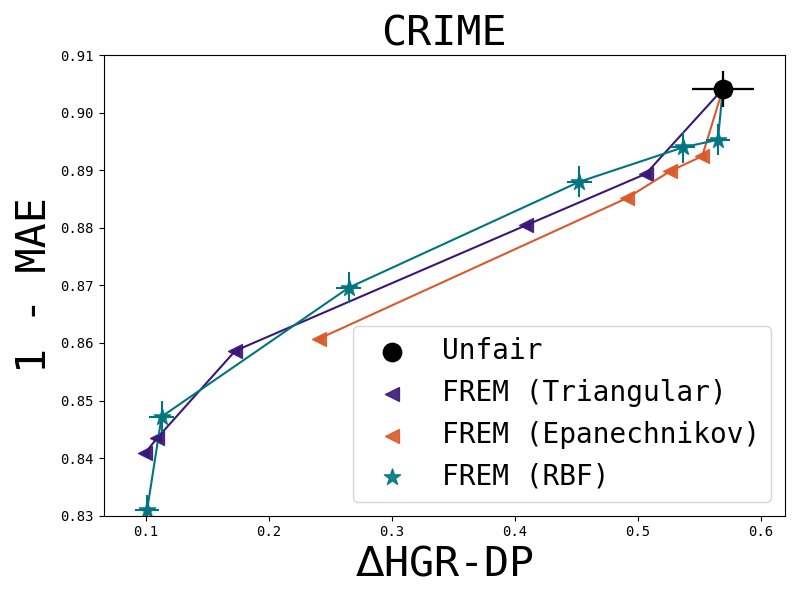}
    \includegraphics[width=0.3\textwidth]{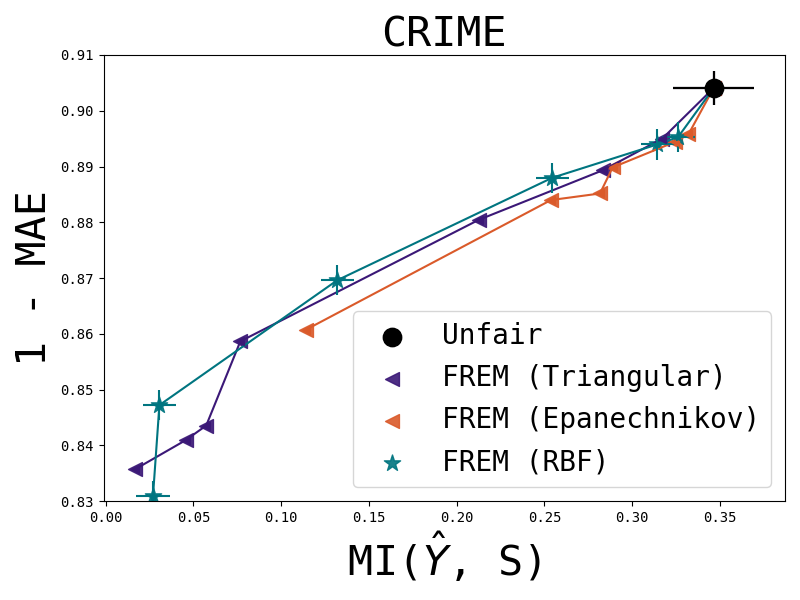}
    \includegraphics[width=0.3\textwidth]{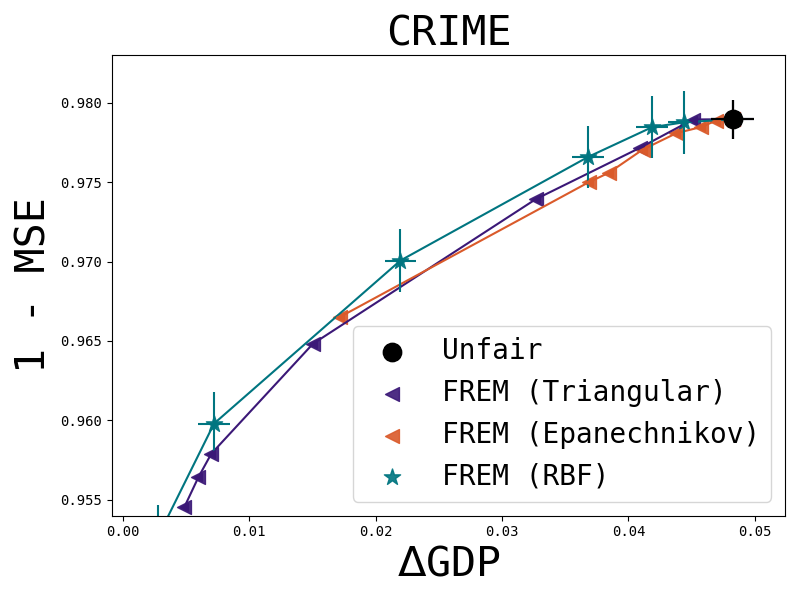}
    \includegraphics[width=0.3\textwidth]{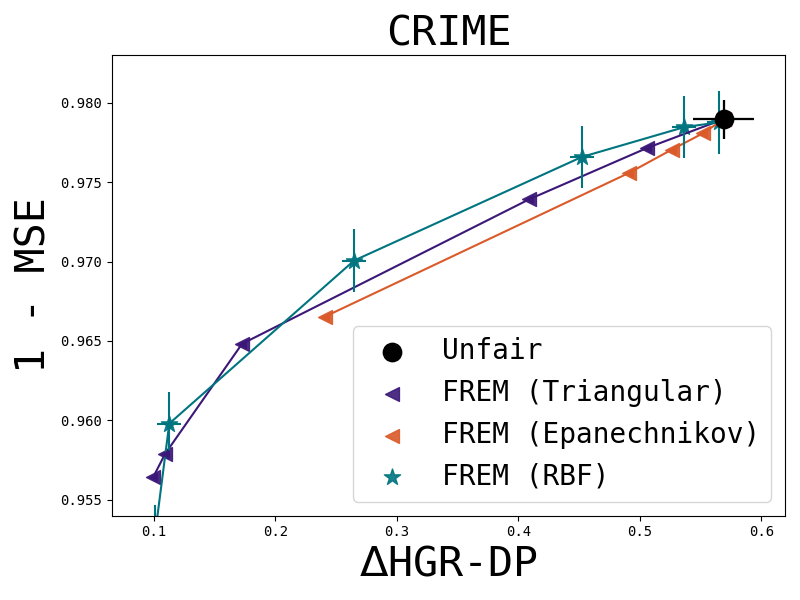}
    \includegraphics[width=0.3\textwidth]{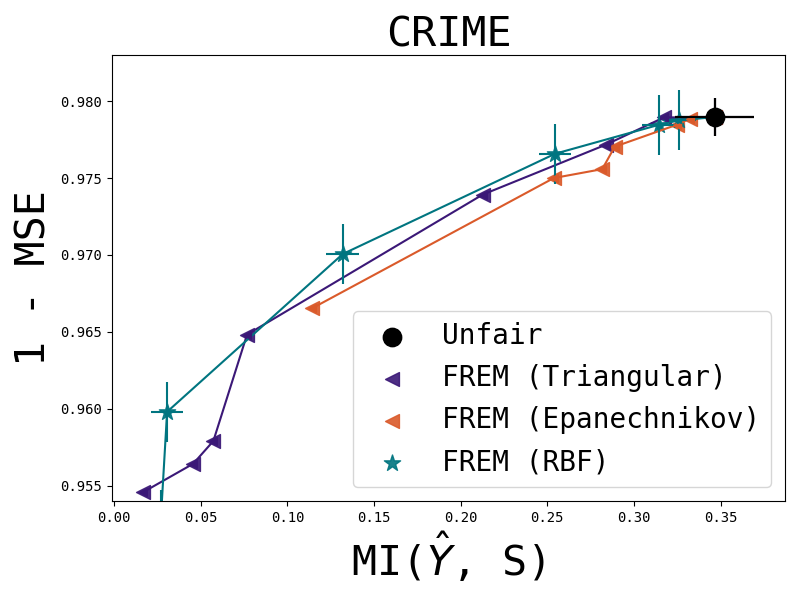}
    \caption{{
    \textbf{Comparison of kernel functions in MMD}: 
    Pareto-front lines for fairness-prediction trade-off on \textsc{Crime} dataset with various kernel functions in MMD.
    $\bullet$: Unfair,
    \textcolor{Trianglecolor}{--$\blacktriangleleft$--}: FREM with Triangular kernel in MMD,
    \textcolor{Epanecolor}{--$\blacktriangleleft$--}: FREM with Epanechnikov kernel in MMD,
    \textcolor{Ourcolor}{--$\mathbf{\filledstar}$--}: FREM with RBF kernel in MMD.
    }}
    \label{fig:tradeoff-kernel_z_2}
\end{figure*}

\clearpage

\textbf{(3) Sensitivity analysis: the scale $\sigma$}
We conduct a sensitivity analysis  regarding $\sigma$  to show the robustness of the choice of $\sigma$ 
to the performance, whose results is presented in Fig. \ref{fig:bandwidths_rep} below.
The result implies that performance of FREM is not sensitive to the choice of 
the scale parameter $\sigma$ in MMD.
Particularly, the choice of $\sigma$ within an appropriate range, such as $[0.8, 1.5]$ yields similar results, which
would be partly because input data are normalized beforehand.
Apparently $\sigma=1$ would be a good choice.

\begin{figure}[ht]
    \centering
    \includegraphics[scale=0.35]{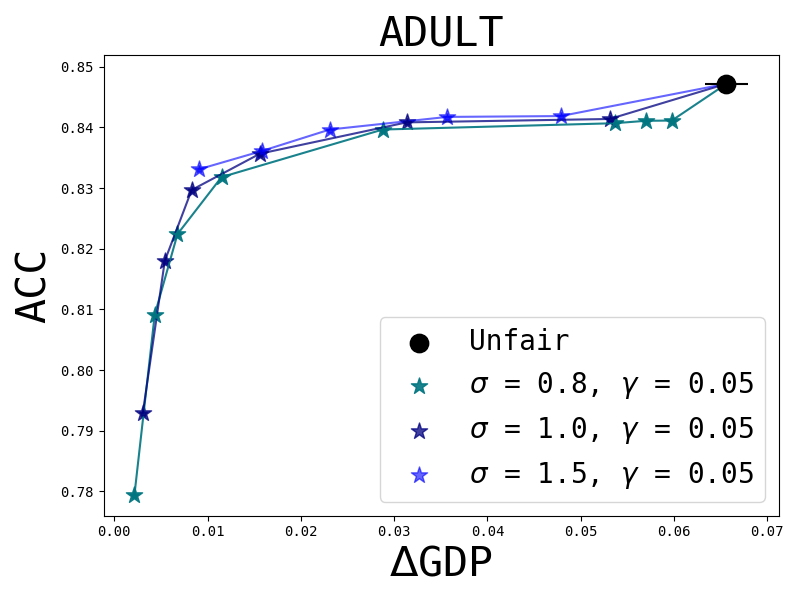}
    \includegraphics[scale=0.35]{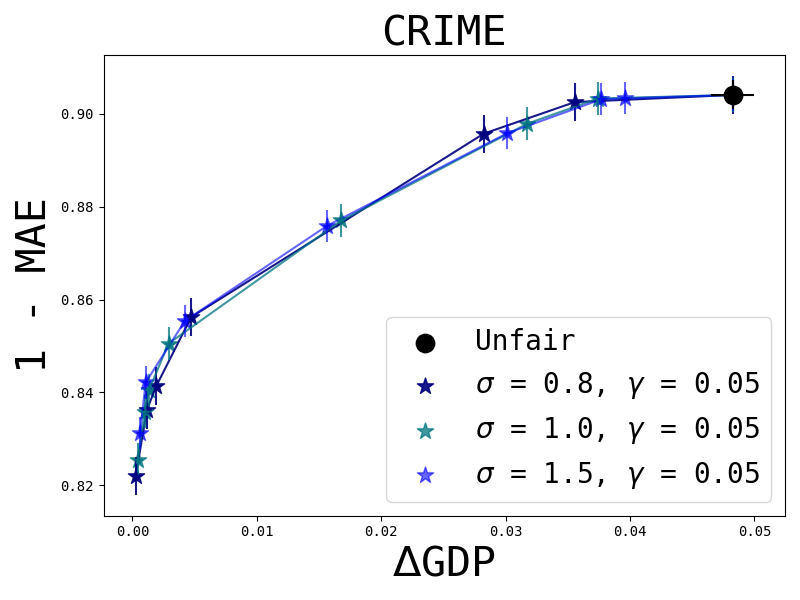}
    \caption{\textbf{Sensitivity analysis of $\sigma$}:
    Pareto-front lines for fairness-prediction trade-off with respect to various values of $\sigma.$
    (Left) \textsc{Adult} dataset: $\sigma \in \{ 0.8, 1.0, 1.5 \}$ for fixed $\gamma = 0.05.$
    (Right) \textsc{Crime} dataset: $\sigma \in \{ 0.8, 1.0, 1.5 \}$ for fixed $\gamma = 0.05.$
    }
    \label{fig:bandwidths_rep}
\end{figure}

\subsubsection{Additional comparison between FREM and Reg-GDP: results on training dataset}\label{frem_vs_gdp}

    This subsection presents the comparison between FREM and Reg-GDP, in terms of training performance.
    We draw pareto-front lines for each method in Fig. \ref{fig:tradeoff-training} (for the tabular datasets) and Fig. \ref{fig:tradeoff-training_graph} (for the graph datasets).
    While Reg-GDP offers lower training losses than FREM on the tabular datasets, FREM and Reg-GDP perform similarly on the graph datasets.
\color{black}

\begin{figure}[ht]
    \centering
    \includegraphics[width=0.35\textwidth]{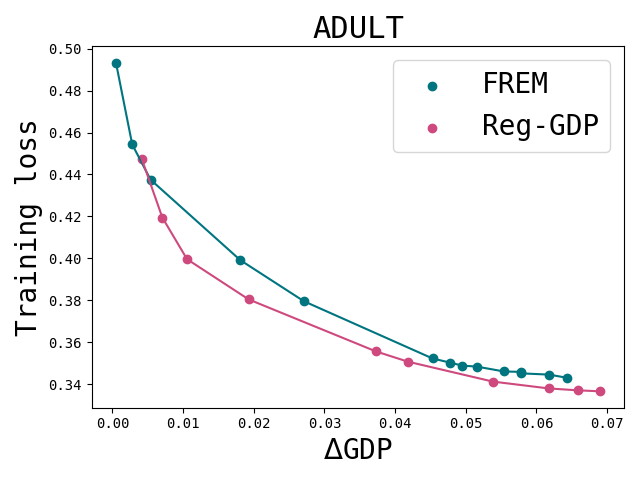}
    \includegraphics[width=0.35\textwidth]{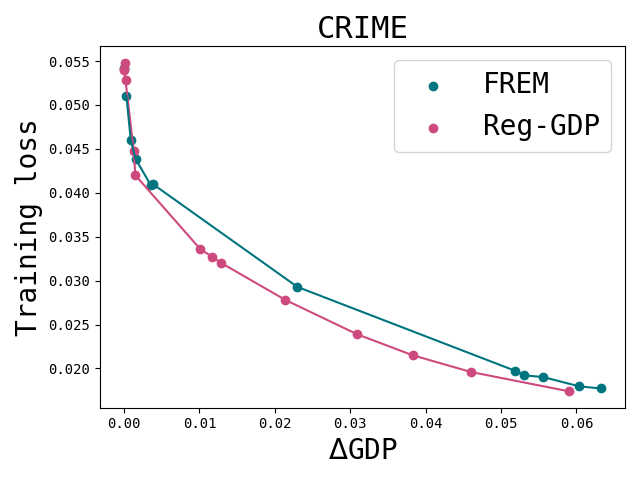}
    \caption{{
    $\Delta \texttt{GDP}$ vs. training loss trade-offs on
    (left) \textsc{Adult} and (right) \textsc{Crime} datasets:
    \textcolor{Regkernelcolor}{$-\bullet-$}: Reg-GDP,
    \textcolor{Ourcolor}{$-\mathbf{\bullet}-$}: FREM.
    }}
    \label{fig:tradeoff-training}
\end{figure}

\begin{figure}[ht]
    \centering
    \includegraphics[width=0.23\textwidth]{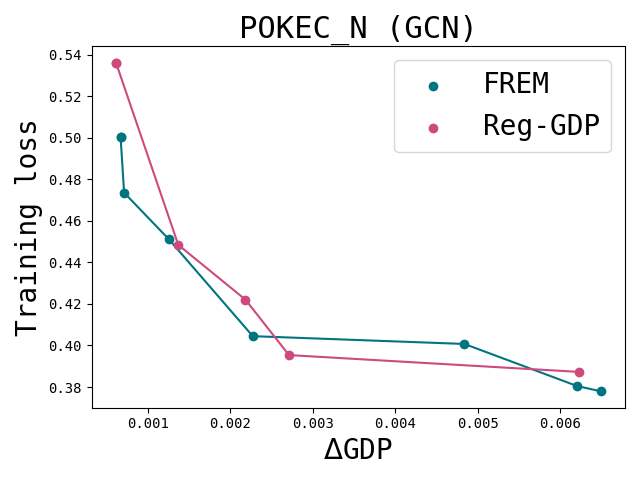}
    \includegraphics[width=0.23\textwidth]{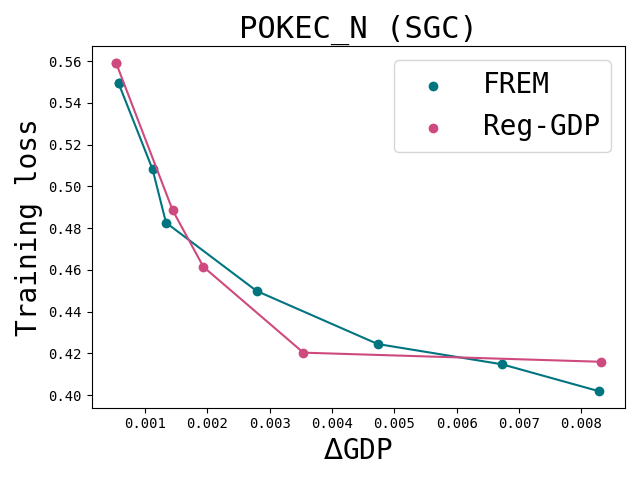}
    \includegraphics[width=0.23\textwidth]{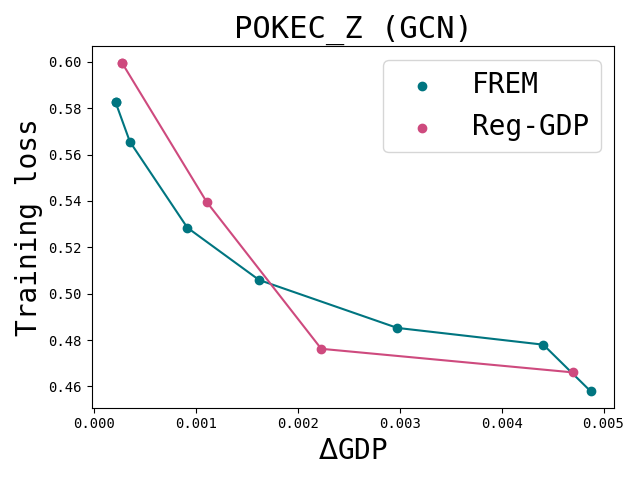}
    \includegraphics[width=0.23\textwidth]{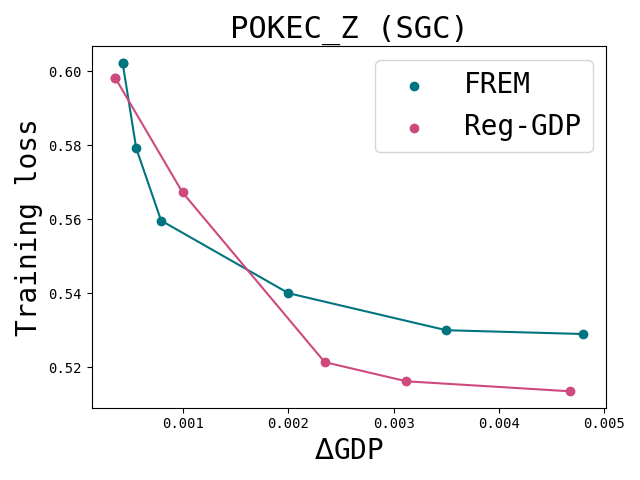}
    \caption{{
    $\Delta \texttt{GDP}$ vs. training loss trade-offs on
    (left two) \textsc{Pokec-n} and (right two) \textsc{Pokec-z} datasets:
    \textcolor{Regkernelcolor}{$-\bullet-$}: Reg-GDP,
    \textcolor{Ourcolor}{$-\mathbf{\bullet}-$}: FREM.
    }}
    \label{fig:tradeoff-training_graph}
\end{figure}

\end{document}